\documentclass[11pt,a4paper,oneside]{book}

\usepackage[titletoc]{appendix}
\usepackage{helvet}
\usepackage{lmodern}
\usepackage{mathpazo}   
\usepackage{amsmath}
\usepackage{amsthm}
\usepackage{appendix}
\usepackage{array}
\usepackage{bm}
\usepackage{booktabs,threeparttable}
\usepackage{color}
\usepackage{comment}
\usepackage{dcolumn}
\usepackage{enumerate}
\usepackage{fancyhdr}
\usepackage{fancyvrb}
\usepackage{float}
\usepackage[margin=3.25cm]{geometry}
\usepackage{graphicx}
\usepackage[bookmarks=False, hidelinks, pdfstartview={XYZ null null 1.00}]{hyperref}
\usepackage{ifpdf}
\usepackage{listings}
\usepackage{mathrsfs}
\usepackage{multirow}
\usepackage{natbib}
\usepackage{placeins}
\usepackage{slashbox}
\usepackage{stackrel}
\usepackage{subfig}
\usepackage{tabularx}
\usepackage{times}
\usepackage{xcolor}
\usepackage{bbm}
\usepackage{xfrac }
\usepackage[normalem]{ulem}
\usepackage{mathtools}
\usepackage{amssymb}
\usepackage{algorithm}
\usepackage{algorithmic}
\usepackage{balance}
\usepackage{lscape}

\RequirePackage{fix-cm}
\usepackage{xfrac}
\usepackage{mathrsfs}
\DeclareFontFamily{U}{rsfs}{\skewchar\font127 }
\DeclareFontShape{U}{rsfs}{m}{n}{%
   <-6> rsfs5
   <6-8> rsfs7
   <8-> rsfs10
}{}

\newtheorem{theorem}{Theorem}

\newtheorem{lemma}{Lemma}
\newtheorem{proposition}{Proposition}

\newtheoremstyle{remboldstyle}
  {}{}{\itshape}{}{\bfseries}{.}{.5em}{{\thmname{#1 }}{\thmnumber{#2}}{\thmnote{ (#3)}}}
\theoremstyle{remboldstyle}
\theoremstyle{definition}

\theoremstyle{definition}

\let\oldproofname=\proofname
\renewcommand{\proofname}{\rm\textbf{\oldproofname}}

\newenvironment{dedication}
{\vspace{6ex}\begin{quotation}\begin{center}\begin{em}}
{\par\end{em}\end{center}\end{quotation}}

\makeatletter
\def\Eqlfill@{\arrowfill@\Relbar\Relbar\Relbar}
\newcommand{\extendEql}[1][]{\ext@arrow 0359\Eqlfill@{#1}}
\makeatother

\makeatletter

\newcommand{\Rmnum}[1]{\expandafter\@slowromancap\romannumeral #1@}
\makeatother

\newcommand{\HRule}{\rule{\linewidth}{0.5mm}}

\newcommand{\argmax}{\mathrm{arg}\displaystyle\max}
\newcommand{\maxmax}{\displaystyle\max}

\newcommand{\argmin}{\mathrm{arg}\displaystyle\min}

\newcolumntype{.}{D{.}{.}{-1}}

\numberwithin{equation}{chapter}

\hypersetup{pageanchor=false}

\newlength{\struthd}

\newcolumntype{M}[1]{>{\centering\arraybackslash}m{#1}}
\newcolumntype{N}{@{}m{0pt}@{}}

\DeclareFontFamily{OT1}{pzc}{}
\DeclareFontShape{OT1}{pzc}{m}{it}{<-> s * [1.10] pzcmi7t}{}
\DeclareMathAlphabet{\mathpzc}{OT1}{pzc}{m}{it}

\begin{document}

\pagenumbering{roman}

\fontfamily{phv}
\selectfont
\begin{titlepage}

\centering
\includegraphics[width=0.4\textwidth]{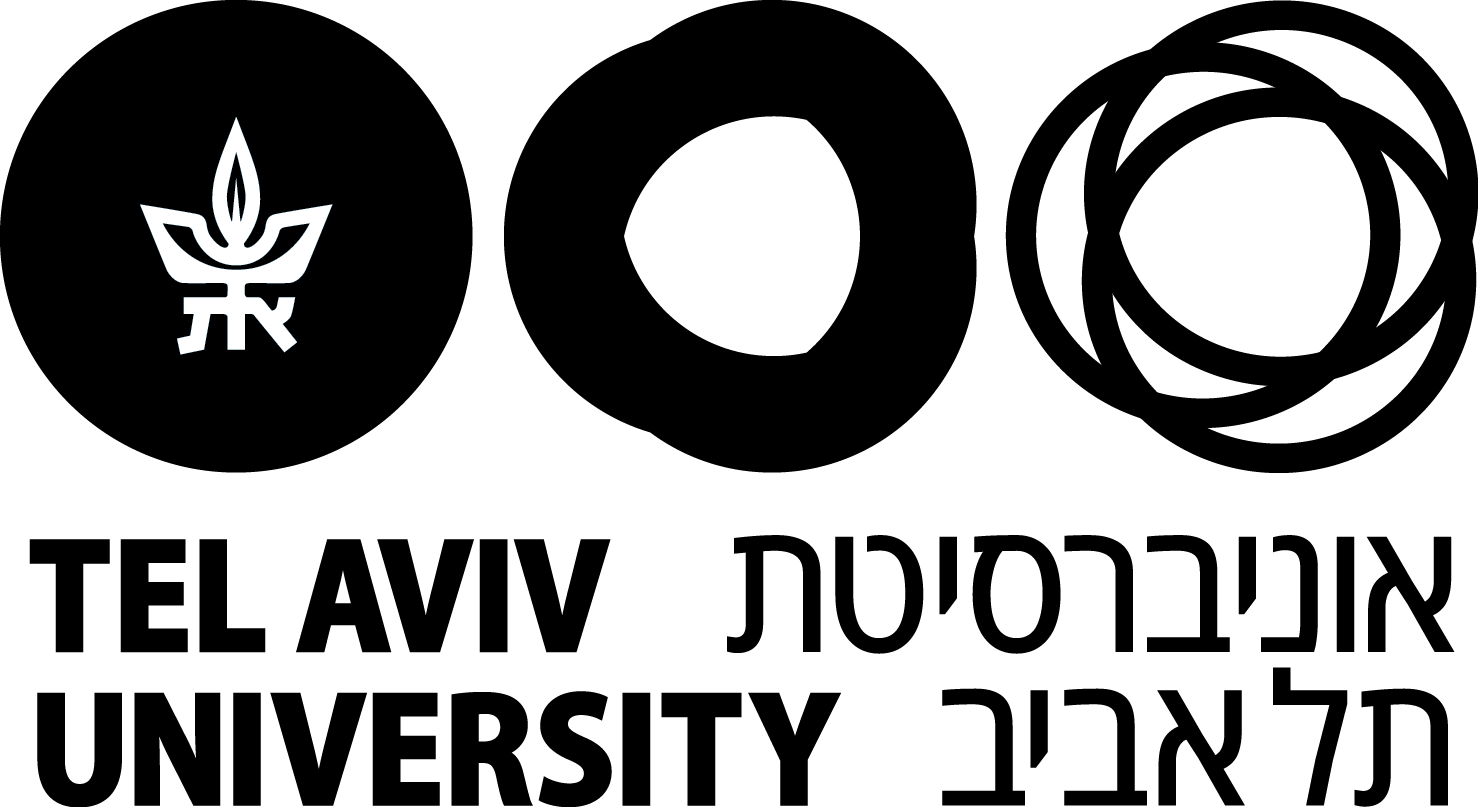}\\[0.5em]
%{\Large Tel-Aviv University}\\
\vspace{1.1em}
{\large
The Raymond and Beverly Sackler Faculty of Exact Sciences\\
School of Mathematical Sciences\\
Department of Statistics and Operations Research
}\\
\vspace{.2in}
\HRule\\
\vspace{0.2em}
\textbf{\LARGE PhD Thesis: \\ Generalized Independent Components Analysis\\[0.5em] Over Finite Alphabets}
\vspace{0.3em}
\HRule\\
\vspace{0.4in}
{\large by}\\[.75em]

\textbf{\LARGE Amichai Painsky}\\
\vspace{.35in}
{\large THESIS SUBMITTED TO THE SENATE OF TEL-AVIV UNIVERSITY}\\
{\large in partial fulfillment of the requirements for the degree of}\\
{\large ``DOCTOR OF PHILOSOPHY''}\\
\vspace{.4in}
{\Large Under the supervision of \\[0.5em] Prof. Saharon Rosset and Prof. Meir Feder}\\
\vspace{.10in}
{\Large September 01, 2016}
\end{titlepage}

\newpage
\thispagestyle{empty}
\mbox{}

\newgeometry{left=2.5cm, right=2.5cm, top = 3cm}
\chapter*{\centering \begin{Large}Abstract\end{Large}}
\addcontentsline{toc}{chapter}{Abstract}
\vspace{-1em}
\begin{quotation}
\begin{center}
{\bf\large Generalized Independent Components 
Analysis Over Finite Alphabets}\\[.5em]
{\large by Amichai Painsky}
\end{center}
\vspace{1em}
\noindent Independent component analysis (ICA) is a statistical method for transforming an observable multi-dimensional random vector into components that are as statistically independent as possible from each other. Usually the ICA framework assumes a model according to which the observations are generated (such as a linear transformation with additive noise). ICA over finite fields is a special case of ICA in which both the observations and the independent components are over a finite alphabet. \\

\noindent In this thesis we consider a formulation of the finite-field case in which an observation vector is decomposed to its independent components (as much as possible) with no prior assumption on the way it was generated. This generalization is also known as Barlow's minimal redundancy representation \citep{barlow1989finding} and is considered an open problem. We propose several theorems and show that this hard problem can be accurately solved with a branch and bound search tree algorithm, or tightly approximated with a series of linear problems \citep{painsky2015generalized}. Moreover, we show that there exists a simple transformation (namely, order permutation) which provides a greedy yet very effective approximation of the optimal solution \citep{painsky2016largetrans}. We further show that while not every random vector can be efficiently decomposed into independent components, the vast majority of vectors do decompose very well (that is, within a small constant cost), as the dimension increases. In addition, we show that we may practically achieve this favorable constant cost with a complexity that is asymptotically linear in the alphabet size. Our contribution provides the first efficient set of solutions to Barlow's problem with theoretical and computational guarantees.\\

\noindent The minimal redundancy representation (also known as factorial coding \citep{schmidhuber1992learning}) has many applications, mainly in the fields of neural networks and deep learning \citep{becker1996unsupervised, obradovic1996information,choi2000factorial, bartlett2002face, martiriggiano2005face, bartlett2007information,schmidhuber2011fast, schmidhuber2015deep}. In our work we show that the generalized ICA also applies to multiple disciplines in source coding \citep{painsky2016largetrans}. A special attention is given to large alphabet source coding \citep{painsky2015Universal, painsky2016largetrans,painsky2016Simple}. We propose a conceptual framework in which a large alphabet memoryless source is decomposed into multiple sources with with a much smaller alphabet size that are ``as independent as possible". This way we slightly increase the average code-word length as the decomposed sources are not perfectly independent, but at the same time significantly reduce the overhead redundancy resulted by the large alphabet of the observed source. Our suggested method is applicable for a variety of large alphabet source coding setups. 
\end{quotation}

%\newpage
%\thispagestyle{empty}
%\mbox{}

\clearpage

\begin{dedication}
To my father, Moti Painsky, who encouraged me to earn my B.Sc. and get a job
\end{dedication}

\newpage
\thispagestyle{empty}
\mbox{}

\chapter*{\centering \begin{Large}Acknowledgements\end{Large}}
\addcontentsline{toc}{chapter}{Acknowledgements}
\begin{quotation}
\noindent First and foremost I would like to express my gratitude and love to my wife Noga and my new born daughter Ofri, who simply make me happy every single day. Thank you for taking part in this journey with me.\\

\noindent I wish to express my utmost and deepest appreciation to my advisers, Prof. Saharon Rosset and Prof. Meir Feder, from whom I learned so much, in so many levels. Coming from different disciplines and backgrounds, Meir and Saharon inspired me to dream high, but at the same time stay accurate and rigorous. This collaboration with two extraordinary experts has led to a fascinating research with some meaningful contributions. On a personal level, it was a privilege to work with such exceptional individuals. Looking back five years ago, when I moved back to Israel to pursue my Ph.D. in Tel Aviv University, I could not dream it up any better.  \\

\noindent In the course of my studies I had the opportunity to meet and learn from so many distinguished individuals. Prof. Felix Abramovich, to whom I owe most of my formal statistical education. Prof. David Burshtein, who gave me the opportunity to teach the undergraduate Digital to Signal Processing class for the past four years. Prof. Uri Erez, who assigned me as a chief of Teaching Assistants in his Random Signals and Noise class. Dr. Ofer Shayevitz, who (unintentionally) led me to pursue my Ph.D. in Israel, on top of my offers abroad. Dr. Ronny Luss, who introduced me to Saharon when I moved back to Israel and guided my first steps in Optimization.  My fellow faculty members and graduate students from both the Statistics and Electrical Engineering departments, Dr. Ofir Harari, Dr.  Shlomi Lifshits, Dr. David Golan, Aya Vituri, Shachar Kaufman, Keren Levinstein, Omer Weissbord, Prof. Rami Zamir, Dr. Yuval Kochman, Dr. Zachi Tamo, Dr. Yair Yona, Dr. Anatoly Khina, Dr. Or Ordentlich, Dr. Ronen Dar, Assaf Ben-Yishai, Eli Haim, Elad Domanovitz, Nir Elkayam, Uri Hadar, Naor Huri, Nir Hadas, Lital Yodla and Svetlana Reznikov.\\

\noindent Finally I would like to thank my parents, Moti and Alicia Painsky, for their endless love, support and caring. It has been a constant struggle for the past decade, explaining what it is that I do for living. Yet it seems like you are quite content with the results. 
  
\end{quotation}

\restoregeometry
\clearpage

\tableofcontents
%\newpage
%\thispagestyle{empty}
%\mbox{}
%\thispagestyle{empty}

\pagenumbering{arabic}
\setcounter{page}{0}
%\mainmatter
\chapter{Introduction}\label{intro}
\graphicspath{{Introduction_Figures//}}
Independent Component Analysis (ICA) addresses the recovery of unobserved statistically independent source signals from their observed mixtures, without full prior knowledge of the mixing function or the statistics of the source signals. The classical Independent Components Analysis framework usually assumes linear combinations of the independent sources over the field of real valued numbers $\mathbb{R}$ \citep{hyvarinen2004independent}. A special variant of the ICA problem is when the sources, the mixing model and the observed signals are over a finite field.\\

\noindent Several types of generative mixing models can be assumed when working over GF(P), such as modulu additive operations, OR operations (over the binary field) and others. Existing solutions to ICA mainly differ in their assumptions of the generative mixing model, the prior distribution of the mixing matrix (if such exists) and the noise model. The common assumption to these solutions is that there exist statistically independent source signals which are mixed according to some known generative model (linear, XOR, etc.). \\

\noindent In this work we drop this assumption and consider a generalized approach which is applied to a random vector and decomposes it into independent components (as much as possible) with no prior assumption on the way it was generated. This problem was first introduced by \cite{barlow1989finding} and is considered a long--standing open problem.\\

\noindent In Chapter \ref{overview} we review previous work on ICA over finite alphabets. This includes two major lines of work. We first review the line of work initiated by \cite{yeredor2007ica}. In this work, Yeredor focuses on linear transformations where the assumptions are that the unknown sources are statistically independent and are linearly mixed (over GF(P)). Under these constraints, he proved that the there exists a unique transformation matrix to recover the independent signals (up to permutation ambiguity). This work was later extended to larger alphabet sizes \citep{yeredor2011independent} and different generative modeling assumptions \citep{vsingliar2006noisy,wood2012non,streich2009multi,nguyen2011binary}. In a second line of work, \cite{barlow1989finding} suggest to decompose the observed signals ``as much as possible", with no assumption on the generative model. Barlow et al. claim that such decomposition would capture and remove the redundancy of the data. However, they do not propose any direct method, and this hard problem is still considered open, despite later attempts  \citep{atick1990towards,schmidhuber1992learning,becker1996unsupervised}.\\

\noindent In Chapter \ref{Generalized_BICA} we present three different combinatorical approaches for independent decomposition of a given random vector, based on our published paper \citep{painsky2015generalized}. In the first, we assume that the underlying components are completely independent. This leads to a  simple yet highly sensitive algorithm which is not robust when dealing with real data. Our second approach drops the assumption of statistically independent components and strives to achieve ``as much independence as possible" (as rigorously defined in Section \ref{problem_formulation}) through a branch-and-bound algorithm. However, this approach is very difficult to analyze, both in terms of its accuracy and its computational burden. Then, we introduce a piece-wise linear approximation approach, which tightly bounds our objective from above. This method shows how to decompose any given random vector to its ``as statistically independent as possible" components with a computational burden that is competitive with any known benchmarks.\

\noindent In Chapter \ref{Order_Permutation} we present an additional, yet simpler approach to the generalized ICA problem, namely, 
\textit{order permutation}. Here, we suggest to represent the $i^{th}$ least probable realization of a given random vector with the number $i$ \citep{painsky2016largetrans}. Despite its simplicity, this method holds some favorable theoretical properties. We show that on the average (where the average is taken over all possible distribution functions of a given alphabet size), the order permutation is only a small constant away from full statistical independence, even as the dimension increases. In fact, this result provides a theoretical guarantee on the ``best we can wish for", when trying to decompose any random vector (on the average). In addition, we show that we may practically achieve the average accuracy of the order permutation with a complexity that is asymptotically linear in the alphabet size.\\

\noindent In Chapter \ref{BICA_Vs_Linear} we focus on the binary case and compare our suggested approaches with linear binary ICA (BICA). Although several linear BICA methods were presented in the past years \citep{attux2011immune,silva2014michigan,silva2014cobica}, they all lack theoretical guarantees on how well they perform. Therefore, we begin this section by introducing a novel lower bound on the generalized BICA problem over linear transformations. In addition, we present a simple heuristic which empirically outperforms all currently known methods. Finally, we show that the simple order permutation (presented in the previous section) outperforms the linear lower bound quite substantially, as the alphabet size increases.\\

\noindent Chapter \ref{Sequential_ICA} discusses a different aspect of the generalized ICA problem, in which we limit ourselves to sequential processing \citep{painsky2013memoryless}. In other words, we assume that the components of a given vector (or process) are presented to us one after the other, and our goal is to represent it as a process with statistically independent components (memoryless), in a no-regret manner.  In this chapter we present a non-linear method to generate such memoryless process from any given process under varying objectives and constraints. We differentiate between lossless and lossy methods, closed form and algorithmic solutions and discuss the properties and uniqueness of our suggested methods. In addition, we show that this problem is closely related to the multi-marginal optimal transportation problem \citep{monge1781memoire,kantorovich1942translocation,pass2011uniqueness}.\\

\noindent In Chapter \ref{Application_to_Data_Compression} we apply our methodology to multiple data compression problems. Here, we propose a conceptual framework in which a large alphabet memoryless source is decomposed into multiple ``as independent as possible" sources with a much smaller alphabet size \citep{painsky2015Universal,painsky2016largetrans,painsky2016Simple}. This way we slightly increase the average code-word length as the compressed symbols are no longer perfectly independent, but at the same time significantly reduce the redundancy resulted by the large alphabet of the observed source. Our proposed algorithm, based on our solutions to the Barlow's problem, shows to efficiently find the ideal trade-off so that the overall compression size is minimal. We demonstrate our suggested approach in a variety of lossless and lossy source coding problems. This includes the classical lossless compression, universal compression and high-dimensional vector quantization. In each of these setups, our suggested approach outperforms most commonly used methods. Moreover, our proposed framework is significantly easier to implement in most of these cases.\\

\noindent This thesis provides a comprehensive overview of the following publications \citep{painsky2013memoryless, painsky2014generalized, painsky2015Universal,painsky2016Binary, painsky2015generalized, painsky2016Simple} and a currently under--review manuscript \citep{painsky2016largetrans}.

\chapter{Overview of Related Work}\label{overview}
\graphicspath{{Overview_Figures//}}

In his work from $1989$, \cite{barlow1989finding} presented a \textit{minimally redundant  representation} scheme for binary data. He claimed that a good representation should capture and remove the redundancy of the data. This leads to a \textit{factorial representation/ encoding} in which the components are as mutually independent of each other as possible. Barlow suggested that such representation may be achieved through \textit{minimum entropy encoding}: an invertible transformation (i.e., with no information loss) which minimizes the sum of marginal entropies (as later presented in (\ref{eq:sum_ent_min})). Barlow's representation is also known as Factorial representation or Factorial coding.\\

\noindent Factorial representations have several advantages. The probability of the occurrence of any realization can be simply computed as the product of the probabilities of the individual components that represent it (assuming such decomposition exists). In addition, any method of finding factorial codes automatically implements \textit{Occam's razor} which prefers simpler models over more complex ones, where simplicity is defined as the number of parameters necessary to represent the joint distribution of the data.
In the context of supervised learning, independent features can also make later learning easier; if the input units to a supervised learning network are uncorrelated, then the Hessian of its error function is diagonal, allowing accelerated learning abilities \citep{becker1988improving}.
There exists a large body of work which demonstrates the use of factorial codes in learning problems. This mainly includes Neural Networks \citep{becker1996unsupervised, obradovic1996information} with application to facial recognition \citep{choi2000factorial, bartlett2002face, martiriggiano2005face, bartlett2007information}  and more recently, Deep Learning \citep{schmidhuber2011fast, schmidhuber2015deep}.  \\

\noindent Unfortunately Barlow did not suggest any direct method for finding factorial codes. Later, \cite{atick1990towards} proposed a cost function for Barlow's principle for linear systems, which minimize the redundancy of the data subject to a minimal information loss constraint. This is closely related to Plumbey's objective function  \citep{plumbley1993efficient}, which minimizes the information loss subject to a fixed redundancy constraint. \cite{schmidhuber1992learning} then introduced several ways of approximating Barlow's minimum redundancy principle in the non--linear case. This naturally implies much stronger results of statistical independence. However, Schmidhuber's scheme is rather complex, and appears to be subject to local minima \citep{becker1996unsupervised}. To our best knowledge, the problem of finding minimal redundant codes, or factorial codes, is still considered an open problem. In this work we present what appears to be the first efficient set of methods for minimizing Barlow's redundancy criterion, with theoretical and computational complexity guarantees. \\

\noindent In a second line of work, we may consider our contribution as a generalization of the BICA problem. In his pioneering BICA work, \cite{yeredor2007ica} assumed linear XOR mixtures and investigated the identifiability problem. A deflation algorithm is proposed for source separation based on entropy minimization. Yeredor assumes the number of independent sources is known and the mixing matrix is a $d$-by-$d$ invertible matrix. Under these constraints, he proves that the XOR model is invertible and there exists a unique transformation matrix to recover the independent components up to permutation ambiguity. \cite{yeredor2011independent} then extended his work to cover the ICA problem over Galois fields of any prime number. His ideas were further analyzed and improved by \cite{gutch2012ica}. \\

%In \cite{kaban2008factorisation} , the problem of factorization and de-noising of binary data due to independent continuous sources is considered. The sources are assumed to be continuous, and conditioned on the latent variables, the observations follow the independent Bernoulli likelihood model with mean vectors taking the form of a linear mixture of the latent variables. A variational EM solution is devised to infer the mixing coefficients. 

\noindent \cite{vsingliar2006noisy} introduced a noise-OR model for dependency among observable random variables using $d$ (known) latent factors. A variational inference algorithm is developed. In the noise-OR model, the probabilistic dependency between observable vectors and latent vectors is modeled via the noise-OR conditional distribution. \cite{wood2012non} considered the case where the observations are generated from a noise-OR generative model. The prior of the mixing matrix is modeled as the Indian buffet process \citep{griffiths2005infinite}. Reversible jump Markov chain Monte Carlo and Gibbs sampler techniques are applied to determine the mixing matrix. \cite{streich2009multi} studied the BICA problem  where the observations are either drawn from a signal following OR mixtures or from a noise component. The key assumption made in that work is that the observations are conditionally independent given the model parameters (as opposed to the latent variables). This greatly reduces the computational complexity and makes the scheme amenable to a objective descent-based optimization solution. However, this assumption is in general invalid. \cite{nguyen2011binary} considered OR mixtures and propose a deterministic iterative algorithm to determine the distribution of the latent random variables and the mixing matrix. \\
    
\noindent There also exists a large body of work on blind deconvolution with binary sources in the context of wireless communication \citep{diamantaras2006blind,yuanqing2003blind} and some literature on Boolean/binary factor analysis (BFA) which is also related to this topic \citep{belohlavek2010discovery}.

\newpage
\thispagestyle{empty}
\mbox{}
\thispagestyle{empty}

\chapter[Generalized ICA - Combinatorical Approach]{Generalized Independent Component Analysis - Combinatorical Approach}\label{Generalized_BICA}
\graphicspath{{Generalized_BICA_Figures//}}
\noindent The material in this Chapter is partly covered in \citep {painsky2015generalized}.

\section{Notation}

\noindent Throughout the following chapters we use the following standard notation: underlines denote vector quantities, where their respective components are written without underlines but with index. For example, the components of the $d$-dimensional vector $\underline{X}$ are $X_1, X_2, \dots X_d$.
Random variables are denoted with capital letters while their realizations are denoted with the respective lower-case letters. $P_{\underline{X}}\left(\b{x}\right)  \triangleq P(X_1 = x_1, X_2= x_2, \dots) $ is the probability function of  $\underline{X}$ while $H\left(\underline{X}\right)$ is the entropy of $\underline{X}$. This means  $H\left(\underline{X}\right)=-\sum_{\b{x}} P_{\underline{X}}\left(\b{x}\right) \log{P_{\underline{X}}\left(\b{x}\right)}$ where the $\log{}$ function denotes a logarithm of base $2$ and $\lim_{x \to 0} x\log{(x)} = 0$. Further, we denote the binary entropy of the Bernoulli parameter $p$ as $h_b (p)=-p\log{p}-(1-p)\log{(1-p)}$.

\section{Problem Formulation}
\label{problem_formulation}
\noindent Suppose we are given a random vector $\underline{X}\sim P_{\b{x}}\left(\b{x}\right)$ of dimension $d$ and alphabet size $q$ for each of its components. We are interested in an invertible, not necessarily linear, transformation $\underline{Y}=g(\underline{X})$ such that $\underline{Y}$  is of the same dimension and alphabet size, $g:\{1,\dots,q\}^d \rightarrow \{1,\dots,q\}^d$. In addition we would like the components of $\underline{Y}$ to be as "statistically independent as possible".\\

\noindent The common ICA setup is not limited to invertible transformations (hence $\underline{Y}$ and  $\underline{X}$ may be of different dimensions). However, in our work we focus on this setup as we would like $\underline{Y}=g(\underline{X})$ to be ``lossless" in the sense that we do not lose any information. Further motivation to this setup is discussed in \citep{barlow1989finding, schmidhuber1992learning} and throughout Chapter \ref{Application_to_Data_Compression}. \\   

\noindent Notice that an invertible transformation of a vector $\underline{X}$, where the components $\{X_i\}_{i=1}^d$ are over a finite alphabet of size $q$, is actually a one-to-one mapping (i.e., permutation) of its $q^d$ words. For example, if $\underline{X}$ is over a binary alphabet and is of dimension $d$, then there are $2^d!$ possible permutations of its words. \\

\noindent To quantify the statistical independence among the components of the vector $\underline{Y}$ we use the well-known total correlation measure, which was first introduced by \cite{watanabe1960information} as a multivariate generalization of the mutual information,
\begin{equation}\label{eq:min_criterion}
C(\underline{Y})={\displaystyle \sum_{j=1}^{d}{H(Y_j)}-H(\underline{Y})}.
\end{equation}
This measure can also be viewed as the cost of coding the vector $\underline{Y}$ component-wise, as if its components were statistically independent, compared to its true entropy. Notice that the total correlation is non-negative and equals zero iff the components of $\underline{Y}$ are mutually independent. Therefore, ``as statistically independent as possible" may be quantified by minimizing $C(\underline{Y})$. The total correlation measure was considered as an objective for minimal redundency representation by \cite{barlow1989finding}. It is also not new to finite field ICA problems, as demonstrated by \cite{attux2011immune}. Moreover, we show that it is specifically adequate to our applications, as described in Chapter \ref{Application_to_Data_Compression}. Note that total correlation is also the Kullback-Leibler divergence between the joint probability and the product of its marginals \citep{comon1994independent}.

Since we define $\underline{Y}$ to be an invertible transformation of $\underline{X}$ we have $H(\underline{Y})=H(\underline{X})$ and our minimization objective is
\begin{equation}
{\displaystyle \sum_{j=1}^{d}{H(Y_j)} \rightarrow min.}
\label{eq:sum_ent_min}
\end{equation}\\
\noindent In the following sections we focus on the binary case. The probability function of the vector  $\underline{X}$  is therefore defined by $P(X_1,\ldots,X_d)$ over $m=2^d$ possible words and our objective function is simply
\begin{equation}
{\displaystyle \sum_{j=1}^{d}{h_b(P(Y_j=0))} \rightarrow min.}
\label{eq:sum_ent_min_binary}
\end{equation}
We notice that $P(Y_j=0)$ is the sum of probabilities of all words whose $j^{th}$ bit equals $0$. We further notice that the optimal transformation is not unique. For example, we can always invert the  $j^{th}$ bit of all words, or shuffle the bits, to achieve the same minimum.\\

\iffalse
\noindent The following theorem shows that our minimization objective in (\ref{eq:sum_ent_min_binary}), can be viewed from a different perspective. The proof of this theorem is left for Appendix \ref{proof_of_Dkl_thm}. 

\begin{theorem}
\label{Dkl_thm}
Let $\mathbb{Q}$ be the group of joint probability distributions which can be expressed as a product of marginal bit-wise probability distributions (i.e. distributions with statistically independent components). 

Let $\underline{X} \sim \underline{p}$ be a binary vector of a dimension $d$. Let $\underline{p}^*$ be the probability distribution of $\underline{Y}=g(\underline{X})$ which minimizes (\ref{eq:sum_ent_min_binary}) under invertible permutations. Then, $\underline{p}^*$ also satisfies
$$ \underline{p}^* = \argmin_{\underline{q} \in \mathbb{Q}} D_{kl}(\underline{p}||\underline{q}) $$ 
where $D_{kl}(\underline{p}||\underline{q}) $ is the Kullback-Leibler divergence of $\underline{p}$ and $\underline{q}$, $D_{kl}(\underline{p}||\underline{q})=\sum_i p_i\log{\frac{p_i}{q_i}}$.
\end{theorem}
\noindent This means that our objective (\ref{eq:sum_ent_min_binary}) can also be viewed as finding a probability distribution $\underline{p}^*$, whose components are ``independent" (in the sense that it is composed of statistically independent marginals), and minimizes the $D_{kl}(\underline{p}||\underline{q})$ for all $\underline{q} \in \mathbb{Q}$.  \\  
\fi

\noindent Any approach which exploits the full statistical description of the joint probability distribution
of $\underline{X}$  would require going over all $2^d$ possible words at least once. Therefore, a computational load of at least $O(2^d)$ seems inevitable. Still, this is significantly smaller (and often realistically far more affordable) than $O(2^d!)$, required by brute-force search over all possible permutations. 
Indeed, the complexity of the currently known binary ICA (and factorial codes) algorithms falls within this range. The AMERICA algorithm \citep{yeredor2011independent}, which assumes a XOR mixture, has a complexity of $O(d^2\cdot 2^d)$. The MEXICO algorithm, which is an enhanced version of AMERICA, achieves a complexity of $O(2^d)$ under some restrictive assumptions on the mixing matrix. In \citep{nguyen2011binary} the assumption is that the data was generated over OR mixtures and the asymptotic complexity is $O(d \cdot 2^d)$. There also exist other heuristic methods which avoid an exhaustive search, such as \citep{attux2011immune} for BICA or \citep{schmidhuber1992learning} for factorial codes. These methods, however, do not guarantee convergence to the global optimal solution.  \\

\noindent Looking at the BICA framework, we notice two fundamental a-priori assumptions:
\begin{enumerate}
\item	The vector $\underline{X}$ is a mixture of independent components and there exists an inverse transformation which decomposes these components.
\item	The generative model (linear, XOR field, etc.) of the mixture function is known.

\end{enumerate}
In this work we drop these assumptions and solve the ICA problem over finite alphabets with no prior assumption on the vector $\underline{X}$. 
As a first step towards this goal, let us drop Assumption $2$ and keep Assumption $1$, stating that underlying independent components do exist. The following combinatorial algorithm proves to solve this problem, over the binary alphabet, in $O(d \cdot 2^d)$ computations.

\section{Generalized BICA with Underlying Independent Components}
\label{BICA with Underlying Independent Components}
\noindent In this section we assume that underlying independent components exist. In other words, we assume there exists a permutation $\underline{Y}=g(\underline{X})$  such that the vector $\underline{Y}$ is statistically independent $P(Y_1,\ldots,Y_d)=\prod_{i=1}^{d}{P(Y_j)}$. Denote the marginal probability of the $j^{th}$ bit equals $0$ as $\pi_j=P(Y_j=0)$. Notice that by possibly inverting bits we may assume every $\pi_j$ is at most $\frac{1}{2}$ and by reordering we may have,
without loss of generality, that $\pi_d \leq  \pi_{d-1} \leq  \cdots  \leq  \pi_1 \leq 1/2$. In addition, we assume a non-degenerate setup where $\pi_d>0$. For simplicity of presentation, we first analyze the case where $\pi_d <  \pi_{d-1} <  \cdots  <  \pi_1 \leq 1/2$. This is easily generalized to the case where several $\pi_j$ may equal, as discussed later in this section. \\

%Denote the given $N=2^n$ probabilities $P(X_1,\ldots,X_n)$ as  $p_i$ for every  $i=1\ldots N$, and then sort them in an ascending order such that $p_1\leq p_2 \leq \cdots \leq p_N$.
\noindent Denote the $m=2^d$ probabilities of $P(\underline{Y}=\underline{y} )$
as $p_1, p_2, \ldots , p_m$, assumed to be ordered so that $p_1 \leq p_2 \leq \cdots \leq 
p_m$.
We first notice that the probability of the all-zeros word, $P(Y_d=0, Y_{d-1}=0, \ldots ,Y_1=0)=\prod_{j=1}^{d}{\pi_j}$  is the smallest possible probability since all parameters are not greater than 0.5. Therefore we have $p_1=\prod_{j=1}^{d}{\pi_j}$.\\

\noindent Since $\pi_1$ is the largest parameter of all $\pi_j$, the second smallest probability is just $P(Y_d=0,\ldots,Y_2=0,Y_1=1)=\pi_d \cdot \pi_{d-1} \cdot \ldots \cdot \pi_2 \cdot (1-\pi_1 )=p_2$. Therefore we can recover $\pi_1$ from $\frac{1-\pi_1}{\pi_1} =\frac{p_2}{p_1}$,  leading to $\pi_1=\frac{p_1}{p_1+p_2}$.
We can further identify the third smallest probability as $p_3=\pi_d \cdot \pi_{d-1} \cdot \ldots \cdot \pi_3 \cdot (1-\pi_2) \cdot \pi_1$. This leads to $\pi_2=\frac{p_1}{p_1+p_3}$.\\

\noindent However, as we get to $p_4$ we notice we can no longer uniquely identify its components; it may either equal $\pi_d \cdot \pi_{d-1} \cdot \ldots \cdot \pi_3 \cdot (1-\pi_2) \cdot (1-\pi_1)$ or $\pi_d \cdot \pi_{d-1} \cdot \ldots \cdot (1-\pi_3 )\cdot \pi_2 \cdot \pi_1$. This ambiguity is easily resolved since we can compute the value of $\pi_d \cdot \pi_{d-1} \cdot \ldots \cdot \pi_3 \cdot (1-\pi_2) \cdot (1-\pi_1)$ from the parameters we already found and compare it with $p_4$. Specifically, If $\pi_d \cdot \pi_{d-1} \cdot \ldots \cdot \pi_3 \cdot (1-\pi_2) \cdot (1-\pi_1) \neq p_4$ then we necessarily have $\pi_d \cdot \pi_{d-1} \cdot \ldots \cdot (1-\pi_3 )\cdot \pi_2 \cdot \pi_1 = p_4$ from which we can recover $\pi_3$. Otherwise $\pi_d \cdot \pi_{d-1} \cdot \ldots \cdot (1-\pi_3 )\cdot \pi_2 \cdot \pi_1 = p_5$ from which we can again recover $\pi_3$ and proceed to the next parameter.\\

\noindent Let us generalize this approach. Denote $\Lambda_k$ as a set of probabilities of all words whose $(k+1)^{th}, \dots, d^{th}$ bits are all zero.
\begin{theorem} 
\label{theorem1_Generalized_BICA}
Let $i$ be an arbitrary index in $\{1,2,\dots,m\}$.
Assume we are given that $p_i$, the $i^{th}$ smallest probability in a given set of probabilities, satisfies the following decomposition 
\begin{equation}\nonumber
p_i=\pi_d \cdot \pi_{d-1} \cdot \ldots \cdot \pi_{k+1} \cdot (1-\pi_k) \cdot \pi_{k-1}  \cdot \ldots \cdot ∙\pi_1.
\end{equation}
 Further assume the values of $\Lambda_{k-1}$ are all given in a sorted manner. Then the complexity of finding the value of $\pi_d \cdot \pi_{d-1} \cdot \ldots \cdot \pi_{k+2} \cdot (1-\pi_{k+1}) \cdot \pi_k \cdot \ldots \cdot ∙\pi_1$, and calculating and sorting the values of $\Lambda_k$  is $O(2^k)$.
\end{theorem}

\begin{proof} 
Since the values of  $\Lambda_{k-1}$ and $\pi_k$ are given we can calculate the values which are still missing to know $\Lambda_k$ entirely by simply multiplying each element of $\Lambda_{k-1}$ by $\frac{1-\pi_k}{\pi_k}$.  Denote this set of values as $\bar{\Lambda}_{k-1}$. Since we assume the set $\Lambda_{k-1}$ is sorted then $\bar{\Lambda}_{k-1}$ is also sorted and the size of each set is $2^{k-1}$. Therefore, the complexity of sorting $\Lambda_{k}$ is the complexity of merging two sorted lists, which is $O(2^k)$.\\

\noindent In order to find the value of $\pi_d \cdot \pi_{d-1} \cdot \ldots \cdot \pi_{k+2} \cdot (1-\pi_{k+1}) \cdot \pi_k \cdot \ldots \cdot ∙\pi_1$ we need to go over all the values which are larger than $p_i$ and are not in $\Lambda_k$. However, since both the list of all $m$ probabilities and the set $\Lambda_k$ are sorted we can perform a binary search to find the smallest entry for which the lists differ. The complexity of such search is  $O(\log⁡{(2^k)})=O(k)$ which is smaller than $O(2^k)$. Therefore, the overall complexity is $O(2^k)$ 
\end{proof}

\noindent Our algorithm is based on this theorem. We initialize the values $p_1, p_2$ and $\Lambda_{1}$, and   
for each step $k=3 \ldots d$ we calculate $\pi_d \cdot \pi_{d-1} \cdot \ldots \cdot \pi_{k+1} \cdot (1-\pi_{k}) \cdot \pi_{k-1} \cdot \ldots \cdot ∙\pi_1$ and $\Lambda_{k-1}$. \\

\noindent The complexity of our suggested algorithm is therefore $\sum_{k=1}^{d}{O(2^k)}=O(d\cdot2^d)$. 
However, we notice that by means of the well-known quicksort algorithm \citep{Hoare:1961:AQ:366622.366644}, the complexity of our preprocessing sorting phase is $O(m\log{(m)})=O(d\cdot2^d)$.
Therefore, in order to find the optimal permutation we need $O(d\cdot2^d)$ for sorting the given probability list and $O(2^d)$ for extracting the parameters of $P(X_1, \ldots ,X_d)$.\\

\noindent Let us now drop the assumption that the values of $\pi_j$'s are non-equal. That is,  $\pi_d \leq  \pi_{d-1} \leq  \cdots  \leq  \pi_1 \leq 1/2$. It may be easily verified that both
Theorem \ref{theorem1_Generalized_BICA} and our suggested algorithm still hold, with the difference that instead of choosing the \textit{single} smallest entry in which the probability lists differ, we may choose \textit{one of the} (possibly many) smallest entries. This means that instead of recovering the unique value of $\pi_k$ at the $k^{th}$ iteration (as the values of $\pi_j$'s are assumed non-equal), we recover the $k^{th}$ smallest value in the list  $\pi_d \leq  \pi_{d-1} \leq  \cdots  \leq  \pi_1 \leq 1/2$.   \\

\noindent Notice that this algorithm is combinatorial in its essence and is not robust when dealing with real data. In other words, the performance of this algorithm strongly depends on the accuracy of $P(X_1, \ldots ,X_d)$ and does not necessarily converge towards the optimal solution when applied on estimated probabilities.  \\

\section{Generalized BICA via Search Tree Based Algorithm}
\label{Generalized BICA via Search Tree Based Algorithm}

We now turn to the general form of our problem (\ref{eq:min_criterion}) with no further assumption on the vector $\underline{X}$.

\noindent We denote $\Pi_j$= \{all words whose $j^{th}$ bit equals 0\}. In other words, $\Pi_j$ is the set of words that ``contribute" to $P(Y_j=0)$. We further denote the set of $\Pi_j$ that each word is a member of as $\Gamma_i$ for all $i=1\ldots m$ words. For example, the all zeros word $\{00\ldots0\}$ is a member of all $\Pi_j$ hence $\Gamma_1=\{ \Pi_1,\ldots,\Pi_d\}$. We define the optimal permutation as the permutation of the $m$ words that achieves the minimum of $C(\underline{Y})$ such that $\pi_j \leq 1/2$ for every $j$.\\ %Notice that this representation can always be met by inverting the $i^{th}$ bit of all the words in $\Pi_i$ for every $\pi_i$ which is greater than 1/2 so that the entropy is maintained. %We denote the optimal $\pi_i$ as $\pi_i^*$. Therefore our goal is to minimize $C(\underline{Y})$, as much as we can, under permutations of $\{p_i\}_{i=1}^{N}$.

\noindent Let us denote the binary representation of the $i^{th}$ word with $\underline{y}(i)$.
Looking at the $m$ words of the vector $\underline{Y}$ we say that a word $\underline{y}(i)$ is majorant to $\underline{y}(l)$ $(\underline{y}(i) \succeq \underline{y}(l))$ if  $\Gamma_l \subset \Gamma_i$. 
In other words, $\underline{y}(i)$ is majorant to $\underline{y}(l)$ iff for every bit in $\underline{y}(l)$ that equals zeros,  the same bit equals zero in $\underline{y}(i)$.  
In the same manner a word $\underline{y}(i)$ is minorant to $\underline{y}(l)$  $(\underline{y}(i)\preceq \underline{y}(l))$ if  $\Gamma_i \subset \Gamma_l$, that is iff for every bit in $\underline{y}(i)$ that equals zeros, the same bit equals zero in $\underline{y}(l)$. For example, the all zeros word $\{00\ldots0\}$ is majorant to all the words, while the all ones word $\{11\ldots1\}$ is minorant to all the word as none of its bits equals zeros. \\

\noindent We say that $\underline{y}(i)$ is a largest minorant to $\underline{y}(l)$ if there is no other word that is minorant to $\underline{y}(l)$ and majorant to $\underline{y}(i)$. We also say that there is a partial order between $\underline{y}(i)$  and $\underline{y}(l)$ if one is majorant or minorant to the other.  \\

\begin{theorem}  \label{th:partial_order}
The optimal solution must satisfy $P(\underline{y}(i))\geq P(\underline{y}(l))$ for all $\underline{y}(i)\preceq \underline{y}(l)$.
\end{theorem}

\begin{proof} 
Assume there exists $\underline{y}(i) \preceq \underline{y}(l), i \neq l$ such that $P(\underline{y}(i))<P(\underline{y}(l))$ which achieves the lowest (optimal) $C(\underline{Y})$. Since $\underline{y}(i) \preceq \underline{y}(l)$ then, by definition, $\Gamma_i \subset \Gamma_l$.  This means there exists $\Pi_{j^*}$ which satisfies $\Pi_{j^*} \in \Gamma_l \setminus \Gamma_i$. Let us now exchange (swap) the words $\underline{y}(i)$ and $\underline{y}(l)$. Notice that this swapping only modifies $\Pi_{j^*}$ but leaves all other $\Pi_j$'s untouched. Therefore this swap leads to a lower $C (\underline{Y})$ as the sum in (\ref{eq:sum_ent_min_binary}) remains untouched apart from its ${j^*}^{th}$ summand which is lower than before. This contradicts the optimality assumption
\end{proof}

\noindent We are now ready to present our algorithm. As a preceding step let us sort the probability vector $\underline{p}$ (of \underline{X}) such that $p_i \leq p_{i+1}$.
As described above, the all zeros word is majorant to all words and the all ones word is minorant to all words. Hence, the smallest probability $p_1$ and the largest probability $p_m$ are allocated to them respectively, as Theorem \ref{th:partial_order} suggests. We now look at all words that are largest minorants to the all zeros word.\\

\noindent Theorem \ref{th:partial_order} guarantees that $p_2$ must be allocated to one of them. We shall therefore examine all of them. This leads to a search tree structure in which every node corresponds to an examined allocation of $p_i$. In other words, for every allocation of $p_i$ we shall further examine the allocation of $p_{i+1}$ to each of the largest minorants that are still not allocated.  This process ends once all possible allocations are examined.\\

\noindent The following example (Figure \ref{fig:exhaustive}) demonstrates our suggested algorithm with $d=3$. The algorithm is initiated with the allocation of $p_1$ to the all zeros word. In order to illustrate the largest minorants to $\{000\}$ we use the chart of the partial order at the bottom left of Figure \ref{fig:exhaustive}. As visualized in the chart, every set $\Pi_j$ is encircled by a different shape (e.g. ellipses, rectangles) and the largest minorants to $\{000\}$ are $\{001\}$, $\{010\}$ and $\{100\}$. As we choose to investigate the allocation of $p_2$ to $\{001\}$ we notice that remaining largest minorants, of all the words that are still not allocated, are $\{010\}$ and $\{100\}$. We then investigate the allocation of $p_3$ to $\{010\}$, for example, and continue until all $p_i$ are allocated.  
\begin{figure}[h!]
\centering
\includegraphics[width = 0.7\textwidth,bb= 50 469 550 715,clip]{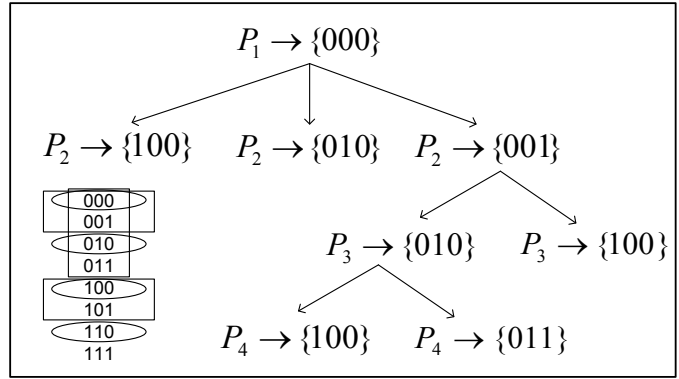}
\caption{Search tree based algorithm with $d=3$}
\label{fig:exhaustive}
\end{figure}

\noindent This search tree structure can be further improved by introducing a depth-first branch and bound enhancement. This means that before we examine a new branch in the tree we bound the minimal objective it can achieve (through allocation of the smallest unallocated probability to all of its unallocated words for example). \\ 

\noindent The asymptotic computational complexity of this branch and bound search tree is quite involved to analyze. However, there are several cases where a simple solution exists (for example, for $d=2$ it is easy to show that the solution is to allocate all four probabilities in ascending order).\\

\section{Generalized BICA via Piecewise Linear Relaxation Algorithm}
\label{The Relaxed Generalized BICA}

In this section we present a different approach which bounds the optimal solution from above as tightly as we want in $O(d^k \cdot 2^d)$ operations, where $k$ defines how tight the bound is. Throughout this section we assume that $k$ is a fixed value, for complexity analysis purposes.  \\

\noindent Let us first notice that the problem we are dealing with (\ref{eq:sum_ent_min_binary}) is a concave minimization problem over a discrete permutation set which is hard. However, let us assume for the moment that instead of our ``true" objective (\ref{eq:sum_ent_min_binary}) we have a simpler linear objective function. That is, 
\begin{equation} \label{eq:linear}
{\displaystyle L(\underline{Y})=\sum_{j=1}^{d}{a_j \pi_j+b_j}=\sum_{i=1}^{m}{c_i P(\underline{Y}=\underline{y}(i))}+d_0}
\end{equation}
where the coefficients $a_j,b_j,c_i,d_0$ correspond to different slopes and intersects that are later defined. Notice that the last equality changes the summation from over $d$ components to a summation over all $m=2^d$ words (this change of summation is further discussed in Section \ref{matrix-vector}).\\

\noindent In order to minimize this objective over the $m=2^d$ given probabilities $\underline{p}$ we simply sort these probabilities in a descending order and allocate them such that the largest probability goes with the smallest coefficient $c_i$ and so on. Assuming both the coefficients and the probabilities are known and sorted in advance, the complexity of this procedure is linear in $m$.\\

\noindent We now turn to the generalized binary ICA problem as defined in (\ref{eq:sum_ent_min_binary}). Since our objective is concave we would first like to bound it from above with a piecewise linear function which contains $k$ pieces, as shown in Figure \ref{fig:piecewise linear}. In this work we do not discuss the construction of such upper-bounding piecewise linear function, nor tuning methods for the value of $k$, and assume this function is given for any fixed $k$. Notice that the problem of approximating concave curves with piecewise linear functions is very well studied (for example, by \cite{gavrilovic1975optimal}) and may easily be modified to the upper bound case.  We show that solving the piecewise linear problem approximates the solution to (\ref{eq:sum_ent_min_binary}) as closely as we want, in significantly lower complexity.\\ 

\begin{figure}[h]
\centering
\includegraphics[width = 0.80\textwidth,bb= 50 140 780 580,clip]{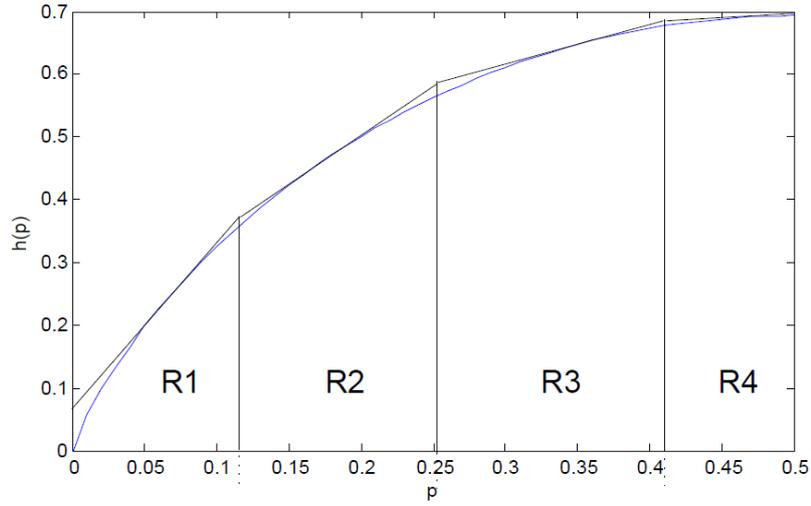}
\caption{piecewise linear ($k=4$) relaxation to the binary entropy}
\label{fig:piecewise linear}
\end{figure}

\noindent From this point on we shall drop the previous assumption that $\pi_d \leq  \pi_{d-1} \leq  \cdots  \leq  \pi_1$, for simplicity of presentation. 
First, we notice that all $\pi_j's$ are equivalent (in the sense that we can always interchange them and achieve the same result). This means we can find the optimal solution to the piecewise linear problem by going over all possible combinations of “placing” the $d$ variables $\pi_j$ in the $k$ different regions of the piecewise linear function. For each of these combinations we need to solve a linear problem (such as in (\ref{eq:linear}), where the minimization is with respect to allocation of the  $m$ given probabilities $\underline{p}$) with additional constraints on the ranges of each $\pi_j$. For example, assume $d=3$ and the optimal solution is such that two $\pi_j's$ (e.g. $\pi_1$ and $\pi_2$) are at the first region, $R_1$, and $\pi_3$ is at the second region, $R_2$. Then, we need to solve the following constrained linear problem,
\begin{equation} \label{eq:linear_program}
\begin{aligned}
& {\text{minimize}}
& &  a_1 \cdot (\pi_1+\pi_2)+2b_1+a_2\cdot \pi_3 +b_2 \\
& \text{subject to}
& & \pi_1, \pi_2 \in R_1 , \pi_3 \in R_2
\end{aligned}
\end{equation}
where the minimization is over the allocation of the given $\left\{p_i\right\}_{i=1}^{m}$, which determine the corresponding $\pi_j$'s, as demonstrated in (\ref{eq:linear}). This problem is again hard. However, if we attempt to solve it without the constraints we notice the following:

\begin{enumerate}

\item	If the collection of $\pi_j's$ which define the optimal solution to the unconstrained linear problem happens to meet the constraints then it is obviously the optimal solution with the constraints.
\item	If the collection of $\pi_j's$ of the optimal solution do not meet the constraints (say, $\pi_2 \in R_2$) then, due to the concavity of the entropy function, there exists a different combination with a different constrained linear problem (again, over the allocation of the $m$ given probabilities $\underline{p}$), 
\begin{equation*} \label{eq:linear_program}
\begin{aligned}
& {\text{minimize}}
& &  a_1\pi_1+b_1+a_2(\pi_2+\pi_3)+2b_2 \\
& \text{subject to}
& & \pi_1 \in R_1  \;  \pi_2,\pi_3 \in R_2
\end{aligned}
\end{equation*}

in which this set of $\pi_j's$ necessarily achieves a lower minimum (since $a_2 x+b_2<a_1 x+b_1$ $\forall x \in R_2$).

\end{enumerate}
Therefore, in order to find the optimal solution to the piecewise linear problem, all we need to do is to go over all possible combinations of placing the $\pi_j's$ in the $k$ different regions, and for each combination solve an unconstrained linear problem (which is solved in a linear time in $m$). If the solution does not meet the constraint then it means that the assumption that the optimal $\pi_j$ reside within this combination's regions is false. Otherwise, if the solution does meet the constraint, it is considered as a candidate for the global optimal solution. \\

\noindent The number of combinations is equivalent to the number of ways of placing $d$ identical balls in $k$ boxes, which is for a fixed $k$, 
\begin{align}
\left(\begin{array}{c} 
d+k-1\\n\end{array}\right)=\left(\begin{array}{c} 
d+k-1\\k-1\end{array}\right)\leq
 \frac{(d+k-1)^{k-1}}{(k-1)!} = O(d^k).
\end{align}
Assuming the coefficients are all known and sorted in advance, for any fixed $k$ the overall asympthotic complexity of our suggested algorithm, as $d \rightarrow \infty$, is simply $O(d^k \cdot 2^d)$.   \\

\subsection{The Relaxed Generalized BICA as a single matrix-vector multiplication}
\label{matrix-vector}
It is important to notice that even though the asymptotic complexity of our approximation algorithm is $O(d^k\cdot2^d)$, it takes a few seconds to run an entire experiment on a standard personal computer for as much as $d=10$, for example. The reason is that the $2^d$ factor refers to the complexity of sorting a vector and multiplying two vectors,  
operations which are computationally efficient on most available software. Moreover, if we assume that the coefficients in (\ref{eq:linear}) are already calculated, sorted and stored in advance, we can place them in a matrix form $A$ and multiply the matrix with the (sorted) vector $\underline{p}$. The minimum of this product is exactly the solution to the linear approximation problem. Therefore, the practical complexity of the approximation algorithm drops to a single multiplication of a ($d^k \times 2^d$) matrix with a ($2^d \times 1$) vector.\\

\noindent Let us extend the analysis of this matrix-vector multiplication approach. 
Each row of the matrix $A$ corresponds to a single coefficient vector to be sorted and multiplied with the sorted probability vector $\underline{p}$. 
Each of these coefficient vectors correspond to one possible way of placing $d$ components in $k$ different regions of the piecewise linear function. Specifically, in each row, each of the $d$ components is assumed to reside in one of the $k$ regions, hence it is assigned a slope $a_j$ as indicated in (\ref{eq:linear}). For each row, our goal is to minimize $L(\underline{Y})$. Since this minimization is solved over the vector $\underline{p}$ we would like to change the summation accordingly. To do so, each entry of the coefficient vector (denoted as $c_i$ in (\ref{eq:linear})) is calculated by summing all the slopes that correspond to each $\pi_j$. For example, let us assume $d=3$ where $\pi_1,\pi_2 \in R_1$, with a corresponding slope $a_1$ and intercept $b_1$, while the $\pi_3 \in R_2$ with $a_2$ and $b_2$. We use the following mapping: $P(\underline{Y}=000)=p_1, P(\underline{Y}=001)=p_2, \dots , P(\underline{Y}=111)=p_8$. Therefore

\begin{equation}
\begin{array}{c} 
\pi_1=P(Y_1=0)=p_1+p_2+p_3+p_4\\
\pi_2=P(Y_2=0)=p_1+p_2+p_5+p_6\\
\pi_3=P(Y_3=0)=p_1+p_3+p_5+p_7\\
\end{array}.
\end{equation}
The corresponding coefficients $c_i$ are then the sum of rows of the following matrix

\begin{equation}
A=\left(
\begin{array}{ccc} 
a_1&a_1& a_2\\
a_1& a_1& 0\\
a_1& 0& a_2\\
a_1& 0& 0\\
0& a_1& a_2\\
0& a_1& 0\\
0& 0& a_2\\
0& 0& 0\\
\end{array}\right).
\end{equation}
This leads to a minimization problem 
\begin{align}
 &  L(\underline{Y})=\sum_{j=1}^{d}{a_j \pi_j+b_j}=a_1(\pi_1+\pi_2)+a_2\pi_3+2b_1+b_2= \\[0.5em] \nonumber
& (2a_1+a_2)p_1+2a_1p_2+(a_1+a_2)p_3+a_1p_4+(a_1+a_2)p_5+a_1p_6+a_2p_7+2b_1+b_2 \nonumber
\end{align}
where the coefficients of $p_i$ are simply the sum of the $i^{th}$ row in the matrix $A$.\\
%to be solved over a permutation set $\{p_i\}_{i=1}^{N}$, as discussed above.

\noindent Now let us assume that $d$ is greater than $k$ (which is usually the case). It is easy to see that many of the coefficients $c_i$ are actually identical in this case. Precisely, let us denote by $l_v$ the number of assignments for the $v^{th}$ region, where $v=\{1\dots k\}$. Then, the number of unique $c_i$ coefficients is simply 
\begin{equation*}
\prod_{v=1}^{k}(l_v+1)-1
\end{equation*}
subject to $\sum_{v=1}^{k}{l_v}=d$.
Since we are interested in the worst case (of all rows of the matrix $A$), we need to find the non-identical coefficients. This is obtained when $l_v$ is as "uniform" as possible. Therefore we can bound the number of non-identical coefficients from above by temporarily dropping the assumption that $l_v$'s are integers and letting $l_v=\frac{d}{k}$ so that
\begin{equation}
 \max{\prod_{v=1}^{k}(l_v+1)} \leq \left(\frac{d}{k}+1\right)^{k}=O(d^k).
 \end{equation}
This means that instead of sorting the $2^d$ coefficients for each row of the matrix $A$, we only need to sort $O(d^k)$ coefficients. \\

\noindent Now, let us further assume that the data is generated from some known parametric model. In this case, some probabilities $p_i$ may also be identical, so that the probability vector $\underline{p}$ may also not require $O(d \cdot 2^d)$ operations to be sorted. For example, if we assume a block independent structure, such that $d$ components (bits) of the data are generated from $\frac{d}{r}$ independent and identically distributed blocks of of size $r$, then it can be shown that the probability vector $\underline{p}$ contains at most
\begin{equation}
 \left(\begin{array}{c} 
\frac{d}{r}+2^r-1\\\frac{d}{r}\end{array}\right)=O\left({\left(\frac{d}{r}\right)}^{2^r}\right)
\end{equation}
non-identical elements $p_i$. Another example is a first order stationary symmetric  Markov model. In this case there only exists a quadratic number, $ d\cdot(d-1)+2=O(d^2) $, of non-identical probabilities in $\underline{p}$ (see Appendix \ref{unique_values_markov_model}).\\

\noindent This means that applying our relaxed generalized BICA on such datasets may only require $O(d^k)$ operations for the matrix $A$ and a polynomial number of operations (in $d$) for the vector $\underline{p}$; hence our algorithm is reduced to run in a polynomial time in $d$.\\ 

\noindent Notice that this derivation only considers the number of non-identical elements to be sorted through a quicksort algorithm. However, we also require the degree of each element (the number of times it appears) to eventually multiply the matrix $A$ with the vector $\underline{p}$. This, however, may be analytically derived through the same combinatorical considerations described above. \\

\subsection{Relaxed Generalized BICA Illustration and Experiments}

In order to validate our approximation algorithm we conduct several experiments. In the first experiment we illustrate the convergence of our suggested scheme as $k$ increases. We arbitrarily choose a probability distribution with $d=10$ statistically independent components and mix its components in a non-linear fashion. We apply the approximation algorithm on this probability distribution with different values of $k$ and compare the approximated minimum entropy we achieve (that is, the result of the upper-bound piecewise linear cost function) with the entropy of the vector. In addition, we apply the estimated parameters $\pi_j$ on the true objective (\ref{eq:sum_ent_min_binary}), to obtain an even closer approximation. Figure \ref{fig:experiment} demonstrates the results we achieve, showing the convergence of the approximated entropy towards the real entropy as the number of linear pieces increases.  As we repeat this experiment several times (that is, arbitrarily choose a probability distributions and examine our approach for every single value of $k$), we notice that the estimated parameters are equal to the independent parameters for $k$ as small as $4$, on the average.\\
     
\begin{figure}[h]
\centering
\includegraphics[width = 0.7\textwidth,bb= 50 210 560 575,clip]{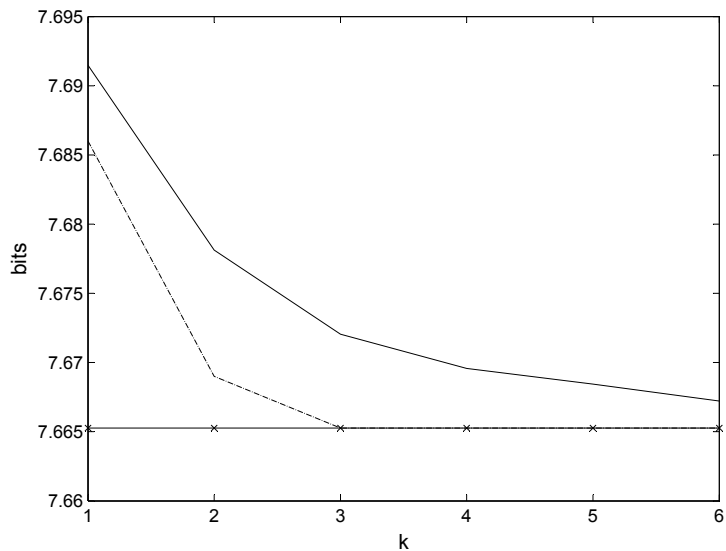}
\caption{Piecewise linear approximation (solid line), entropy according to the estimated parameters (dashed-dot line) and the real entropy (horizontal line with the X's),  for a vector size $d=10$ and different $k$ linear pieces}
\label{fig:experiment}
\end{figure}
%In the second experiment we generate a probability distribution of size $n=6$ and compare our approximation algorithm with an exhaustive search algorithm. As before, our approximation algorithm was able to find the optimal permutation that minimizes (\ref{eq:min_criterion}) quite quickly, in $k=3$ pieces on average.

\noindent We further illustrate the use of the BICA tool by the following example on ASCII code. The ASCII code is a common standardized eight bit representation of western letters, numbers and symbols. We gather statistics on the frequency of each character, based on approximately 183 million words that appeared in the New York Times magazine \citep{jones2004case}. We then apply the BICA (with $k=8$, which is empirically sufficient) on this estimated probability distribution, to find a new eight bit representation of characters, such that the bits are "as statistically independent" as possible. We find that the entropy of the joint probability distribution is $4.8289$ bits, the sum of marginal entropies using ASCII representation is $5.5284$ bits and the sum of marginal entropies after applying BICA is just $4.8532$ bits. This means that there exists a different eight bit representation of characters which allows nearly full statistical independence of bits. Moreover, through this representation one can encode each of the eight bit separately without losing more than $0.025$ bits, compared to encoding the eight bits altogether.\\

\section{Generalized ICA Over Finite Alphabets} 
\label{GICA over FF}
\subsection{Piecewise Linear Relaxation Algorithm - Exhaustive Search}
\label{Exhaustive Search}
Let us extend the notation of the previous sections, denoting the number of components as $d$ and the alphabet size as $q$. We would like to minimize $\sum_{j=1}^{d}{H(Y_j)}$ where $Y_j$ is over an alphabet size $q$.
We first notice that we need $q-1$ parameters to describe the multinomial distribution of $Y_j$ such that all of the parameters are not greater than $\frac{1}{2}$. Therefore, we can bound from above the marginal entropy with a piecewise linear function in the range $[0,\frac{1}{2}]$, for each of the parameters of $Y_j$.  
We refer to a ($q-1$)-tuple of regions as \textit{cell}.
As in previous sections we consider $k$, the number of linear pieces, to be fixed. Notice however, that as $q$ and $d$ increase, $k$ needs also to take greater values in order to maintain the same level of accuracy. As mentioned above, in this work we do not discuss methods to determine the value of $k$ for given $q$ and $d$, and empirically evaluate it.\\ 

\noindent Let us denote the number of cells to be visited in our approximation algorithm (Section \ref{The Relaxed Generalized BICA}) as $C$. Since each parameter is approximated by $k$ linear pieces and there are $q-1$ parameters, $C$ equals at most $k^{q-1}$. In this case too, the parameters are exchangeable (in the sense that the entropy of a multinomial random variable with parameters $\{p_1,p_2,p_3\}$ is equal to the entropy of a multinomial random variable with parameters $\{p_2,p_1,p_3\}$, for example). Therefore, we do not need to visit all $k^{q-1}$ cells, but only a unique subset which disregards permutation of parameters. In other words, the number of cells to be visited is bounded from above by the number of ways of choosing $q-1$ elements (the parameters) out of $k$ elements (the number of pieces in each parameter) with repetition and without order. Notice this upper-bound (as opposed to full equality) is a result of not every combination being a feasible solution, as the sum of parameters may exceed $1$. Assuming $k$ is fixed and as $q \rightarrow \infty$ this equals 
\begin{align}
\label{q^k}
\left(\begin{array}{c} 
q-1+k-1\\q-1\end{array}\right)=&\left(\begin{array}{c} q-1+k-1\\k-1\end{array}\right)\leq
\frac{(q-1+k-1)^{k-1}}{(k-1)!}=O(q^k).
\end{align}
Therefore, the number of cells we are to visit is simply $C=\min\left(k^{q-1},O(q^k)\right)$. For sufficiently large $q$ it follows that $C=O\left(q^k\right)$.
As in the binary case we would like to examine all combinations of $d$ entropy values in $C$ cells. The number of iterations to calculate all possibilities is equal to the number of ways of placing $d$ identical balls in $C$ boxes, which is
\begin{equation}
\label{d^C}
\left(\begin{array}{c} 
d+C-1\\d\end{array}\right)=O\left(d^C\right).
\end{equation}
In addition, in each iteration we need to solve a linear problem which takes a linear complexity in $q^d$. Therefore, the overall complexity of our suggested algorithm is $O\left(d^{C}\cdot q^d \right)$.\\

\noindent We notice however that for a simple case where only two components are mixed $(d=2)$, we can calculate (\ref{d^C}) explicitly
\begin{equation}
\left(\begin{array}{c} 
2+C-1\\2\end{array}\right)=\frac{C(C+1)}{2}.
\end{equation}
Putting this together with (\ref{q^k}), leads to an overall complexity which is polynomial in $q$, for a fixed $k$,
\begin{equation}
\left(\frac{q^k(q^k+1)}{2}q^2\right)=O\left(q^{2k+2}\right).
\end{equation}\\
\noindent Either way, the computational complexity of our suggested algorithm may result in an excessive runtime, to a point of in-feasibility, in the case of too many components or an alphabet size which is too large. This necessitates a heuristic improvement to reduce the runtime of our approach.

\subsection{Piecewise Linear Relaxation Algorithm - Objective Descent Search}
\label{objective descent search}
In Section \ref{The Relaxed Generalized BICA} we present the basic step of our suggested piecewise linear relaxation to the generalized binary ICA problem. As stated there, for each combination of placing $d$ components in $k$ pieces (of the piecewise linear approximation function) we solve a linear problem (LP). Then, if the solution happens to meet the constraints (falls within the ranges we assume) we keep it. Otherwise, due to the concavity of the entropy function, there exists a different combination with a different constrained linear problem in which this solution that we found necessarily achieves a lower minimum, so we disregard it.\\

\noindent This leads to the following objective descent search method: instead of searching over all possible combinations we shall first guess an initial combination as a starting point (say, all components reside in a single cell). We then solve its unconstrained LP. If the solution meets the constraint we terminate. Otherwise we visit the cell that meets the constraints of the solution we found. We then solve the unconstrained LP of that cell and so on. We repeat this process for multiple random initialization. 

\begin{algorithm}
\caption{Relaxed Generalized ICA For Finite Alphabets via Gradient Search}
\begin{algorithmic} [1]
\REQUIRE $\underline{p}$ = the probability function of the random vector $\underline{X}$
\REQUIRE $d$ = the number of components of $\underline{X}$
\REQUIRE $C$ = the number of cells which upper-bound the objective.
\REQUIRE $I$ = the number of initializations.
\STATE $opt \leftarrow \infty$, where the variable $opt$ is the minimum sum of marginal entropies we are looking for. 
\STATE $V \leftarrow \emptyset$, where $V$ is the current cells the algorithm is visiting.
\STATE $S\leftarrow \infty$, where $S$ is the solution of the current LP.

\STATE $i \leftarrow 1$.
\WHILE{$i \leq I$}

\STATE $V\leftarrow$ randomly select an initial combination of placing $d$ components in $C$ cells 
\STATE $S\leftarrow$ LP($V$) solve an unconstrained linear prograsm which corresponds to the selected combination, as appears in (\ref{eq:linear}). \label{marker}
\IF {the solution falls within the bounds of the cell} 
\IF {$H(S)<opt$} 
\STATE $opt\leftarrow H(S)$, the sum of marginal entropies which correspond to the parameters found by the LP 
\ENDIF
\STATE $i\leftarrow i+1$
\ELSE 
\STATE $V \leftarrow$ the cells in which $S$ reside.   
\STATE \textbf{goto 7}
\ENDIF
\ENDWHILE
\RETURN $opt$
\end{algorithmic}
\label{alg:algorithm}
\end{algorithm}

\noindent This suggested algorithm is obviously heuristic, which does not guarantee to provide the global optimal solution. Its performance strongly depends on the number of random initializations and the concavity of the searched domain.\\ 

\noindent The following empirical evaluation demonstrates our suggested approach.
In this experiment we randomly generate a probability distribution with $d$ independent and identically distributed components over an alphabet size $q$. We then mix its components in a non-linear fashion. We apply the objective descent algorithm with a fixed number of initialization points ($I=1000$) and compare the approximated minimum sum of the marginal entropies with the true entropy of the vector.
Figure \ref{fig:experiment_q4} demonstrates the results we achieve for different values of $d$. We see that the objective descent algorithm approximates the correct components well for smaller values of $d$ but as $d$ increases the difference between the approximated minimum and the optimal minimum increases, as the problem becomes too involved. 
\begin{figure}[h]
\centering
\includegraphics[width = 0.7\textwidth,bb= 50 180 550 590,clip]{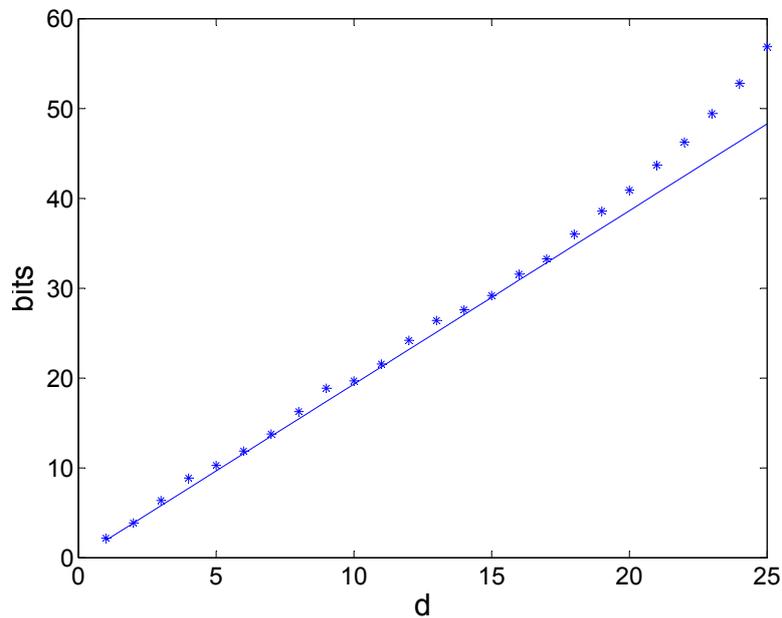}
\caption{The real entropy (solid line) and the sum of marginal entropies as discovered by the objective descent algorithm, for an i.i.d vector over an alphabet size $q=4$ and of varying number of components $d$}
\label{fig:experiment_q4}
\end{figure}

\section{Application to Blind Source Separation}

Assume there exist $d$ independent (or practically "almost independent") sources where each source is over an alphabet size $q$. These sources are mixed in an invertible, yet unknown manner. Our goal is to recover the sources from this mixture. \\

\noindent For example, consider a case with $d=2$ sources $X_1,X_2$, where each source is over an alphabet size $q$. The sources are linearly mixed (over a finite field) such that $Y_1=X_1, Y_2=X_1+X_2$. However, due to a malfunction, the symbols of $Y_2$ are randomly shuffled, before it is handed to the receiver.
Notice this mixture (including the malfunction) is unknown to the receiver, who receives $Y_1,Y_2$ and strives to ``blindly" recover $X_1,X_2$. 
 In this case any linearly based method such as \citep{yeredor2011independent} or \citep{attux2011immune} would fail to recover the sources as the mixture, along with the malfunction, is now a non-linear invertible transformation. Our method on the other hand, is designed especially for such cases, where no assumption is made on the mixture (other than being invertible).\\

\noindent To demonstrate this example we introduce two independent sources $X_1,X_2$, over an alphabet size $q$. We apply the linear mixture $Y_1=X_1, Y_2=X_1+X_2$ and shuffle the symbols of $Y_2$. We are then ready to apply (and compare) our suggested methods for finite alphabet sizes, which are the exhaustive search method (Section \ref{Exhaustive Search}) and the objective descent method (Section \ref{objective descent search}).  For the purpose of this experiment we assume both $X_1$ and $X_2$ are distributed according to a Zipf's law distribution, 
\begin{equation}\nonumber
P(k;s,q)=\frac{k^{-s}}{\sum_{i=1}^q i^{-s}}
\end{equation}
with a parameter $s=1.6$. The Zipf's law distribution is a commonly used heavy-tailed distribution. This choice of distribution is further motivated in Chapter \ref{Application_to_Data_Compression}. We apply our suggested algorithms for different alphabet sizes, with a fixed $k=8$, and with only $100$ random initializations for the objective descent method. Figure \ref{fig:BSS} presents the results we achieve.

\begin{figure}[h]
\centering
\includegraphics[width = 0.9\textwidth,bb= 65 125 740 480,clip]{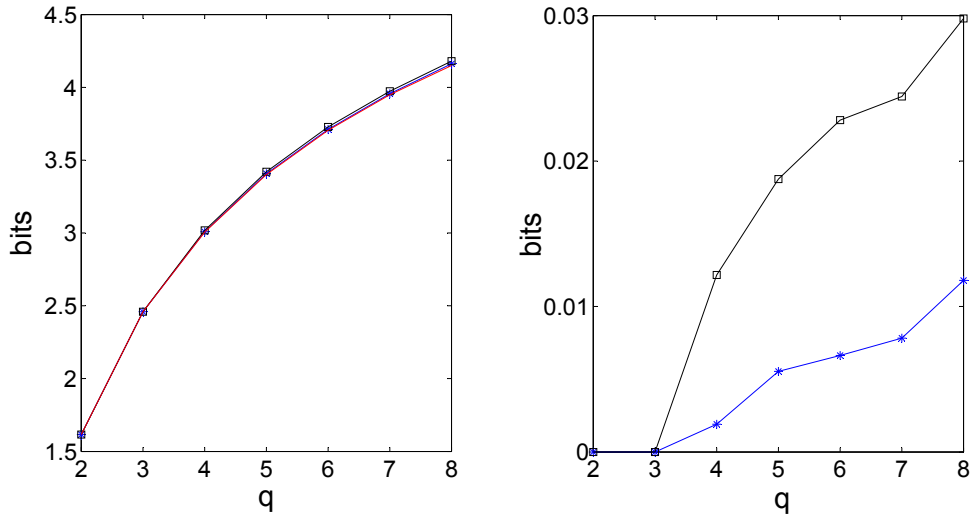}
\caption{BSS simulation results. Left: the lower curve is the joint entropy of $Y1,Y2$, the asterisks curve is the sum of marginal entropies using the exhaustive search method (Section \ref{Exhaustive Search}) while the curve with the squares corresponds the objective descent method (Section \ref{objective descent search}). Right: the curve with the asterisks corresponds to the difference between the exhaustive search method and the joint entropy while the curve with the squares is the difference between the objective descent search method and the joint entropy.}
\label{fig:BSS}
\end{figure}

\noindent We first notice that both methods are capable of finding a transformation for which the sum of marginal entropies is very close to the joint entropy. This means our suggested methods succeed in separating the non-linear mixture $Y_1,Y_2$ back to the statistically independent sources $X_1,X_2$, as we expected.  Looking at the chart on the right hand side of Figure \ref{fig:BSS}, we notice that the difference between the two methods tends to increase as the alphabet size $q$ grows. This is not surprising since the search space grows while the number of random initializations remains fixed.  However, the difference between the two methods is still practically negligible, as we can see from the chart on the left. This is especially important since the objective descent method takes significantly less time to apply as the alphabet size $q$ grows.

\section{discussion}

In this chapter we considered a generalized ICA over finite alphabets framework where we dropped the common assumptions on the underlying model. Specifically, we attempted to decompose a given multi-dimensional vector to its ``as statistically independent as possible" components with no further assumptions, as introduced by \cite{barlow1989finding}. \\ 

\noindent We first focused on the binary case and proposed three algorithms to address this class of problems. In the first algorithm we assumed that there exists a set of independent components that were mixed to generate the observed vector. We showed that these independent components are recovered in a combinatorial manner in $O(n \cdot 2^n)$ operations. The second algorithm drops this assumption and accurately solves the generalized BICA problem through a branch and bound search tree structure. Then, we proposed a third algorithm which bounds our objective from above as tightly as we want to find an approximated solution in $O(n^k \cdot 2^n)$ with $k$ being the approximation accuracy parameter. We further showed that this algorithm can be formulated as a single matrix-vector multiplications and under some generative model assumption the complexity is dropped to be polynomial in $n$. Following that we extended our methods to deal with a larger alphabet size. This case necessitates a heuristic approach to deal with the super exponentially increasing complexity. An objective descent search method is presented for that purpose. We concluded the chapter  by presenting a simple Blind Source Separation application. \\

\chapter[Generalized ICA - The Order Permutation]{Generalized Independent Component Analysis - The Order Permutation}\label{Order_Permutation}
\graphicspath{{Order_Permutation_Figures//}}
\noindent The material in this Chapter is partly covered in \citep{painsky2016largetrans}.
\\

\noindent In the previous chapter we presented the generalized BICA problem. This minimization problem (\ref{eq:sum_ent_min_binary}) is combinatorial in its essence and is consequently considered hard. Our suggested algorithms (described in detail in Sections \ref{BICA with Underlying Independent Components}, \ref{Generalized BICA via Search Tree Based Algorithm} and \ref{The Relaxed Generalized BICA}) strive to find its global minimum, but due to the nature of the problem, they result in quite involved methodologies. This demonstrates a major challenge in providing theoretical guarantees to the solutions they achieve. We therefore suggest a simplified greedy algorithm which is much easier to analyze, as it sequentially minimizes each term of the summation (\ref{eq:sum_ent_min_binary}), $h_b(P(Y_j=0))$, for $j=1,\dots,d$. For the simplicity of presentation, we denote the alphabet size of a $d$ dimensional vector as $m=2^d$. 
\\

\noindent With no loss of generality, let us start by minimizing $h_b(P(Y_1=0))$, which corresponds to the marginal entropy of the most significant bit (msb). Since the binary entropy is monotonically increasing in the range $\left[0,\frac{1}{2}\right]$, we would like to find a permutation of $\underline{p}$ that minimizes a sum of half of its values. This means we should order the $p_i$'s so that half of the $p_i$'s with the smallest values are assigned to $P(Y_1=0)$ while the other half of $p_i$'s (with the largest values) are assigned to $P(Y_1=1)$. For example, assuming $m=8$ and $p_1 \leq p_2 \leq \dots \leq p_8$, a permutation which minimizes $H_b(Y_1)$ is 
\begin{table}[H]
\centering
{
\begin{tabular}{|c|c|c|c|c|c|c|c|c|}
\hline
codeword &000 &001 &010 &011&100 &101 &110 &111 \\
\hline
probability &$p_2$ &$p_3$&$p_1$&$p_4$&$p_8$ &$p_5$&$p_6$&$p_7$ \\
\hline
\end{tabular}}
\end{table}

\noindent We now proceed to minimize the marginal entropy of the second most significant bit, $h_b(P(Y_2=0))$. Again, we would like to assign $P(Y_2=0)$ the smallest possible values of $p_i$'s. However, since we already determined which $p_i$'s are assigned to the msb, all we can do is reorder the $p_i$'s without changing the msb. This means we again sort the $p_i$'s so that the smallest possible values are assigned to  $P(Y_2=0)$, without changing the msb. In our example, this leads to,
\begin{table}[H]
\centering
{
\begin{tabular}{|c|c|c|c|c|c|c|c|c|}
\hline
codeword &000 &001 &010 &011&100 &101 &110 &111 \\
\hline
probability &$p_2$ &$p_1$&$p_3$&$p_4$&$p_6$ &$p_5$&$p_8$&$p_7$ \\
\hline
\end{tabular}}
\end{table} 

\noindent Continuing in the same manner, we would now like to reorder the $p_i$'s to minimize $h_b(P(Y_3=0))$ without changing the previous bits. This results in

\begin{table}[H]
\centering
{
\begin{tabular}{|c|c|c|c|c|c|c|c|c|}
\hline
codeword &000 &001 &010 &011&100 &101 &110 &111 \\
\hline
probability &$p_1$ &$p_2$&$p_3$&$p_4$&$p_5$ &$p_6$&$p_7$&$p_8$ \\
\hline
\end{tabular}}
\end{table} 

\noindent Therefore, we show that a greedy solution to (\ref{eq:sum_ent_min_binary}) which sequentially minimizes $H(Y_j)$ is attained by simply ordering the joint distribution $\underline{p}$ in an ascending (or equivalently descending) order. In other words, the  \textit{order permutation} suggests to  simply order the probability distribution $p_1, \dots, p_m$ in an ascending order, followed by a mapping of the $i^{th}$ symbol (in its binary representation) the $i^{th}$ smallest probability. \\

\noindent At this point it seems quite unclear how well the order permutation performs, compared both with the relaxed BICA we previously discussed, and the optimal permutation which minimizes  (\ref{eq:sum_ent_min_binary}). In the following sections we introduce some theoretical properties which demonstrate this method's effectiveness.

\section{Worst-case Independent Components Representation}
\label{worst case}
We now introduce the theoretical properties of our suggested algorithms. Naturally, we would like to quantify how much we ``lose" by representing a given random vector $\underline{X}$ as if its components are statistically independent. We notice that our objective (\ref{eq:sum_ent_min_binary}) depends on the distribution of a given random vector $\underline{X}  \sim \underline{p} $, and the applied invertible  transformation $\underline{Y}=g(\underline{X})$. Therefore, we slightly change the notation of (\ref{eq:sum_ent_min_binary}) and denote the cost function as $C(\underline{p},g)=\sum_{j=1}^d H(Y_j) - H(\underline{X})$.\\

\noindent Since our methods  strongly depend on the given probability distribution $\underline{p}$, we focus on the worst-case and the average case of $C(\underline{p},g)$, with respect to  $\underline{p}$. 
Let us denote the order permutation as $g_{ord}$ and the permutation which is found by the piece-wise linear relaxation as $g_{lin}$. We further define $g_{bst}$ as the permutation that results with a lower value of   $C(\underline{p},g)$, between $g_{lin}$ and $g_{ord}$. This means that $$g_{bst}=\underset{\{g_{lin},g_{ord}\}} {\arg\min}C(\underline{p},g).$$ 
In addition, we define $g_{opt}$ as the optimal permutation that minimizes (\ref{eq:sum_ent_min_binary}) over all possible permutations. Therefore, for any given $\underline{\tilde{p}}$, we have that $C(\underline{\tilde{p}},g_{opt}) \leq C(\underline{\tilde{p}},g_{bst}) \leq C(\underline{\tilde{p}},g_{ord})$. 
In this Section we examine the worst-case performance of both of our suggested algorithms. Specifically, we would like to quantify the maximum of $C(\underline{p},g)$ over all joint probability distributions $\underline{p}$, of a given alphabet size $m$. 

\begin{theorem}
\label{worst-case}
For any  random vector $\underline{X} \sim \underline{p}$, over an alphabet size $m$ we have that
$$\maxmax_{\underline{p}}C(\underline{p},g_{opt})=\Theta(\log(m))$$
\end{theorem}

\begin{proof}
\noindent We first notice that $\sum_{j=1}^d H(Y_j)=\sum_{j=1}^d h_b(P(Y_j=0)) \leq d=\log(m)$. In addition, $H(\underline{X}) \geq 0$. Therefore, we have that $C(\underline{p},g_{opt})$ is bounded from above by $\log(m)$. Let us also show that this bound is tight, in the sense that there  exists a joint probability distribution $\underline{\tilde{p}}$ such that  $C(\underline{\tilde{p}},g_{opt})$ is linear in $\log(m)$. 
Let $\tilde{p}_1=\tilde{p}_2=\dots=\tilde{p}_{m-1}=\frac{1}{3(m-1)}$ and $\tilde{p}_m=\frac{2}{3}$.
Then, $\underline{\tilde{p}}$ is ordered and satisfies $P(Y_i=0)=\frac{m}{6(m-1)}$.\\

\noindent In addition, we notice that assigning symbols in a decreasing order to $\underline{\tilde{p}}$ (as mentioned in above) results with an optimal permutation. This is simply since  $P(Y_j=0)=\frac{m}{6(m-1)}$ is the minimal possible value of any $P(Y_j=0)$ that can be achieved when summing any $\frac{m}{2}$ elements of $\tilde{p}_i$.
Further we have that, 
\begin{align}
C(\underline{\tilde{p}},g_{opt})=&\sum_{j=1}^d H(Y_j) - H(\underline{X})=\sum_{j=1}^d h_b(P(Y_j=0)) - H(\underline{X})=\\\nonumber
&\log(m)\cdot h_b\left( \frac{m}{6(m-1)}\right) +\left( (m-1) \frac{1}{3(m-1)}\log \frac{1}{3(m-1)} +\frac{2}{3}\log \frac{2}{3} \right)=\\\nonumber
&\log(m)\cdot h_b\left( \frac{m}{6(m-1)}\right)-\frac{1}{3}\log(m-1)+\frac{1}{3}\log\frac{1}{3}+\frac{2}{3}\log\frac{2}{3} \underset{m\rightarrow \infty}{\longrightarrow}\\\nonumber
& \log(m) \cdot \left(h_b\left( \frac{1}{6}\right)-\frac{1}{3}\right)-h_b\left( \frac{1}{3}\right).
\end{align}
Therefore, $\maxmax_{\underline{p}}C(\underline{p},g_{opt})=\Theta(\log(m))$. 
\end{proof}
\noindent Theorem \ref{worst-case} shows that even the optimal permutation achieves a sum of marginal entropies which is $\Theta(\log(m))$ bits greater than the joint entropy, in the worst case. This means that there exists at least one source $\underline{X}$ with a joint probability distribution which is impossible to encode as if its components are independent without losing at least $\Theta(\log(m))$ bits. Note that this extra number of bits is high, and corresponds to a trivial encoding of the components without any factorization. However, we now show that such sources are very ``rare".  

\section{Average-case Independent Components Representation}
\label{average case}
%As we show in the previous section, there exist at least one probability distribution which results with $C(\underline{p},g_{opt})=\Theta(\log(m))$. However, we argue that there are only very few such distributions. 
In this section we show that the expected value of  $C(\underline{p},g_{opt})$ is bounded by a small constant, when averaging uniformly over all possible $\underline{p}$ over an alphabet size $m$. \\

\noindent To prove this, we recall that $C(\underline{p},g_{opt}) \leq C(\underline{p},g_{ord})$ for any given probability distribution $\underline{p}$. Therefore, we would like to find the expectation of $C(\underline{p},g_{ord})$ where the random variables are  $p_1, \dots, p_m$, taking values over a uniform simplex. 

\begin{proposition}
\label{H(X)}
Let $\underline{X} \sim \underline{p} $ be a random vector of an alphabet size $m$ and a joint probability distribution $\underline{p}$. The expected joint entropy of $\underline{X}$, where the expectation is  over a uniform simplex of joint probability distributions $\underline{p}$ is
\begin{equation}\nonumber
\mathbb{E}_{\underline{\smash{p}}}\left\{H(\underline{X}) \right\}=\frac{1}{\log_e{2}}\left(\psi(m+1)-\psi(2)\right) 
\end{equation}
where $\psi$ is the \textit{digamma function}.
\end{proposition}
\noindent The proof of this proposition is left for the Appendix \ref{proof_for_H(X)}. \\

\noindent We now turn to examine the expected sum of the marginal entropies, $\sum_{j=1}^d H(Y_j)$ under the order permutation. As described above, the order permutation suggests sorting the probability distribution $p_1, \dots, p_m$ in an ascending order, followed by mapping of the $i^{th}$ symbol (in a binary representation) the $i^{th}$ smallest probability. Let us denote $p_{(1)}\leq \dots\leq p_{(m)}$ the ascending ordered probabilities $p_1, \dots, p_m$. \cite{bairamov2010limit} show that the expected value of $p_{(i)}$ is
\begin{equation}
\label{expectation}
\mathbb{E}\left\{p_{(i)}\right\}=\frac{1}{m}\sum_{k=m+1-i}^{m}\frac{1}{k}=\frac{1}{m}\left(K_m-K_{m-i}\right)
\end{equation}
where $K_m=\sum_{k=1}^{m}\frac{1}{k}$ is the Harmonic number. Denote the ascending ordered binary representation of all possible symbols in a matrix form $A \in \{0,1\}^{(m\times d)}$. This means that entry $A_{ij}$ corresponds to the $j^{th}$ bit in the $i^{th}$ symbol, when the symbols are given in an ascending order. Therefore, the expected sum of the marginal entropies of $\underline{Y}$, when the expectation is over a uniform simplex of joint probability distributions $p$, follows
\begin{align}
\label{sum_of_marginals}
\mathbb{E}_{\underline{\smash{p}}}\left\{\sum_{j=1}^d H(Y_j) \right\} \underset{(a)}{\leq}& \sum_{j=1}^d h_b(\mathbb{E}_{\underline{\smash{p}}}\{Y_j\}) \underset{(b)}= \sum_{j=1}^d  h_b \left(\frac{1}{m}\sum_{i=1}^{m}A_{ij}  \left(K_m-K_{m-i}\right)  \right) \underset{(c)}=\\\nonumber
&\sum_{j=1}^d   h_b \left( \frac{1}{2}K_m- \frac{1}{m}\sum_{i=1}^{m}A_{ij}K_{m-i} \right)
\end{align}
where $(a)$ follows from Jensen's inequality, $(b)$ follows from (\ref{expectation}) and $(c)$ follows since $\sum_{i=1}^{m}A_{ij}=\frac{1}{2}$ for all $j=1, \dots,d$.  \\

\noindent We now turn to derive asymptotic bounds of the expected difference between the sum of $\underline{Y}$'s marginal entropies and the joint entropy of $\underline{X}$, as appears in (\ref{eq:sum_ent_min_binary}).

\begin{theorem}
\label{average case - asymp}
Let $\underline{X} \sim \underline{p}$ be a random vector of an alphabet size $m$ and joint probability distribution $\underline{p}$. Let $\underline{Y}=g_{ord}(\underline{X})$ be the order permutation. 
For $d \geq 10$, the expected value of $C(\underline{p},g_{ord})$, over a uniform simplex of joint probability distributions $\underline{p}$, satisfies
\begin{equation}
\nonumber
\mathbb{E}_{\underline{\smash{p}}}C(\underline{p},g_{ord})= \mathbb{E}_{\underline{\smash{p}}}\left\{\sum_{j=1}^d H(Y_j)- H(\underline{X}) \right\} < 0.0162 +O\left(\frac{1}{m}\right)
\end{equation}
\end{theorem}

\begin{proof}
Let us first derive the expected marginal entropy of the least significant bit, $j=1$, according to (\ref{sum_of_marginals}).
\begin{align}
\label{LSB}
\mathbb{E}_{\underline{\smash{p}}}\left\{ H(Y_1) \right\} \leq &  h_b \left( \frac{1}{2}K_m- \frac{1}{m}\sum_{i=1}^{m/2}K_{m-i} \right)=\\\nonumber 
&h_b \left( \frac{1}{2}K_m- \frac{1}{m}\left(\sum_{i=1}^{m-1}K_{i}-\sum_{i=1}^{\frac{m}{2}-1}K_{i} \right)\right)\underset{(a)}{=}\\\nonumber
&h_b \left( \frac{1}{2}K_m- \frac{1}{m}\left(mK_{m}-m-\frac{m}{2}K_{\frac{m}{2}}+\frac{m}{2} \right)\right)=\\\nonumber
&h_b \left( \frac{1}{2}\left(K_{\frac{m}{2}}-K_{m}+1 \right)\right)\underset{(b)}{<}\\\nonumber
&h_b \left( \frac{1}{2}\log_{e}\left(\frac{1}{2}\right)+\frac{1}{2}+O\left(\frac{1}{m}\right)\right)\underset{(c)}{\leq} \\\nonumber
&h_b \left(\frac{1}{2}\log_{e}\left(\frac{1}{2}\right)+\frac{1}{2}\right)+O\left(\frac{1}{m}\right) h_b' \left(\frac{1}{2}\log_{e}\left(\frac{1}{2}\right)+\frac{1}{2}\right)=\\\nonumber
&h_b \left(\frac{1}{2}\log_{e}\left(\frac{1}{2}\right)+\frac{1}{2}\right)+O\left(\frac{1}{m}\right)
\end{align}
 where $(a)$ and (b) follow the harmonic number properties:
\begin{enumerate} [(a)]

\item	$\sum_{i=1}^{m}K_{i}=(m+1)K_{m+1}-(m+1)$
\item	$\frac{1}{2(m+1)} < K_m-\log_e(m)-\gamma<\frac{1}{2m}$, where $\gamma$ is the Euler-Mascheroni constant \citep{young199175}
\end{enumerate}
and $(c)$ results from the concavity of the binary entropy.

\noindent Repeating the same derivation for different values of $j$, we attain

\begin{align}
\label{all_bits}
\mathbb{E}_{\underline{\smash{p}}}\left\{ H(Y_j) \right\} \leq &  h_b \left( \frac{1}{2}K_m-\frac{1}{m} \sum_{l=1}^{2^j-1} (-1)^{l+1}\sum_{i=1}^{l \frac{m}{2^j}}K_{m-i} \right)= \\\nonumber
&h_b \left( \frac{1}{2}K_m-\frac{1}{m}\sum_{l=1}^{2^j} (-1)^{l}\sum_{i=1}^{l \frac{m}{2^j}-1}K_{i} \right)=\\\nonumber
& h_b \left( \frac{1}{2}K_m-\frac{1}{m}\sum_{l=1}^{2^j} (-1)^{l}\left(l  \frac{m}{2^j}K_{l  \frac{m}{2^j}}-l  \frac{m}{2^j}  \right) \right)<\\\nonumber
&  h_b \left( \sum_{i=1}^{2^j-1}(-1)^{i+1}\frac{i}{2^j}\log_{e}\left(\frac{i}{2^j}\right)+\frac{1}{2}\right)+O\left(\frac{1}{m} \right)\quad\quad \forall j=1,\dots,d.
\end{align}

\noindent We may now evaluate the sum of expected marginal entropies of $\underline{Y}$. For simplicity of derivation let us obtain $\mathbb{E}_{\underline{\smash{p}}}\left\{ H(Y_j) \right\}$ for $j=1, \dots, 10$ according to (\ref{all_bits}) and upper bound $\mathbb{E}_{\underline{\smash{p}}}\left\{ H(Y_j) \right\}$ for $j>10$ with $h_b\left(\frac{1}{2}\right)=1$.  This means that for $d \geq 10$ we have 
\begin{align}
\label{asymp_H}
 \mathbb{E}_{\underline{\smash{p}}}\left\{\sum_{j=1}^d H(Y_j)\right\} <&
\sum_{j=1}^{10} \mathbb{E}_{\underline{\smash{p}}}\left(H\left\{Y_j\right\}\right) +\sum_{j=11}^d h_b\left(\frac{1}{2}\right)< \\\nonumber
&9.4063 + (d-10)+O\left(\frac{1}{m} \right).
\end{align}
\noindent The expected joint entropy may also be expressed in a more compact manner. In Proposition \ref{H(X)} it is shown than $\mathbb{E}_{\underline{\smash{p}}}\left\{H(\underline{X}) \right\}=\frac{1}{\log_e{2}}\left(\psi(m+1)-\psi(2)\right) $. Following the inequality in \citep{young199175}, the Digamma function, $\psi(m+1)$, is bounded from below by $\psi(m+1)=H_m-\gamma>\log_e(m)+\frac{1}{2(m+1)}$. Therefore, we conclude that for $d\geq10$ we have that  
\begin{align}
 \mathbb{E}_{\underline{\smash{p}}}\left\{\sum_{j=1}^d H(Y_j)-H(\underline{X})\right\} < &9.4063 + (d-10)-\log{(m)} +\\\nonumber
& \frac{\psi(2)}{\log_e{2}} +O\left(\frac{1}{m}\right) = 0.0162 +O\left(\frac{1}{m}\right)
\end{align}
\end{proof}
\noindent In addition, we would like to evaluate the expected difference between the sum of marginal entropies and the joint entropy of $\underline{X}$, that is, without applying any permutation. This shall serve us as a reference  to the upper bound we achieve in Theorem \ref{average case - asymp}. 

\begin{theorem}
\label{theorem4}
Let $\underline{X} \sim \underline{p}$ be a random vector of an alphabet size $m$ and joint probability distribution $\underline{p}$. 
The expected difference between the sum of marginal entropies and the joint entropy of $\underline{X}$, when the expectation is taken over a uniform simplex of joint probability distributions $\underline{p}$, satisfies
\begin{equation}
\nonumber
 \mathbb{E}_{\underline{\smash{p}}}\left\{\sum_{j=1}^d H(X_j)- H(\underline{X}) \right\} < \frac{\psi(2)}{\log_e{2}}=0.6099
\end{equation}
\end{theorem}

\begin{proof}
We first notice that $P\left(X_j=1\right)$ equals the sum of one half of the probabilities $p_i, i=1, \dots, m$ for every $j=1 \dots d$. Assume $p_i$'s are randomly (and uniformly) assigned to each of the $m$ symbols. Then, $\mathbb{E}\{P\left(X_j=1\right)\}=\frac{1}{2}$ for every $j=1 \dots d$. Hence,

\begin{align}
\nonumber
 \mathbb{E}_{\underline{\smash{p}}}\left\{\sum_{j=1}^d H(X_j)- H(\underline{X}) \right\} =& 
\sum_{j=1}^d \mathbb{E}_{\underline{\smash{p}}}\left\{H_b(X_j)\right\}- \mathbb{E}_{\underline{\smash{p}}}\{H(\underline{X})\} <\\\nonumber
&
d-\log{(m)}+\frac{1}{\log_e{2}} \left( \psi(2)-\frac{1}{2(m+1)}\right)<\frac{\psi(2)}{\log_e{2}}
\end{align}
\end{proof}

\noindent To conclude, we show that for a random vector $\underline{X}$ over an alphabet size $m$, we have 
$$ \mathbb{E}_{\underline{\smash{p}}}C(\underline{p},g_{opt}) \leq \mathbb{E}_{\underline{\smash{p}}}C(\underline{p},g_{bst}) \leq \mathbb{E}_{\underline{\smash{p}}}C(\underline{p},g_{ord}) < 0.0162+O\left(\frac{1}{m}\right)$$
for $d\geq10$, where the expectation is over a uniform simplex of joint probability distributions $\underline{p}$. 
This means that when the alphabet size is large enough, even the simple order permutation achieves, on the average, a sum of marginal entropies which is only $0.0162$ bits greater than the joint entropy, when all possible probability distributions $\underline{p}$ are equally likely to appear. Moreover, we show that the simple order permutation reduced the expected difference between the sum of the marginal entropies and the joint entropy of $\underline{X}$ by more than half a bit, for sufficiently large $m$.  

\section{Block-wise Order Permutation}
\label{Block-wise Order Permutation}

The computational complexity of the order permutation is $O(d2^d)$, according to a simple quick-sort algorithm \citep{Hoare:1961:AQ:366622.366644}. We would now like to introduce a structured transformation which achieves a lower complexity while (almost) maintaining the same favorable asymptotic properties as the order permutation.\\

\noindent We define a \textit{block transformation} (with a parameter $b$) as splitting the $m$ values of $\{p_i\}_{i=1}^m$ into non-overlapping blocks of size $m_b=2^b$, and applying an invertible transformation on each of the blocks independently. For example, assume a given $3$-dimensional binary vector with the following probability distribution:

\begin{table}[!htbp]
\centering
{
\begin{tabular}{|c|c|c|c|c|c|c|c|c|}
\hline
codeword &000 &001 &010 &011&100 &101 &110 &111 \\
\hline
probability &$p_6$ &$p_3$&$p_1$&$p_8$&$p_2$ &$p_5$&$p_4$&$p_7$ \\
\hline
\end{tabular}}
\end{table} 
\noindent then, a block transformation with $b=2$ is any mutually exclusive permutation of the probabilities in the sets $\{p_6,p_3,p_1,p_8\}$ and $\{p_2,p_5,p_4,p_7\}$.\\   

\noindent We define the \textit{block order permutation}, with a parameter $b$, as an order permutation, applied on each of the $\frac{m}{m_b}$ blocks, independently. For example, assume $p_1 \leq p_2 \leq \dots \leq p_8$, then the block order permutation with  $b=2$ is simply

\begin{table}[!htbp]
\centering
{
\begin{tabular}{|c|c|c|c|c|c|c|c|c|}
\hline
codeword &000 &001 &010 &011&100 &101 &110 &111 \\
\hline
probability &$p_1$ &$p_3$&$p_6$&$p_8$&$p_2$ &$p_4$&$p_5$&$p_7$ \\
\hline
\end{tabular}}
\end{table} 

\noindent Denote the block order permutation as $\underline{Y}=g_{ord,b}(\underline{X})$. Notice that for $b=log(m)=d$, the block order permutation is simply the order permutation.\\

\noindent We would like to quantify the expected difference between the sum of marginal entropies after the block order permutation is applied, and the joint entropy of $\underline{X}$,  $C(\underline{p},g_{ord,b})$
 \begin{equation}
\label{objective}
\mathbb{E}_{\underline{\smash{p}}}C(\underline{p},g_{ord,b})= \mathbb{E}_{\underline{\smash{p}}}\left\{\sum_{j=1}^d H(Y_j)- H(\underline{X}) \right\}
\end{equation}
where the expectation is with respect to a uniform prior on the probability distribution, as before. In other words, we would like to find the expectation (\ref{objective}) where the random variables are  $p_1, \dots, p_m$, taking values over a uniform simplex.

\begin{theorem}
\label{theorem1}
Let $\underline{X} \sim \underline{p}$ be a random vector of an alphabet size $m$ and joint probability distribution $\underline{p}$. Let $\underline{Y}=g_{ord,b}(\underline{X})$ be the block order permutation. 
The expected value of  $C(\underline{p},g_{ord,b})$, where the expectation is over a uniform simplex of joint probability distributions $\underline{p}$, satisfies
\begin{align}
\nonumber
\mathbb{E}_{\underline{\smash{p}}}C(\underline{p},g_{ord,b})= \mathbb{E}_{\underline{\smash{p}}}\left\{\sum_{j=1}^d H(Y_j)- H(\underline{X}) \right\}\leq H_{ord,b}+d-b-\frac{1}{\log_e{2}}\left(\psi(m+1)-\psi(2)\right).
\end{align}
where $H_{ord,b}$ is the upper bound on the expected sum of marginal entropies of a random vector with an alphabet size $m_b=2^b$, after an order permutation is applied (\ref{all_bits}) and $\psi$ is the digamma function.
\end{theorem}

\begin{proof}
Let $Z_1,Z_2, \dots, Z_{m+1}$ be independent exponential random variables with the same parameter value $\Lambda$. Set $$S=Z_1+Z_2+\dots+Z_{m+1}$$ and $$D_i=\frac{Z_i}{S} \quad (1\leq i \leq m+1).$$
 
\noindent Then, $\{D_i\}_{i=1}^{m+1}$ is distributed as a set of $m+1$ spacings determined by  $m$ independent uniform random variables \citep{pyke1965spacings}. In other words, 
$$f_{(D_1, D_2, \dots, D_m)}(d_1, d_2, \dots, d_m)=m! \quad d_i\geq0, \quad 0\leq \sum_{i=1}^m d_i \leq 1.$$

\noindent Moreover, it can be shown \citep{pyke1965spacings} that $(D_1, D_2, \dots, D_m)$ are distributed independently of $S$ ,
\begin{equation}
\label{independence}
f_{(D_1, D_2, \dots, D_m|S)}(d_1, d_2, \dots, d_m|s)=f_{(D_1, D_2, \dots, D_m)}(d_1, d_2, \dots, d_m).
\end{equation}
 
\noindent We now apply the block order permutation on the first block, $(D_1, D_2, \dots, D_{m_b})$.
We have that
\begin{align}\nonumber 
(D_1, D_2, \dots, D_{m_b+1})=&\left(\frac{Z_1}{S},\frac{Z_2}{S}, \dots, \frac{Z_{m_b+1}}{S}\right)=\\\nonumber
&\left(\frac{Z_1}{S_{m_b}}\frac{S_{m_b}}{S},\frac{Z_2}{S_{m_b}}\frac{S_{m_b}}{S}, \dots, \frac{Z_{m_b+1}}{S_{m_b}}\frac{S_{m_b}}{S}\right)=\frac{S_{m_b}}{S}\left(\tilde{D}_1,\tilde{D}_2,\dots,\tilde{D}_{m_b+1} \right)
\end{align}
where $S_{m_b}=\sum_{i=1}^{m_b+1} Z_i$ and $\tilde{D_i}=\frac{Z_i}{S_{m_b}}$. 

\noindent We define $\{D_{(i)}\}_{i=1}^{m_b}$ as the ordering of  $\{D_{i}\}_{i=1}^{m_b}$. 
Therefore, $$E\left(D_{(i)}\right)=E\left(\frac{S_{m_b}}{S}{\tilde{D}_{(i)}}\right)=E_{\frac{S_{m_b}}{S}} E\left(\frac{S_{m_b}}{S}{\tilde{D}_{(i)}}\bigg|\frac{S_{m_b}}{S}\right)=E\left(\frac{S_{m_b}}{S}\right) E\left(\tilde{D}_{(i)}\right)$$
where the last equation follows from the $\tilde{D}_i$'s being uniformaly distributed over a unit simplex and are therefore distributed independently of $S_{m_b}$, as indicated in (\ref{independence}).
Since $\frac{S_{m_b}}{S}$ is Beta distributed with parameters $(\alpha=m_b,\beta=m-{m_b})$, we have that  $E\left(\frac{S_{m_b}}{S}\right)=\frac{m_b}{m}$ and $$E\left(D_{(i)}\right)=\frac{m_b}{m}E\left(\tilde{D}_{(i)}\right)$$
where $E\left(\tilde{D}_{(i)}\right)$ is the expected value of $i^{th}$ smallest value drawn from a uniform simplex of size $m_b$. We denote the vector of  $\left\{E\left(\tilde{D}_{(i)}\right)\right\}_{i=1}^{m_b}$ as $E\left(\underline{\tilde{D}}_{ord}\right)$.
Notice this derivation applies for any of the $\frac{m}{m_b}$ blocks. \\

\noindent Let us now derive the expected marginal probabilities of the vector $\underline{Y}=g_{ord,b}(\underline{X})$. We begin our derivation with the expected marginal probabilities of a $b$-dimensional vector $\underline{X}^{(b)}$, after the order permutation is applied, $\underline{Y}^{(b)}=g_{ord}(\underline{X}^{(b)})$. Then, the marginal probabilities of the $b$ components of $Y^{(b)}$ satisfy:
$$E\left(\underline{P}(Y^{(b)})\right)=E\left(\underline{\tilde{D}}_{ord}\right)^T \cdot A_{m_b}$$
where $E\left(\underline{P}(Y^{(b)})\right)=\left\{E(P(Y^{(b)}_i=1))\right\}_{i=1}^{m_b}$ and $A_{m_b}$ is the fixed list of binary symbols. For example, for $b=2$ we have that

$$ E\left(\underline{P}\left(Y^{(b)}\right)\right)=\left[E(P(Y^{(b)}_1=1)) \  E(P(Y^{(b)}_2=1)) \ \dots \  E(P(Y^{(b)}_4=0))\right]$$

$$E\left(\underline{\tilde{D}}_{ord}\right)^T=\left[E\left(\tilde{D}_{(1)}\right) \ E\left(\tilde{D}_{(2)}\right)\  \dots\  E\left(\tilde{D}_{(4)}\right)\right]$$ 
and
\begin{equation}\nonumber
A_{m_2}=\left[
\begin{array}{cc} 
0&0\\
0& 1\\
1& 0\\
1& 1\\
\end{array}\right].
\end{equation}
 We now go back to $\underline{Y}=g_{ord,b}(\underline{X})$.
In the same manner, we have that
$$E\left(\underline{P}(Y)\right)=\frac{m_b}{m} \left[E\left(\underline{\tilde{D}}_{ord}\right) \ E\left(\underline{\tilde{D}}_{ord}\right) \ \dots \ E\left(\underline{\tilde{D}}_{ord}\right)\right]^T \cdot A_{m}.$$

\noindent For example, assume $d=3$ and $b=2$:
\begin{align}\nonumber
E(\underline{P}(Y))=\frac{1}{2}\left[E\left(\tilde{D}_{(1)}\right) \ \dots \ E\left(\tilde{D}_{(4)}\right) \, E\left(\tilde{D}_{(1)}\right) \ \dots \ E\left(\tilde{D}_{(4)}\right)\right] \cdot \left[
\begin{array}{ccc} 
0&0& 0\\
0& 0& 1\\
0& 1& 0\\
0& 1& 1\\
1& 0& 0\\
1& 0& 1\\
1& 1& 0\\
1& 1& 1\\
\end{array}\right].
\end{align}

\noindent Looking at the last $b$ components (LSB's), we notice that by construction, 
\begin{equation}
\label{block_bits_1}
\left\{E(P(Y_j=1))\right\}_{j=b}^{m}=E\left(\underline{P}(Y^{(b)})\right).
\end{equation}

\noindent In addition, we have that for the first $d-b$ bits (MSB's)
\begin{equation}
\label{block_bits_2}
\left\{E(P(Y_j=1))\right\}_{j=1}^{d-b}=\frac{1}{2}.
\end{equation}
Therefore, 
\begin{align}
 \mathbb{E}_{\underline{\smash{p}}}\left\{\sum_{j=1}^d H(Y_j)\right\} \underset{(a)}{\leq}&
 \sum_{j=1}^d h_b\left(\mathbb{E}_{\underline{\smash{p}}}\left(P(Y_j=1)\right)\right) \underset{(b)}{\leq}\\\nonumber
&\sum_{j=1}^{d-b}h_b\left(\frac{1}{2}\right)+\sum_{j=1}^b h_b\left(\mathbb{E}_{\underline{\smash{p}}}\left(P\left(Y^{(b)}_j=1\right)\right)\right)\underset{(c)}{\leq} 
H_{ord,b}+d-b.
\end{align}
where $(a)$ follows Jensen's inequality, $(b)$ follows from (\ref{block_bits_1}, \ref{block_bits_2}) and $(c)$ introduces a notation of $H_{ord,b}$ as the upper bound on the expected sum of marginal entropies of a $b$-dimensional vector, after an order permutation is applied.\\

\noindent Further, we showed in Proposition \ref{H(X)} that the expected joint entropy of $\underline{X}$ satisfies
\begin{equation}\nonumber
\mathbb{E}_{\underline{\smash{p}}}\left\{H(\underline{X}) \right\}=\frac{1}{\log_e{2}}\left(\psi(m+1)-\psi(2)\right). 
\end{equation}
Therefore, we have that 
\begin{equation}
\nonumber
\mathbb{E}_{\underline{\smash{p}}}C(\underline{p},g_{ord,b})= \mathbb{E}_{\underline{\smash{p}}}\left\{\sum_{j=1}^d H(Y_j)- H(\underline{X}) \right\}\leq H_{ord,b}+d-b-\frac{1}{\log_e{2}}\left(\psi(m+1)-\psi(2)\right).
\end{equation}
\end{proof}

\noindent In the same manner as with the order permutation, we may further derive an asympthotic bound for $\mathbb{E}_{\underline{\smash{p}}}C(\underline{p},g_{ord,b})$:

\begin{theorem}
Let $\underline{X} \sim \underline{p}$ be a random vector of an alphabet size $m$ and joint probability distribution $\underline{p}$. Let $\underline{Y}=g_{ord,b}(\underline{X})$ be the block order permutation. 
For $d \geq b \geq 10$, the expected value of $C(\underline{p},g_{ord,b})$, where the expectation is over a uniform simplex of joint probability distributions $\underline{p}$, satisfies
\begin{equation}
\label{block_perm_bound}
\mathbb{E}_{\underline{\smash{p}}}C(\underline{p},g_{ord,b})=\mathbb{E}_{\underline{\smash{p}}}\left\{\sum_{j=1}^d H(Y_j) - H(\underline{X})\right\} \leq  0.0162 +O\left(\frac{1}{2^b}\right).
\end{equation}

\end{theorem}

\begin{proof}

We showed in (\ref{asymp_H}) that for $d\geq10$,  
\begin{align}
\label{order_permutation}
 H_{ord,d}=\mathbb{E}_{\underline{\smash{p}}}\left\{\sum_{j=1}^d H(Y_j)\right\} <&
\sum_{j=1}^{10} \mathbb{E}_{\underline{\smash{p}}}\left(H\left\{Y_j\right\}\right) +\sum_{j=11}^d h_b\left(\frac{1}{2}\right)<\\\nonumber
& 9.4063 + (d-10)+O\left(\frac{1}{2^d} \right).
\end{align}
\noindent Therefore, 
\begin{align}
\nonumber
\mathbb{E}_{\underline{\smash{p}}}C(\underline{p},g_{ord,b})=&\mathbb{E}_{\underline{\smash{p}}}\left\{\sum_{j=1}^d H_b(Y_j)- H(\underline{X}) \right\}\underset{(a)}{\leq} \\\nonumber
&H_{ord,b}+d-b-\frac{1}{\log_e{2}}\left(\psi(m+1)-\psi(2)\right)\underset{(b)}{\leq} \\\nonumber 
&9.4063 + (b-10)+O\left(\frac{1}{2^b}\right)+d-b-\frac{1}{\log_e{2}}\left(\psi(m+1)-\psi(2)\right)\underset{(c)}{\leq} \\\nonumber
&9.4063-10+d-\log(m)+\frac{\psi(2)}{\log_e{2}}+O\left(\frac{1}{2^b}\right)\leq 0.0162 +O\left(\frac{1}{2^b}\right)
\end{align}
where $(a)$ follows from Theorem \ref{theorem1}, $(b)$ follows form (\ref{order_permutation}) and $(c)$ follows from $\psi(m+1)>\log_e(m)+\frac{1}{2(m+1)}$, as derived in \citep{young199175}.
\end{proof}

\noindent This means that the upper bound of $\mathbb{E}_{\underline{\smash{p}}}C(\underline{p},g_{ord,b})$ depends on the size of the block $2^b$ and not on the alphabet size $m=2^d$. In other words, assuming there exists a value $b=b^*$ for which the bound (\ref{block_perm_bound}) is practically sufficient, then there is no need to apply the costly $O(d2^d)$ order permutation to achieve (almost) the same results (on the average). Moreover, the cost of applying the block order permutation in this case is sorting each of the $\frac{m}{m_{b^*}}$. This leads to an overall complexity of $2^{d- b^*}O\left(b^* 2^{b^*}\right)=O\left( b^* 2^d\right)$. Notice this complexity is linear in the size of $\underline{p}$ and therefore asymptomatically achieves the computational lower bound, as indicated in Section \ref{problem_formulation}.   

\section{Discussion}

Barlow's minimal redundancy representation problem \citep{barlow1989finding} is a hard long standing open problem. The main difficulty results from this problem's combinatorial nature, which makes it very challenging, both in terms of providing bounds or designing efficient algorithms. In this chapter we tackle Barlow's problem from a new angle, by providing a sub-optimal solution which is much easier to analyze. This provides us not only with a simple, non-combinatorial, algorithm that is easy to implement, but also with theoretical bounds and guarantees on the results we achieve. Moreover, it gives us some insight on the optimal solution. Specifically, it shows us how well arbitrary random vectors over finite alphabets decompose into independent components. This property is of high value, as for the first time, it answers the question ``how well we can do?", when dealing with Barlow's problem.\\

\noindent In addition, we introduce a computationally simplified version of the order permutation, namely, the block-wise order permutation. This method separates the alphabet of the random vector into disjoint blocks and orders each block separately. We show that asymptotically, the block order permutation achieves the same accuracy as the order permutation (on the average), while benefiting from a computational complexity that is practically linear in the alphabet size. This computational complexity is the best we can achieve, without further assumptions on the structure of the vector we decompose (see Section \ref{problem_formulation}).

\chapter[Generalized Versus Linear BICA]{Generalized Versus Linear Independent Component Analysis}\label{BICA_Vs_Linear}
\graphicspath{{BICA_Vs_Linear_Figures//}}
\noindent The material in this Chapter is partly covered in \citep{painsky2016Binary}.

\section{Introduction}
As described in Chapter \ref{overview}, the linear ICA problem over finite fields has been given a considerable amount of attention during the past years. This is mainly manifested in a line of work initiated by \cite{yeredor2007ica}. In his setup, Yeredor considers a linear mixture of statistically independent sources and proposes a method for source separation based on entropy minimization. Yeredor assumes that the number of independent sources $d$ is known  and the mixing matrix is a $d$-by-$d$ invertible matrix. Specifically,
\begin{equation}
\label{linear_model}
\underline{X}=A\underline{S}
\end{equation}
where $S$ is a vector of $d$ indepdendet sources, $A$ is an (unknown) $d$-by-$d$ invertible matrix and $\underline{X}$ is the observable mixture. Under these constraints, Yeredor proves that the XOR model is invertible and there exists a unique transformation matrix to recover the independent components up to permutation ambiguity. As discussed in previous chapters, the complexity of the BICA is at least asymptotically linear in $2^d$. The AMERICA algorithm \citep{yeredor2011independent}, which assumes a XOR mixture, has a complexity of $O(d^2\cdot 2^d)$. The MEXICO algorithm, which is an enhanced version of AMERICA, achieves a complexity of $O(2^d)$ under some restrictive assumptions on the mixing matrix. \cite{attux2011immune} extend Yeredor's formulation for sources which are not necessarily independent. Specifically, under the same model (\ref{linear_model}), they suggest minimizing the difference between the sum of marginal entropies and the joint entropy (as in (\ref{eq:sum_ent_min_binary})), where 
$$\underline{Y}=W\underline{X}$$ 
and $W$ is a $d$-by-$d$ invertible matrix over the XOR field. In their work, \cite{attux2011immune} present an immune-inspired algorithm for minimizing (\ref{eq:sum_ent_min_binary}).  Their algorithm starts with a random "population" where each element in the population represents a valid transformation ($W$, an invertible matrix). At each step, the affinity function evaluates the objective $\left(1-\frac{1}{d}\sum_{j=1}^{d}{H(Y_j)}\right)$ for each element in the population, which is subsequently cloned. Then, the clones suffer a mutation process that is inversely proportional to their affinity, generating a new set of individuals. This new set is evaluated again in order to select the individual with highest affinity, for each group of clone individuals and their parent individual. The process is finished with a random generation of $d$ new individuals to replace the lowest affinity individuals in the population. The entire process is repeated until a pre-configured  number of repetitions is executed. Then, the solution with the highest affinity is returned. It is important to notice that the mutation phase is implemented with a random bit resetting routine, with the constraint of accepting only new individuals that form a nonsingular matrix (invertible transformation).
The use of this immune-inspired methodology for the binary ICA problem is further extended in \citep{silva2014michigan} and \citep{silva2014cobica}.\\

As discussed, in a different line of work \cite{barlow1989finding} suggest to decompose the observed signals ``as much as possible", with no assumption on the generative model. Barlow claim that such decomposition would capture and remove the redundancy of the data. However, he does not propose any direct method, and this hard problem is still considered open, despite later attempts  \citep{atick1990towards,schmidhuber1992learning,becker1996unsupervised}.

\noindent A major drawback in this line of work is the lack of theoretical guarantees on the results these algorithms achieve. Specifically, given a vector $\underline{X} \sim \underline{p}$, it is unclear what is the minimal value of (\ref{eq:sum_ent_min_binary}) we can hope for, even under linear transformations, $\underline{Y}=W\underline{X}$, where $W$ a binary invertible matrix, $W \in \{0,1\}^{d \times d}$. This means that practically, one shall apply every known linear BICA algorithm and choose the one that achieves the minimal value of  (\ref{eq:sum_ent_min_binary}).\\

\noindent Therefore, we would first like to suggest a naive yet highly efficient lower bound to  (\ref{eq:sum_ent_min_binary}), under invertible liner transformation,  $\underline{Y}=W\underline{X}$.

\section{Lower Bound on Linear BICA}
\label{linear_BICA_lowerbound}
 In his line of binary ICA work, Yeredor establishes a methodology based on a basic property of the binary entropy. He suggests that the binary entropy of the XOR of two independent binary variables is greater than each variables' entropy. Specifically, $H(U\oplus V) \geq H(U)$, where $U$ and $V$ are binary independent variables and $\oplus$ is the XOR operand. Unfortunately, there is no such guarantee when the variables are dependent. This means that in general, the entropy of the XOR of binary variables may or may not be greater than the entropy of each of the variables. \\

\noindent When minimizing  (\ref{eq:sum_ent_min_binary}) over $\underline{Y}=W\underline{X}$, we notice that each $Y_j$ is a XOR of several, possibly dependent, variables $\{X_1,\dots,X_d\}$. This means that naively, we may go over all possible subsets of  $\{X_1,\dots,X_d\}$ and evaluate their XOR. Specifically, we would like to calculate $U_i= A_{i1}X_1 \oplus A_{i2}X_2\oplus \ldots \oplus A_{id}X_d$ for all $i=1,\dots,2^d$, where each row of the matrix $A$ corresponds to a possible choice of subset of variables from the set $\{X_1,\dots,X_d\}$, $A_{ij} \in \{0,1\}^d$. Then, we shall evaluate the binary entropy of each $U_i$. A necessary condition for $W$ to be invertible is that it has no two identical rows. Therefore, we a lower bound on  (\ref{eq:sum_ent_min_binary}) may be achieved by simply choosing the $d$ rows of the matrix $A$ for which $H(U_i)$ are minimal.\\

\noindent Notice this lower bound is by no means tight or attainable. It defines a simple lower bound on (\ref{eq:sum_ent_min_binary}), which may be attained iff we are lucky enough to have chosen $d$ rows of the matrix $A$ which are linearly independent.     

\section{A Simple Heuristic for Linear BICA}
\label{linear_BICA_algo}
\noindent We now present our suggested approach for the linear BICA problem, based on the same methodology presented in the previous section.  
Again, we begin by evaluating all possible XOR operations $U_i= A_{i1}X_1 \oplus A_{i2}X_2\oplus \ldots \oplus A_{id}X_d$ for all $i=1,\dots,2^d$. Further, we evaluate the binary entropy of each $U_i$. We then sort the rows of $A$ according to the binary entropy values of their corresponding $U_i$. This means that the row which corresponds to the smallest binary entropy value among all $\left\{H(U_i)\right\}_{i=1}^{2^d}$ shall be the first in order. Then, the row with the second smallest value in $\left\{H(U_i)\right\}_{i=1}^{2^d}$ shall be the second in order and so on. Let us denote the sorted list of rows as $\tilde{A}$.\\

\noindent Our remaining challenge is to choose $d$ rows from $\tilde{A}$ such that the rank of these rows is $d$. Additionally, our objective suggests to choose rows which are located higher in $\tilde{A}$, as they result in a lower entropy. Our suggested greedy algorithm begins with an empty matrix $W$. It then goes over the rows in $\tilde{A}$ in ascending order. If the current row in $\tilde{A}$ is linearly independent of the rows in $W$ it adds it to $W$. Otherwise, it skips it  and proceeds to the next row in $\tilde{A}$.  
The algorithm terminates once $W$ is of full rank (which necessarily happens at some point).\\

\noindent Our suggested algorithm is obviously a heuristic method which selects linearly independent  rows from  $\tilde{A}$ in a no-regret manner. It achieves the lower bound presented in Section \ref{linear_BICA_lowerbound} in the case where the first $d$ rows in  $\tilde{A}$ are indeed linearly independent. \\

\noindent Although our suggested algorithm looks for $d$ linearly independent rows from $\tilde{A}$ in a greedy manner, we may still evaluate the average number of rows goes through in order to construct a full rank matrix $W$, as shown in \citep{shulman2003communication}. Assume we have already found $k$ linearly independent rows (rank$(W)=k$) and are now seeking for an additional independent row. Notice that there are $2^k$ rows which are linearly dependent of the rows we have already found (all possible linear combinations of these rows). Assume we uniformly draw (with return) a row from a list of all possible rows (of size $2^d$ rows). The probability of drawn row to be independent of the $k$ rows is simply $1-\frac{2^k}{2^d}$. Therefore, the number of draws needed in order to find another linearly interdependent row follows a geometric distribution with a parameter $1-\frac{2^k}{2^d}$ (as the draws are i.i.d.).
Then, the average number of draws is $\frac{1}{1-\frac{2^k}{2^d}}=\frac{2^d}{2^d-2^k}$. Denote the average total number of draws needed to construct a full rank matrix as $$\bar{L}(d)=\sum_{k=0}^{d-1}\frac{2^d}{2^d-2^k}.$$
It can be shown that  
$$\bar{L}(d)-d \leq 2$$ 
and
$$\lim_{d \rightarrow \infty} \bar{L}(d)-d = 1.606 \dots$$
\noindent This means that even if we choose rows from $\tilde{A}$ with replacement, our suggested algorithm skips up to $2$ rows on the average, before terminating with a full rank matrix $W$. Practically, it means that our greedy algorithm does not substantially deviate, on the average, from the lower bound presented in the previous section.

\section{Experiments}
Let us now conduct several experiments to demonstrate the performance of our suggested algorithm. In the first experiment we draw $n=10^6$ independent samples from a Zipf's law distribution with $s=1$ and varying alphabet sizes $m=2^d$: \begin{equation}\nonumber
P(k;s,q)=\frac{k^{-s}}{\sum_{m=1}^q m^{-s}}
\end{equation}
where $s$ is the skewness parameter. The Zipf's law distribution is a commonly used heavy-tailed distribution. This choice of distribution is further motivated in Section \ref{Application_to_Data_Compression}.
We evaluate the lower bound of (\ref{eq:sum_ent_min_binary}) over linear transformations, as discussed in Section \ref{linear_BICA_lowerbound}. We further apply our suggested linear algorithm (Section \ref{linear_BICA_algo}) and, in addition, apply the order permutation (Chapter \ref{Order_Permutation}). Figure \ref{fig:linear_BICA_vs_order_perm} demonstrates the results we achieve. We first notice that the difference between our suggested linear algorithm and the linear lower bound is fairly small, as expected. Moreover, we notice that the order permutation outperforms both methods quite significantly. This is simply since linear transformations are not very ``flexible" models as the dimension of the problem increases.  

\begin{figure}[h]
\centering
\includegraphics[width = 0.7\textwidth,bb= 40 180 550 590,clip]{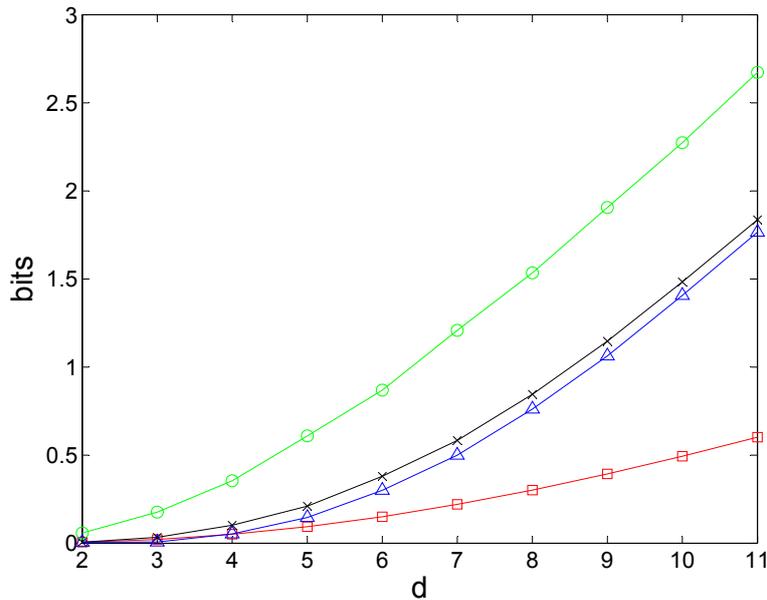}
\caption{Minimizing (\ref{eq:sum_ent_min_binary}) for independent draws from a Zipf distribution. Blue curve with the triangles: lower bound on linear transformation, black curve with the $X$'s: our suggested linear transformation, red curve with the squares: the order permutation, green curve with the circles: without applying any transformation}
\label{fig:linear_BICA_vs_order_perm}
\end{figure}

\noindent In addition, we would like to compare our suggested linear algorithm to the immune-inspired method \citep{silva2014cobica}. As before, we draw $n=10^6$ independent samples from a Zipf law distribution with $s=1$ and varying alphabet sizes $m=2^d$. Notice that this time we limit ourselves to a maximal dimension of $d=7$, as the cobICA algorithm \citep{silva2014cobica} fails to perform within a reasonable time frame for greater values of $d$ (more than several hours, using a standard personal computer). Figure \ref{fig:linear_BICA_vs_cobICA} demonstrates the results we achieve. It is easy to notice that our suggested algorithm outperforms the cobICA. Moreover, it takes significantly less time to execute (several seconds as opposed to almost an hour, for $d=7$). 

\begin{figure}[h]
\centering
\includegraphics[width = 0.7\textwidth,bb= 40 180 550 590,clip]{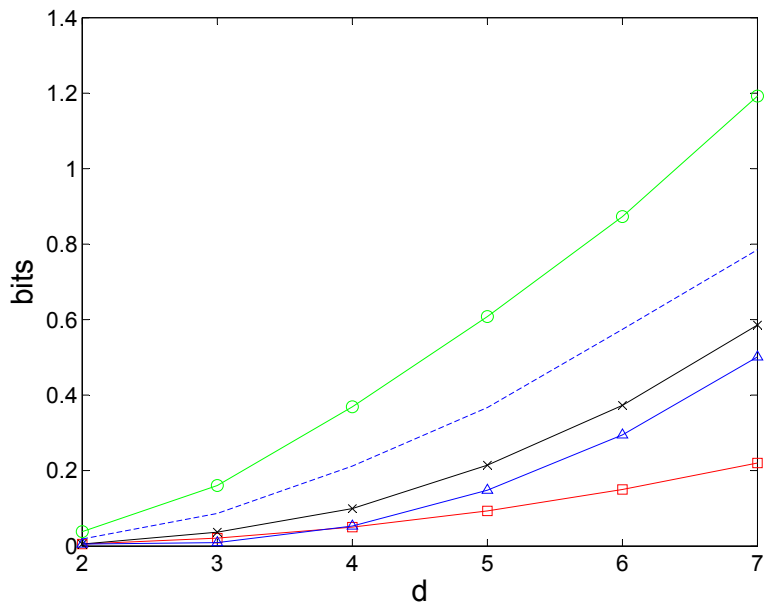}
\caption{Minimizing (\ref{eq:sum_ent_min_binary}) for independent draws from a Zipf distribution. The curves color and shapes correspond to the same methods as in Figure \ref{fig:linear_BICA_vs_order_perm}. In addition, the dashed blue curve is the cobICA}
\label{fig:linear_BICA_vs_cobICA}
\end{figure}

\noindent Lastly, we would like to empirically evaluate the expected value of  $C(\underline{p},g)$, when averaging uniformly over all possible $\underline{p}$ of an alphabet size $m=2^d$. In this experiment we go over all possible $\underline{p}$ for a given alphabet size $m=2^d$ and evaluate the lower bound of $C(\underline{p},g)$ over linear transformations, (Section \ref{linear_BICA_lowerbound}), our suggested linear algorithm (Section \ref{linear_BICA_algo}) and the order permutation (Chapter \ref{Order_Permutation}). Figure \ref{fig:linear_BICA_avg} demonstrates the results we achieve. As we can see, the linear lower bound converges to $0.6099$, which exactly equals to the value we derived analytically, for the case where no transformation is applied. This further justifies our claim that linear transformations are not powerful enough as minimizer for  (\ref{eq:sum_ent_min_binary}), when the dimension increases. We further notice that the order permutation converges to approximately $0.0162$, as expected  (see Chapter \ref{Order_Permutation}).

\begin{figure}[h]
\centering
\includegraphics[width = 0.7\textwidth,bb= 40 170 550 590,clip]{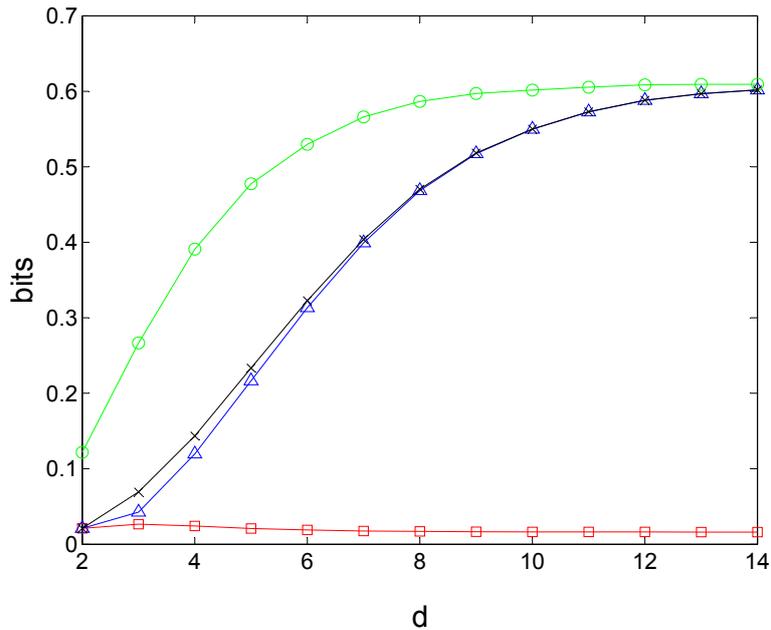}
\caption{Minimizing the expected value of  $C(\underline{p},g)$, when averaging uniformly over all possible $\underline{p}$ of an alphabet size $m=2^d$. The curves color and shapes correspond to the same methods as in Figure \ref{fig:linear_BICA_vs_order_perm}}
\label{fig:linear_BICA_avg}
\end{figure}

\section{Discussion}
Although the generalized ICA over finite fields problem was introduced quite a while ago, there is still a limited understanding on how well a random vector may be linearly decomposed  into independent components (as much as possible). In this chapter we proposed a novel lower bound for this problem, followed by a simple heuristic algorithm.Our suggested lower bound is not tight, in the sense that we cannot guarantee that there exists a linear transformation which achieves it. However, it provides an easy-to-evaluate benchmark on the best we can hope for. Moreover, our lower bound may be easily used to provide a feasible sub-optimal solution to the linear generalized ICA problem. This solution shows significantly outperform any currently known BICA methods, both in terms of accuracy and computational complexity, as demonstrated analytically and empirically.\\

\noindent Using the lower bound we developed, we showed that a linear transformation is not a favorable approach when the dimension of the problem increases. Specifically, we showed that the order permutation, presented in the previous chapter, incomparably outperforms any linear solution. Moreover, we show that on the average, applying a linear transformation is practically redundant, as it achieves the same results as if no transformation is applied. Clearly, this happens since the alphabet size increases exponentially with the number of components $m=2^d$, while the free parameters of the linear transformation increase only polynomially, $d^2$.

\newpage
\thispagestyle{empty}
\mbox{}
\thispagestyle{empty}

\chapter[Sequential Generalized ICA]{Sequential Generalized Independent Component Analysis}\label{Sequential_ICA}
\graphicspath{{Sequential_ICA_Figures//}}

\noindent The material in this Chapter is partly covered in \citep {painsky2013memoryless}.

\section{Introduction}
In this chapter we impose an additional constraint on the generalized ICA problem by limiting ourselves to sequential processing of the vector $\underline{X}$. Several methods have been suggested to sequentially construct an uncorrelated or independent process from a given stochastic process. The Gram-Schmidt procedure suggests a simple sequential method which projects every new component on the linear span of the components previously observed \citep{arfken1985gram}. The difference between the current component and its projection is guaranteed to be orthogonal to all previous components. Applied on a Gaussian process, orthogonality results statistical independence and the subsequent process is therefore considered \textit{memoryless}. Non-Gaussian processes on the other hand, do not hold this quality and a generalized form of sequentially generating a memoryless  process from any given time dependent series is therefore required. Several non-sequential methods such as Principal Components Analysis \citep{jolliffe2002principal} and Independent Component Analysis \citep{hyvarinen1998independent} have received a great deal of attention, but we are aware of a little previous work on sequential schemes for generating memoryless “innovation” processes. \\

\noindent The importance of innovation process representation spans a variety of fields. One example is dynamic system analysis in which complicated time dependent processes are approximated as independent processes triggering a dynamic system (human speech mechanism, for instance). Another major field for example is cryptography, where a memoryless language is easier to encrypt as it prevents an eavesdropper from learning the code by comparing its statistics with those of the serially correlated language.
Recently, \cite{shayevitz2011optimal} presented the Posterior Matching (PM) scheme for communication with feedback. It turns out that an essential part of their scheme is to produce statistical independence between every two consecutive transmissions. Inspired by this we suggest a general framework to sequentially construct memoryless processes from any given Markov process, for various types of desired distribution function, under different objective functions and constraints.\\

\section{Problem Formulation}
\noindent For the remaining sections of this chapter we use the following notation: we denote the input process at a time $k$ as $X_k$ while $X^k$ refers to the vector $\{X_i\}_{i=1}^k.$ We use the same notation for our outcome process $Y$. Therefore, for any process $X^k$ with a cumulative distribution function $F(X^k)$ we would like to sequentially construct $Y^k$ such that:
\begin{enumerate}
\item	$F(Y^k)=\prod_{j=i}^k F(Y_j)$
\item	$X^k$ can  be uniquely recovered from $Y^k$ for any $k$.
\end{enumerate}
\noindent Using the notation from previous chapters, we look for a sequential invertible transformation on the set of ``components" $\{X_j\}_{j=1}^k$, so that the resulting ``components" $\{Y_j\}_{j=1}^k$ are statistically independent.
\noindent We show that the two constraints can always be met if we allow the $Y_j$'s to take values on a continuous set and may need to be relaxed otherwise. The continuous case is discussed in the next section, followed by a comprehensive discussion on the discrete case in the remaining part of this chapter. 

\section{Generalized Gram-Schmidt}

\noindent Following the footsteps of the Posterior Matching scheme \citep{shayevitz2011optimal} we define a generalized Gram-Schmidt method for the continuous case.  

\begin{theorem}
\label{seq_theorem_1}
Let $X$ be any random variable $X \sim F_X (x)$  and $\theta \sim \text{Unif}[0,1]$ be statistically independent of it. In order to shape $X$ to a uniform distribution (and vice versa) the following applies:
\begin{enumerate}
\item	$F_X^{-1}(\theta) \sim F_X(x)$
\item	Assume $X$ is a non-atomic distribution ($F_X (x)$ is strictly increasing) then  $F_X(X)\sim \text{Unif}[0,1]$
\item	Assume $X$ is discrete or a mixture probability distribution then  $F_X(X)-\theta P_X(x) \sim \text{Unif}[0,1]$
\end{enumerate}
\end{theorem}
\noindent The proof of this theorem can be located in Appendix $1$ of \citep{shayevitz2011optimal}.

\noindent We define $\tilde{F}_X (x)$ as $\tilde{F}_X (x)=F_X (x)$ if $F_X (x)$ is strictly increasing and $\tilde{F}_X (x)=F_X (x)-\theta P_X (x)$ otherwise. For a desired $F_{Y_k}(y_k)$ we construct our process by setting:
\begin{equation}
Y_1=F_{Y_1}^{-1} \left(\tilde{F}_{X_1} (X_1)\right)
\end{equation}
\begin{equation}
Y_k=F_{Y_k}^{-1} \left(\tilde{F}_{X_k |X^{k-1}} (X_k |X^{k-1} )\right) \quad \forall k>1
\end{equation}
\noindent Theorem \ref{seq_theorem_1} guarantees that $\tilde{F}_{X_k |X^{k-1}} \left(X_k |X^{k-1} \right)$ is uniformly distributed and applying $F_{Y_k}^{-1}$ on it shapes it to the desired continuous distribution. In other words, this method suggests that for every possible history of the process at a time $k$, the transformation $\tilde{F}_{X_k |X^{k-1}} \left(X_k |X^{k-1}\right)$ shapes $X_k$ to the same (uniform) distribution. This ensures independence of its history. The method then reshapes it to the desired distribution. It is easy to see that $Y_k$ are statistically independent as every $Y_k$ is independent of $X^{k-1}$. Moreover, since $F(Y_k)$ is strictly increasing and $\tilde{F}_{X_1}(X_1)$ is uniformly distributed we can uniquely recover $X_1$ from $Y_1$ according to the construction of Theorem \ref{seq_theorem_1}. Simple induction steps show that this is correct for every $Y_k$ for  $k>1$. A detailed discussion on the uniqueness of this method is located in Appendix \ref{on_the_uniquness}.

\section{Lossy Transformation in the Discrete Case}
\noindent Let us now assume the both $X_j$ and $Y_j$ take values on finite alphabet size of $A$ and $B$ respectively (for every $j$). Even in the simplest case, where both are binary and $X$ is a first order non-symmetric Markov chain it is easy to see that no transformation can meet both of the constraints mentioned above. We therefore relax the second constraint by replacing the uniquely recoverable constraint with mutual information maximization of  $I\left(X_k;Y_k |X^{k-1} \right)$. This way, we make sure that the mutual information between the two processes is maximized at any time given its history. Notice that the case where $X_k$ is uniquely recoverable from $Y_k$ given its past, results in $ I\left(X_k;Y_k |X^{k-1} \right)$ achieving its maximum as desired.\\

\noindent This mutual information maximization problem is substantially different than the ICA framework presented in the previous chapters. Here, we insist on full statistical independence at the cost of lossy reconstruction, while in the previous chapters we focused on lossless reconstruction at the cost of  ``almost statistical independence". Our problem can be reformulated as follows:\\
For any realization of $X_k$, given any possible history the process $X^{k-1}$, find a set of mapping functions to a desired distribution $P(Y_k)$ such that the mutual information between the two processes is maximal. For example, in the binary case where $X_k$ is a first order Markov process, and $Y_k$ is i.i.d. Bernoulli distributed,
\begin{equation}
\label{markov_model}
Y_k \sim \text{Ber}(\beta), \quad P_{X_k}(X_k=0)=\gamma_k
\end{equation}
\begin{equation}
\nonumber
P_{X_k|X_{k-1}}\left(X_k=0|X_{k-1}=0\right)=\alpha_1
\end{equation}
\begin{equation}
\nonumber
P_{X_k|X_{k-1}}\left(X_k=0|X_{k-1}=1\right)=\alpha_2
\end{equation}
we would like to maximize
\begin{equation}
 I\left(X_k;Y_k | X^{k-1}\right)=\gamma_{k-1} I\left(X_k;Y_k | X_{k-1}=0\right)+(1-\gamma_{k-1}) I\left(X_k;Y_k | X_{k-1}=1\right)
\end{equation}
\noindent In addition, we would like to find the distribution of $Y_k$ such that this  mutual information is maximal. This distribution can be viewed as the closest approximation of the process $X$ as a memoryless process in terms of maximal mutual information with it. 
Notice that this problem is a concave minimization over a convex polytope shaped set \citep{kovacevic2012hardness} and the maximum is guaranteed on to lie on one of the polytope's vertices. Unfortunately, this is an NP hard problem and generally there is no closed form solution to it. Several approximations and exhaustive search solutions are available for this kind of problem, such as \citep{kuno2007simplicial}. There are, however, several simple cases in which such a closed form solution exists. One notable example is the binary case.

\subsection{The Binary Case}
\label{binary_case}
Let us begin by considering the following problem: given two binary random variables $X$ and $Y$ and their marginal distributions $P_X (X=0)=\alpha<\frac{1}{2}$ and $P_Y (Y=0)=\beta<\frac{1}{2}$ we would like to find the conditional distributions $P_{Y|X} (y|x)$ such that the mutual information between $X$ and $Y$ is maximal. Simple derivation shows that the maximal mutual information is: 
\noindent For $\beta>\alpha$:
\begin{equation}
\label{beta>alpha}
I_{\text{max}}^{\beta>\alpha}(X;Y)=h_b(\beta)-(1-\alpha)h_b\left(\frac{\beta-\alpha}{1-\alpha} \right)
\end{equation}

\noindent For $\beta<\alpha$:
\begin{equation}
\label{beta<alpha}
I_{\text{max}}^{\beta<\alpha}(X;Y)=h_b(\beta)-\alpha h_b\left(\frac{\beta}{\alpha} \right).
\end{equation}
Applying this result on the first order Markov process setup described above and assuming all parameters are smaller than $\frac{1}{2}$ , the maximal mutual information is simply:

\noindent For $\beta<\alpha_1<\alpha_2$:
\begin{equation}
 I\left(X_k;Y_k | X^{k-1}\right)=\gamma_{k-1} I_{\text{max}}^{\beta<\alpha_1}\left(X;Y\right)+(1-\gamma_{k-1}) I_{\text{max}}^{\beta<\alpha_2}\left(X;Y\right)
\end{equation}

\noindent For $\alpha_1 \leq \beta <\alpha_2$:
\begin{equation}
 I\left(X_k;Y_k | X^{k-1}\right)=\gamma_{k-1} I_{\text{max}}^{\beta>\alpha_1}\left(X;Y\right)+(1-\gamma_{k-1}) I_{\text{max}}^{\beta<\alpha_2}\left(X;Y\right)
\end{equation}

\noindent For $\alpha_1 <\alpha_2 \leq \beta $:
\begin{equation}
 I\left(X_k;Y_k | X^{k-1}\right)=\gamma_{k-1} I_{\text{max}}^{\beta>\alpha_1}\left(X;Y\right)+(1-\gamma_{k-1}) I_{\text{max}}^{\beta>\alpha_2}\left(X;Y\right)
\end{equation}

\noindent It is easy to see that $ I\left(X_k;Y_k | X^{k-1}\right)$ is continuous in $\beta$. Simple derivation shows that for $\beta<\alpha_1<\alpha_2$ the maximal mutual information is monotonically increasing in $\beta$ and for $\alpha_1 <\alpha_2 \leq \beta $ it is monotonically decreasing in $\beta$. It can also be verified that all optimum points in the range of  $\alpha_1 \leq \beta <\alpha_2$ are local minima which leads to the conclusion that the maximum must be on the bounds of the range, $\beta=\alpha_1$ or $\beta=\alpha_2$.  The details of this derivation is located in Appendix \ref{Appendix_mutual_inf_max_points}. Figure \ref{illustration} illustrates the shape of  $ I\left(X_k;Y_k | X^{k-1}\right)$ as a function of $\beta$, for $\alpha_1=0.15$, $\alpha_2=0.45$, for example.\\

\noindent  Since we are interested in the $\beta$ that maximizes the mutual information between the two possible options, we are left with a simple decision rule
\begin{equation}
\label{decision_rule}
\gamma_{k-1}\begin{array}{c} 
\beta=\alpha_2\\ \lessgtr \\\beta=\alpha_1\end{array} \frac{h_b(\alpha_2)-h_b(\alpha_1)+\alpha_2 h_b\left(\frac{\alpha_1}{\alpha_2}\right)}{\alpha_2 h_b \left(\frac{\alpha_1}{\alpha_2}\right)+(1-\alpha_1)h_b \left(\frac{\alpha_2-\alpha_1}{1-\alpha_1}\right)}
\end{equation}
\noindent which determines the conditions according to which we choose our $\beta$, depending on the parameters of the problem $\gamma_{k-1}, \alpha_1, \alpha_2$.\\

\begin{figure}[h]
\centering
\includegraphics[width = 0.8\textwidth,bb= 50 370 540 700,clip]{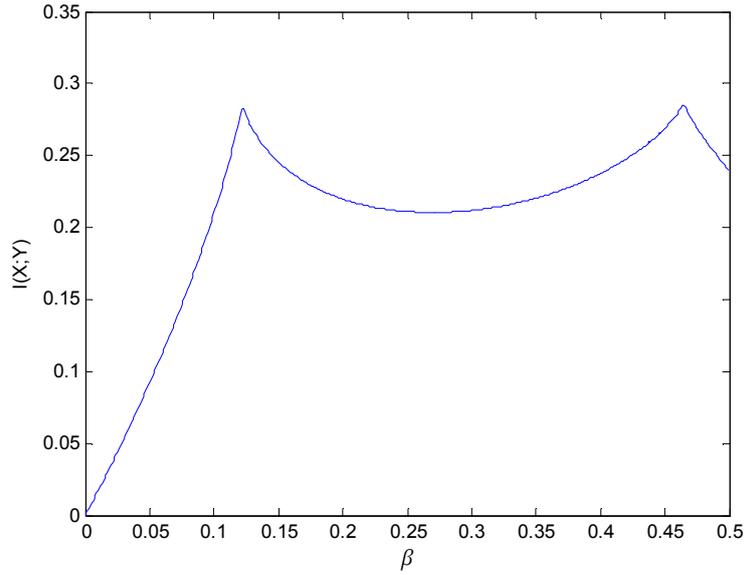}
\caption{The mutual information $ I\left(X_k;Y_k | X^{k-1}\right)$ as a function of $\beta$, for a first order Markov model (\ref{markov_model}), with $\alpha_1=0.15$, $\alpha_2=0.45$}
\label{illustration}
\end{figure}

\noindent Further, assuming the process $X$ is at its stationary state yields $\gamma=\frac{\alpha_2}{1-\alpha_1+\alpha_2}$. Applying this result to the decision rule above (\ref{decision_rule}), it is can be verified (Appendix \ref{Appendix_stationary_case_a_1_less_a_2}) that for $\alpha_1<\alpha_2<\frac{1}{2}$  we have:
\begin{equation}
\nonumber
\frac{\alpha_2}{1-\alpha_1+\alpha_2}< \frac{h_b(\alpha_2)-h_b(\alpha_1)+\alpha_2 h_b\left(\frac{\alpha_1}{\alpha_2}\right)}{\alpha_2 h_b \left(\frac{\alpha_1}{\alpha_2}\right)+(1-\alpha_1)h_b \left(\frac{\alpha_2-\alpha_1}{1-\alpha_1}\right)}
\end{equation}
which leads to the conclusion that $\beta_{opt}=\alpha_2$.\\

\noindent This derivation is easily generalized to all values of $\alpha_1$ and $\alpha_2$. This results in a decision rule stating that $\beta_{opt}$ equals the parameter closest to $\frac{1}{2}$:
\begin{equation}
\beta_{opt}=\argmax_{\theta \in \{\alpha_1,\alpha_2,1-\alpha_1,1-\alpha_1\}}\left(\frac{1}{2}-\theta\right).
\end{equation}
In other words, in order to best approximate a binary first order Markov process at its stationary state we set the distribution of the binary memoryless process to be similar to the conditional distribution which holds the largest entropy. \\

\noindent Expanding this result to an $r$-order Markov process we have $R=2^r$ Bernoulli distributions to be mapped to a single one ($P(Y_k)$). The maximization objective is therefore  
\begin{equation}
 I\left(X_k;Y_k | X^{k-1}\right)=\sum_{i=0}^{R-1} \gamma_i I\left( X_k;Y_k | \left[ X_{k-1}\, \dots \, K_{k-R-1}\right]^T=i\right)
\end{equation}
where $\gamma_i$ is the probability of the vector $\left[X_{k-1}\, \dots \, X_{k-R-1}\right]^T$ to be equal to its $i^{th}$ possible value, $\gamma_i=P\left(\left[X_{k-1}\,\dots \,X_{k-R-1}\right]^T=i\right)$. Notice that $I\left( X_k;Y_k | \left[ X_{k-1}\, \dots \, K_{k-R-1}\right]^T=i\right)$ is either $h_b(\beta)-\alpha_i h_b \left(\frac{\beta}{\alpha_i}\right)$ or $h_b(\beta)-(1-\alpha_i) h_b \left(\frac{\beta-\alpha_i}{1-\alpha_i}\right)$, depending on $\beta$ and $\alpha_i$,  as described above in (\ref{beta>alpha}),(\ref{beta<alpha}).\\

\noindent Simple calculus shows that as in the $R=2$ case, the mutual information $ I\left(X_k;Y_k | X^{k-1}\right)$ reaches its maximum on one of the inner bounds of $\beta$'s range
\begin{equation}
\beta_{opt}=\argmax_{\beta \in \{\alpha_i\}} \left(h_b(\beta)-\sum_{\beta<\alpha_i} \gamma_i \alpha_i h_b\left(\frac{\beta}{\alpha_i}\right)-\sum_{\beta>\alpha_i} \gamma_i (1-\alpha_i) h_b\left(\frac{\beta-\alpha_i}{1-\alpha_i}\right)\right)
\end{equation}
\noindent Here, however, it is not possible to conclude that $\beta$ equals the parameter closest to $\frac{1}{2}$. Simple counter example shows that it is necessary to search over all possible parameters, as a result of the nature of our concave minimization problem over a convex polytope.  

\section{Lossless Transformation in the Discrete Case}

The lossy approximation may not be adequate in applications where unique recovery of the original process is required. It is therefore necessary to increase the alphabet size of the output so that every marginal distribution of $X_k$, given any possible history of the process can be accommodated.  
This problem can be formulated as follows: 

\noindent Assume we are given a set of $R$ random variables, $\{X_i\}_{i=1}^R$, such that each random variable $X_i$ is multinomial distributed, taking on $A$ values,  $X_i \sim \text{multnom}\left(\alpha_{1i},\alpha_{2i},\dots,\alpha_{Ai}\right)$.  Notice that $A$ is the marginal alphabet size of the original process $X_k$, and $R$ corresponds to its Markov memory length $R=2^r$. Using the notation from previous sections, we have that $P(X_i)$ corresponds to $P(X_k | \left[ X_{k-1}\, \dots \, K_{k-R-1}\right]^T=i)$. In addition, we use the notation $x_{a;i}$ to define the $a^{th}$ value of the $i^{th}$  random variable $X_i$. We would like to find a distribution $Y \sim \text{multnom}\left(\beta_1,\beta_2,\dots,\beta_B\right)$  where the $\beta$'s and alphabet size $B \geq A$ are unknown. In addition, we are looking for $R$ sets of conditional probabilities between every possible realization $X_i=x_{a;i}$ and $Y$, such that $X_i=x_{a;i}$ can be uniquely recoverable from $Y=y_b$ for every $j,a$ and $b$. Further, we would like the entropy of $Y$ to be as small as possible so that our memoryless process is as ``cheap" as possible to describe. Notice that in terms of the generalized ICA framework, here we require both full statistical 
independence and unique recovery, at the cost of an increased objective (\ref{eq:sum_ent_min_binary}).

\noindent Without loss of generality we assume that $\alpha_{ai} \leq \alpha_{(a+1)i}$ for all $a \leq A$, since we can always denote them in such an order. We also order the sets according to the smallest parameter,  $\alpha_{1i} \leq \alpha_{1(j+i)}$. Notice we have $\alpha_{1i}\leq \frac{1}{2}$ for all $i=1,\dots, R$ , as an immediate consequence. \\

\noindent For example, for $A=2$ and $R=2$, it is easy to verify that $B \geq 3$ is a necessary condition for $X_i$ to be uniquely recoverable from $Y$. Simple calculus shows that the conditional probabilities which achieve the minimal entropy are $\beta_1=\alpha_1,\,  \beta_2=\alpha_2-\alpha_1$ and $\beta_3=1-\alpha_2$, as appears in Figure \ref{lossless_matching}.

\begin{figure}[h]
\centering
\includegraphics[width = 0.75\textwidth,bb= 50 520 540 720,clip]{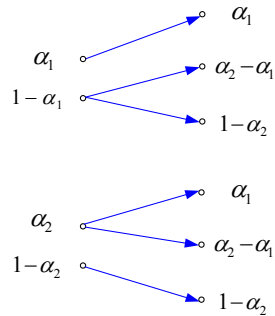}
\caption{Lossless representation of two binary sources with a single ternary source}
\label{lossless_matching}
\end{figure}

\subsection{Minimizing $B$}
Let us start with finding the minimal alphabet size of the output process $B$, such that the  $X$ is guaranteed to be uniquely recoverable from it.
Looking at the free parameters of our problem we first notice that defining the distribution of $Y$, takes exactly $B-1$ parameters. Then, defining $R$ conditional probability distributions between each alphabet size $A$ and the output process $Y$ takes $R(A-1)(B-1)$ parameters. In order for $X_i$'s to be uniquely recoverable from $Y$, each value of $Y$ needs to be at most assigned to a single value of $X_i$ (see Figure \ref{lossless_matching} for example). This means that for each of the $R$ sets, we have $B(A-1)$ constraints ($B$ possible realizations of $Y$ , each of them has $A-1$ zero conditional probability constraints). Therefore, in order to have more free parameters than constraints we require that:
\begin{equation}
(B-1)+R(A-1)(B-1) \geq RB(A-1).                          
\end{equation}
\noindent Rearranging this inequality leads to
\begin{equation}
B \geq R(A-1)+1.                    	
\end{equation}

\noindent For example, assuming the $X_i$'s are over a binary alphabet we get that $B\geq R+1$. There exist several special cases in which it is possible to go under this lower bound, like cases where some parameters are additions or subtraction of other parameters. For example, $\alpha_2=1-\alpha_1$ in the binary case. We focus however on solving the most general case. 

\subsection{The Optimization Problem}
\noindent The problem stated above can be formulated as the following optimization problem:
\begin{equation}
\min H(Y) \quad \text{s.t.} \quad H\left(X_i | Y=y_b\right)\leq0 \quad \forall i =\{1,\dots, R\}
\end{equation}
\noindent Unfortunately this is a concave minimization problem over a non-convex set. However, we show this problem can also be formulated as a mixed integer problem.

\subsection{Mixed Integer Problem Formulation}
\label{mixed_int}
In order to formulate our problem as a mixed integer problem we first notice the free parameters are all conditional probabilities, as they fully determine the outcome distribution. We use the notation $p_{iab}$ to describe the conditional probability $P(Y=y_b | X_i=x_{a;i})$.
\noindent Therefore, our bounds on the variables are $0 \leq p_{iab} \leq 1$ for all $i,\,a$ and $b$. The equality constraints we impose on our minimization objective are:

\begin{itemize}
\item		All $R$ conditional probability sets must result with the same output distribution:
\begin{align}
\nonumber
P(Y=y_b)=&\sum_{a=1}^A P\left(Y=y_b | X_i=x_{a;i}\right) P\left( X_i=x_{a;i}\right)=\sum_{a=1}^A p_{iab} \alpha_{ai}
\end{align}
\noindent for all $i=\{1,\dots,R\}$ and $b=\{1,\dots,B\}$. Since the parameters $\alpha_{1i},\dots,\alpha_{Ai}$ are assumed to be given we have that 
\begin{align}
\nonumber
\sum_{a=1}^A p_{iab}\alpha_{ai}-\sum_{a=1}^A p_{jab}\alpha_{aj}=0
\end{align}
\noindent for all $i,j=\{1,\dots,R\}$ and $b=\{1,\dots,B\}$ 

\item	$P(Y|X_i)$ is a valid conditional distribution function:
\begin{align}
\nonumber
\sum_{a=1}^A p_{iab}=1 
\end{align}
\noindent for all $i=\{1,\dots,R\}$,  $a=\{1,\dots,A\}$
 and  $b=\{1,\dots,B\}$
\item	$Y$ must be a valid probability function:
\begin{align}
\nonumber
\sum_{b=1}^B P(Y=y_b)=1.
\end{align}
In terms of $p_{iab}$: 
\begin{align}
\nonumber
\sum_{b=1}^B \sum_{a=1}^A p_{iab}\alpha_{ai}=1 
\end{align}
for all  $i=\{1,\dots,R\}$.
Notice that this constraint, together with all previous ones, follows that $P(Y)$ is also bounded by $0$ and $1$.
\end{itemize}

In addition, the inequality constraints are:
\begin{itemize}
\item	For convenience reasons we will ask that $P(Y=y_b ) \leq P(Y=y_{b+1} )$ for all $b=\{1,\dots,B\}$:
\begin{align}
\nonumber
\sum_{a=1}^A p_{iab}\alpha_{ai}-\sum_{a=1}^A p_{ia(b+1)}\alpha_{ai} \leq 0 \quad \text{for all} \quad 1 \leq b \leq B
\end{align}

\item Zero conditional entropy constraint:
As stated above, a necessary and sufficient condition for zero conditional entropy is that for every value $Y=y_b$, in every set $i=\{1,\dots,R\}$, there is only a single value $X_i=x_{a;i}$ such that $p_{iab}>0$.
Therefore, for each of the $R$ sets, and for each of the $B$ values $Y$ can take on, we define $A$ boolean variables, $T_{iab}$, that must satisfy:
\begin{equation}
\nonumber
p_{iab}-T_{iab} \leq 0
\end{equation}
\begin{equation}
\nonumber
\sum_{a=1}^A T_{iab} =1
\end{equation}
\begin{equation}
\nonumber
T_{iab} \in \{0,1\}
\end{equation}
Notice that the summation ensures only a single $T_{iab}$ equals one, for which $p_{iab}\leq1$. For each of the other $T_{iab}=0$ the inequality constraint verifies that $p_{iab}\leq0$.
This set of constraints can also be written using $A-1$ Boolean variables:
\begin{equation}
\nonumber
p_{iab}-T_{iab} \leq 0 \quad  \forall \,\, a=\{1,\dots,A\}
\end{equation}
\begin{equation}
\nonumber
p_{iAb}-\left(1-\sum_{a=1}^{A-1}T_{iab}\right) \leq 0 \quad \Leftrightarrow \quad p_{iAb}+\left(\sum_{a=1}^{A-1}T_{iab}\right) \leq 1 
\end{equation}
\begin{equation}
\nonumber
T_{iab} \in \{0,1\} \quad \forall \,\, a=\{1,\dots,A\}
\end{equation}

\end{itemize}

\noindent Therefore, our minimization problem can be written as follows:
Define a vector of parameters $z=\left[p_{iab} \quad T_{iab}\right]^T$. 
Define $A_{eq}$ and $b_{eq}$ as the equality constraints in a matrix and vector forms respectively.
Define $A_{ineq}$ and $b_{ineq}$ as the inequality constraints in a matrix and vector forms respectively. This leads to
\begin{equation}
\label{f(z)}
\min f(z) 
\end{equation}
\begin{equation}
\nonumber
\text{s.t.} \quad A_{eq}z=b_{eq}
\end{equation}
\begin{equation}
\nonumber
A_{ineq}z=b_{ineq}
\end{equation}
\begin{equation}
\nonumber
0\leq z \leq 1
\end{equation}
\begin{equation}
\nonumber
z(\text{boolean indicators}) \in \{0,1\}
\end{equation}
where f(z) is the entropy of the random variable $Y$ in terms of $p_{iab}$ and \textit{boolean indicators} define which elements in $z$ correspond to $T_{iab}$.

\subsection{Mixed Integer Problem Discussion}
Mixed integer problems are studied broadly in the computer science community. There are well established methodologies for convex minimization in a mixed integer problem and specifically in the linear case \citep{floudas1995nonlinear,tawarmalani2004global}. The study of non-convex optimization in mixed integer problem is also growing quite rapidly, though there is less software available yet. The most broadly used mixed integer optimization solver is the CPLEX\footnote{\url{http://www-01.ibm.com/software/commerce/optimization/cplex-optimizer/}} , developed by IBM. CPLEX provides a mixed integer linear programming (MILP) solution, based on a branch and bound oriented algorithm. We use the MILP in lower bounding our objective function (\ref{f(z)}) as described in the following sub-sections.

\subsection{An Exhaustive Solution}
As shown in Section \ref{mixed_int}, the problem we are dealing with is a hard one and therefore we present an exhaustive method which searches over all valid solutions to find the minimal entropy.
We notice that each of the given parameters $\alpha_{ai}$ can be expressed as a convex combination of the free parameters $\beta_b$ such that $A\underline{\beta}=\underline{\alpha}_i$, where $A$ represents the convex coefficients, $\underline{\beta}$ is a vector of the $\beta_b$'s and $\underline{\alpha}_i$ is a vector of the $\alpha_{ai}$'s. Additionally, it is easy to notice that the matrix $A$ must be a boolean matrix, to ensure the zero conditional entropy constraint stated above. Moreover, a necessary condition for the recovery of $\underline{\beta}$ from $A$ and $\underline{\alpha}_i$ is that $A$ is of full rank. However, this is not a sufficient condition since there is no guarantee that $\underline{\beta}$ is a valid probability distribution. This means we need to search over all boolean matrices $A$ of a full rank, and for each of these matrices check if the resulting $\underline{\beta}$ is a valid probability distribution. If so, we calculate its entropy and proceed. This process grows exponentially with $K$ (and $R$) but may be feasible for smaller values of these figures. 

\subsection{Greedy Solution}

The entropy minimization problem can also be viewed as an attempt to minimize the entropy of a random variable $Y \sim \text{multinom}\left(\beta_1,\beta_2,\dots,\beta_B \right)$ on a set of discrete points representing valid solutions to the problem we defined. 
Let us remember that $\beta_b\leq \beta_{b+1}$ for all $b=\{1,\dots,B\}$ as stated in the previous sections. 
\begin{proposition}
$\beta_B$ is not greater than $\min_i \left\{1-\sum_{a=1}^{A-1} \alpha_{ai}\right\}$
\end{proposition}
\begin{proof}
Assume $\beta_B>\min_i \left\{1-\sum_{a=1}^{A-1} \alpha_{ai}\right\}$. Then, for this $i$ there must be at least two values $x_{u;i}$ and $x_{v;i}$ for which $p_{iub}>0$ and $p_{ivb}>0$. This contradicts the zero conditional entropy constraint 
\end{proof}

\noindent Therefore, a greedy algorithm would  like to “squeeze” all the distribution to the values which are less constrained from above, so that it is as large as possible. We then suggest that in every step of the algorithm we set $\beta_B=\min_i \left\{1-\sum_{a=1}^{A-1} \alpha_{ai}\right\}$ which leaves us with a $B-1$ problem. Rearranging the remaining probabilities and repeating this maximization step, ensures that in each step we increase the least constrained  value of $\beta$ as much as possible. However, It is easy to notice that this solution is not optimal through simple counter examples. 

\subsection{Lowest Entropy Bound}
As discussed in the previous sections, we are dealing with an entropy minimization problem over a discrete set of valid solutions. Minimizing the entropy over this set of points can be viewed as a mixed integer non convex minimization, which is a hard problem. \\
\noindent However, we can find boundaries on each of the parameters  $\beta_b$ and see the lowest entropy we can hope for. This way, we relax the search over a set of valid solutions to a search in a continuous space, bounded by a polytope. 
We find the boundaries of each $\beta_b$ by changing our minimization objective to a simpler linear one (minimize/maximize $\beta_b$). This way we find a valid solution for which $\beta_b$ is at its bound. This problem is a simple MILP as shown above.  By looking at all these boundaries together and minimizing the entropy in this continuous space we can find a lower bound for the minimal entropy one can expect. We notice that this bound is not tight, and we even do not know how far it is from the valid minima, as it is not necessarily a valid solution. However, it gives us a benchmark to compare our greedy algorithm against and decide if we are satisfied with it or require more powerful tools. We also note that as we increase $B$, the number of valid solutions grows exponentially. This leads to a more packed set of solutions which tightens the suggested lower bound as we converge to a polytope over a continuous set. 

\section{Applications}
As mentioned before, the problem we are dealing with has a vast amount of applications in multiple fields as it deals with a very fundamental problem. Besides the memoryless representation problem which comes from the stochastic signal processing world, we can identify other applications from entirely different domains. One example is the following economic problem dealing with optimal design of a mass production storage units.\\

\noindent Consider the following problem: a major home appliances vendor is interested in mass manufacture of storage units. These units hold a single and predetermined design plan according to the market demand. Let us further assume that the customers market is defines by $R$ major storing types (customers) and each of these storing types is characterized by a different distribution of items it wishes to store. The vendor is therefore interested in designing a single storage unit that suits all of its customers. In addition, the vendor would like to storage unit to be as ``compact" and ``cheap" as possible. We denote this problem as \textit{The IKEA Problem}.

\subsection{The IKEA Problem}
We consider the $R$ storage distributions as $\{X_i\}_{i=1}^R$ such that each storing type $X_i$ is multinomial distributed with $A$ values, $X_i \sim \text{multnom}(\alpha_{1i},\alpha_{2i},\dots,\alpha_{Ai})$. We assume that all storage distributions have the same cardinality $A$. It is easy to generalize our solution to different cardinalities. 
As in previous sections, we use the notation $x_{a;i}$ to define the $a^{th}$ value of the $i^{th}$  random variable $X_i$. For our storing units problem, we would like to find a multinomial distribution over $B$ values ($B\geq A$ is unknown), $Y \sim \text{multnom}(\beta_1,\beta_2,\dots,\beta_B )$,  and $R$ sets of conditional probabilities between every $X_i=x_{a;i}$ and $Y$, such that $X_i=x_{a;i}$ can be uniquely recoverable (reversible) from $Y=y_b$ for every $i,\,a$ and $b$. This means every customer is able to store its items “exclusively”; different items will not need to be stored together.
In addition, we would like the storing unit to be ``compact" and ``cheap". For most functionalities, a compact storing unit is rectangular shaped (closets, cabins, dressers etc.) and it is made of multiple compartments (shelves) in numerous columns.  We define the number of columns in our storage unit as $L$ and the number of shelves as $N$. 
We would therefore like to design a rectangular shaped storing unit such that given a number of columns $L$, every costumer is able to store its items exclusively and the number of shelves is minimal.   
This problem is again NP hard for the same reasons as in the previous sections, but it can be reformulated to a set of Mixed Integer Quadratic Programming (MIQP) which is quite an established research area with extensive software available.

\subsection{Mixed Integer Quadratic Programming Formulation}
Let us first assume we are given both the number of columns in our desired storing unit $L$ and the number of shelves $N$. Since we require the storing unit to be rectangular, we need to find such distribution $Y$ that can be partitioned to $L$ columns with no residue. Therefore, we define $L$ equivalent partitions $\{\delta_l\}_{l=1}^L$   in the size of $\frac{1}{L}$ for which each $\{\beta_b\}_{b=1}^B$ is exclusively assigned. We are interested in such distribution $Y$ that the assignment can be done with no residue at all. 
To guarantee an exclusive assignment for a partition $\delta_l$ we introduce $T$ integer variables $\{T_{lb}\}_{b=1}^B$, indicating which of the $\{\beta_b\}_{b=1}^B$  is assigned to it. Therefore, we have

\begin{equation}
\sum_{b=1}^B T_{lb}\beta_b=\delta_l \quad \text{for all} \quad l=\{1,\dots,L\}
\end{equation}
\begin{equation}
\nonumber
\sum_{l=1}^L T_{lb}=1 \quad \text{for all}  \quad b=\{1,\dots,B\}
\end{equation}
\begin{equation}
\nonumber
T_{lb} \in \{0,1\}
\end{equation}
and the optimization objective is simply
\begin{equation}
\sum_{l=1}^L \left(\delta_l-\frac{1}{L}\right)^2 \rightarrow \text{min}
\end{equation}
Our constraints can easily be added to the mixed integer formulation presented in the previous sections and the new optimization problem is therefore: 
\begin{equation}
\min z^T cc^T z -\frac{2}{L} c^T z
\end{equation}
\begin{equation}
\nonumber
\text{s.t.} A_{eq} z = b_{eq}
\end{equation}
\begin{equation}
\nonumber
A_{ineq} z \leq b_{ineq}
\end{equation}
\begin{equation}
\nonumber
0 \leq  z \leq 1
\end{equation}
\begin{equation}
\nonumber
z(\text{boolean indicators}) \in \{0,1\} 
\end{equation}
where $z$ is a vector of all parameters in our problem  $z=\left[ p_{iab} \,\,  T_{lb}\right]^T$ and $c^T z=\delta$.

\subsection{Minimizing the Number of Shelves }
As demonstrated in the previous sections, the problem of minimizing the residue of the assignment given the number of columns and the number of shelves can be formulated as MIQP. In this section we focus on finding the minimal number of shelves $N$ that guarantees zero residue.  Notice that for large enough $N$ the residue goes to zero, as $Y$ tends to take values on a continuous set. We also notice that the residue is a monotonically non-increasing function of $N$, since by allowing a greater number of shelves we can always achieve the same residue by repeating the previous partitioning up to a meaningless split of one of the compartments. These two qualities allow very efficient search methods (gradient, binary etc.) to find the minimal $N$ for which the residue is ``$\epsilon$ close" to zero. 

\subsubsection{A Binary Search Based Algorithm}
The following simple binary search based algorithm for minimizing the number of shelves for a rectangular shaped storing unit is therefore suggested:

\begin{enumerate}
\item	Choose a large enough initial value $N$ such that applying it in the MIQP presented above results with zero residue.
\item	Define a step size as $Stp = \lfloor N/2 \rfloor$
\item	Apply the MIQP with $N'=N-Stp$
\item	If the residue is zero repeat previous step with $N=N'$ and $Stp=\lfloor Stp/2 \rfloor$. Otherwise repeat the previous step with $N=N'$ and $Stp=-\lfloor Stp/2 \rfloor$. Terminate if $Stp=0$.
\end{enumerate}

\section{Memoryless Representation and its Relation to the Optimal Transportation Problem}

As discussed in the previous sections, the essence of the our suggested problem formulation is finding a single marginal distribution function to be matched to multiple ones under varying costs functions. This problem can be viewed as a design generalization of a multi-marginal setup for the well-studied optimal transportation problem \citep{monge1781memoire}.  In other words, we suggest that the optimal transportation problem can be generalized to a design problem in which we are given not a single but multiple source probability measures. Moreover, we interested not only in finding mappings that minimizes some cost function, but also in finding the single target probability measure that minimizes that cost.

\subsection{The Optimal Transportation Problem}

The optimal transportation problem was presented by \cite{monge1781memoire} and has generated an important branch of mathematics in the last decades. The optimal transportation problem has many applications in multiple fields such as Economics, Physics, Engineering and others. The problem originally studied by Monge was the following: assume we are given a pile of sand (in $\mathbb{R}^3$) and a hole that we have to completely fill up with that sand. Clearly the pile and the hole must have the same volume and different ways of moving the sand will give different costs of the operation. Monge wanted to minimize the cost of this operation. Formally, the optimal transportation problem is defined as follows. Let $X$ and $Y$ be two seperable metric spaces such that any probability measure on $X$ (or $Y$) is a Radon measure. Let $c:X \times Y \rightarrow[0,\infty]$ be a Borel-measurable function. Given probability measure $\mu$ on $X$ and $\nu$ on $Y$, Monge's optimal transportation problem is to find a mapping $T:X \rightarrow Y$ that realizes the infimum 
$$ \inf \left\{ \int_X c(x,T(x))d\mu(x) \bigg| T_{*}(\mu)=\nu  \right\}$$
where $T_{*}(\mu)$ denotes the \textit{push forward} of $\mu$ by $T$. A map $T$ that attains the infimum is called the \textit{optimal transport map}. 

Notice that this formulation of the optimal transportation problem can be ill-posed as in some setups in which there is no ``one-to-one" transportation scheme. For example, consider the case where the original pile is a Dirac measure but hole is not shaped in this manner. A major advance on this problem is due to \cite{kantorovich1942translocation} who proposed the notation of a ``weak solution" to the optimal transportation problem; he suggested looking for plans instead of transport maps \citep{kantorovich2006problem}. The main difference between Kantorovich work and Monge formulation is that while the original Monge problem is restricted to transportation of the complete mass at each point on the original pile, the relaxed Kantorovich version allows splitting of masses.   Kantorovich argued that the problem of showing existence of optimal transport maps reduces to prove that an optimal transport plan in concentrated in a graph. It is however clear that no such result can be expected without additional assumptions on the measures and cost. The first existence and uniqueness result is due to \cite{brenier1987polar}. In his work, Brenier considers the case where both the pile $X$ and the hole $Y$ satisfy $X=Y \in R^n$, and the cost function is  $c(x,y)=|x-y|^2$. He then showed that if the probability measure of $X$ is absolutely continuous with respect to the Lebesgue measure there exists a unique optimal transport map. After this result many researchers started working on this problem, showing existence of optimal maps with more general costs both in the Euclidean setting (for example, \cite{ambrosio2003lecture,caffarelli2002constructing,evans1997partial,evans1999differential,evans2015measure, ambrosio2003existence,trudinger2001monge}).

\subsection{A Design Generalization of the Multi-marginal Optimal Transportation Problem}
Recently, Pass published a series of papers discussing a multi-marginal generalization of the optimal transportation problem \citep{pass2011uniqueness,pass2012local,pass2013class}. In his work, Pass considers multiple marginal distributions to be matched to a single destination with a given distribution. In his papers, Pass discusses the existence and uniqueness of solutions for both a Monge-like and Kantorovich-like multi-marginal problems, under different measures and cost functions and the connection between both formulations.\\

\noindent  In our work we generalize the multi-marginal optimal transportation from a design perspective; we look at the multi-marginal optimal transportation problem not only as a minimization problem over a set of mappings but also ask ourselves what is the optimal target measure such that the cost function is minimal. We show that this problem has very broad use in many fields, especially when taking an equivalent form of multiple source measures matched to a single target. More specifically, we focus our interest on a set of mappings that allow unique recovery between the measures. That is, given a source measure and a the target measure one can uniquely recover any realization of the sources from a given realization of the target. This type of mappings hold a special interest in many applications, as it is shown throughout this chapter.
%\subsection{Entropy Minimization on a Polytope Set}  
%The following proposition suggests an optimal method of entropy minimization over a polytope (thus convex) set. 
%\begin{proposition} Assume a random variable $Y$ with a multinomial distribution, $Y \sim \text{multinom}(\beta_1,\beta_2,\dots,\beta_B)$,  where $\beta_i$ are parameters bounded such that:
%\begin{enumerate}
%\item $a_i \leq \beta_i \leq b_i \quad \forall i$
%\item $\sum_i a_i \leq 1$ and $\sum_i b_i \geq 1$ (to ensure a feasible solution)
%\item $a_i \leq a_{i+1}$ and $b_i \leq b_{i+1} \quad \forall i$
%\item $a_i \leq b_i \quad \forall i$ 
%\end{enumerate}
%\noindent Then, the minimal entropy is achieved iff there exists $k>0$ such that
%\begin{enumerate}
%\item $\beta_i = b_i  \quad \forall i>k$
%\item  $\beta_i = a_i  \quad \forall i<k$
%\item  $\beta_k = 1-\sum_{i \neq k} \beta_i$
%\end{enumerate}
%\end{proposition}
%\noindent The proof of this proposition is located in the Appendix.

\section{Discussion}
In this chapter we presented a sequential non-linear method to generate a memoryless process from any process under different objectives and constraints. We show there exists a simple closed form solution if we allow the outcome process to take values on a continuous set. However, restricting the alphabet may cause lossy recovery of the original process. Two solutions are presented in the face of two possible objectives in the discrete case. First, assuming the alphabet size is too small to allow lossless recovery we aim to maximize the mutual information with the original process. The second objective focuses on finding a minimal alphabet size so that a unique recovery is guaranteed, while minimizing the entropy of the resulting process.  In both cases the problem is shown to be hard and several approaches are discussed. In addition, a simple closed-form solution is provided for a binary first order Markov process.  \\

\noindent The problem of finding a single marginal distribution function to be fitted to multiple ones under varying costs functions can be viewed as a multi-marginal generalization of the well-studied optimal transportation problem.  In other words, we suggest that the optimal transportation problem can be generalized to a design problem in which we are given not a single but multiple source distribution functions. We are then interested not only in finding conditional distributions to minimize a cost function, but also in finding the single target distribution that minimizes the cost. We conjecture that this problem has multiple applications in the fields of Economics, Engineering and others.

\newpage
\thispagestyle{empty}
\mbox{}
\thispagestyle{empty}

\chapter[ICA Application to Data Compression]{ICA Application to Data Compression}\label{Application_to_Data_Compression}
\graphicspath{{Application_to_Data_Compression_Figures//}}
\noindent The material in this Chapter is partly covered in \citep{painsky2015Universal,painsky2016largetrans, painsky2016Simple}.

\section{Introduction}
\noindent Large alphabet source coding is a basic and well--studied problem in data compression. It has many applications such as compression of natural language text, speech and images. The classic perception of most commonly used methods is that a source is best described over an alphabet which is at least as large as the observed large alphabet. Here, we challenge this approach and introduce a conceptual framework in which a large alphabet source is decomposed into ``as statistically independent as possible" components. This decomposition allows us to apply entropy encoding to each component separately, while benefiting from their reduced alphabet size.  We show that in many cases, such decomposition results in a sum of marginal entropies which is only slightly greater than the entropy of the source. \\

\noindent Assume a source over an alphabet size $m$, from which a sequence of $n$ independent samples are drawn. The classical source coding problem is concerned with finding a sample-to-codeword mapping, such that the average codeword length is minimal, and the codewords may be uniquely decodable. This problem was studied since the early days of information theory, and a variety of algorithms \citep{huffman1952method,witten1987arithmetic} and theoretical bounds \citep{cover2012elements} were introduced throughout the years.\\ 

\noindent The classical source coding problem usually assumes an alphabet size $m$ which is  small, compared with $n$. Here, in large alphabet compression, we focus on a more difficult (and common) scenario, where the source's alphabet size is considered ``large" (for example, a word-wise compression of natural language texts).
In this setup, $m$ takes values which are either comparable (or even larger) than the length of the sequence $n$. The main challenge in large alphabet source coding is that the redundancy of the code, formally defined as the excess number of bits used over the source's entropy, typically increases with the alphabet size \citep{davisson1973universal},  in any compression method where the source statistics is not precisely known in advance. \\

\noindent In this chapter we propose a conceptual framework for large alphabet source coding, in which we reduce the alphabet size by decomposing the source into multiple components which are ``as statistically independent as possible". This allows us to encode each of the components separately, while benefiting from the reduced redundancy of the smaller alphabet. To utilize this concept we introduce a framework based on the generalized ICA method (Section \ref{Generalized_BICA}). This framework efficiently searches for an invertible transformation  which minimizes the difference between the sum of marginal entropies (after the transformation is applied) and the joint entropy of the source. Hence, it minimizes the (attainable) lower bound on the average codeword length, when applying marginal entropy coding.\\

\noindent We demonstrate our method in a variety of large alphabet source coding setups. This includes even the classical lossless coding, where the probability distribution of the source is known both to the encoder and the decoder, universal lossless coding, in which the decoder is not familiar with the distribution of the source, and lossy coding in the form of vector quantization. We show that our approach outperforms currently known methods in all these setups, for a variety of typical sources.

\section{Previous Work}
\noindent In the classical lossless data compression framework, one usually assumes that both the encoder and the decoder are familiar with the probability distribution of the encoded source, $\underline{X}$. Therefore, encoding a sequence of $n$ memoryless samples drawn form this this source takes on average at least $n$ times its entropy $H\left(\underline{X}\right)$, for sufficiently large $n$ \citep{cover2012elements}. In other words, if $n$ is large enough to assume that the joint empirical entropy of the samples, $\hat{H}\left(\underline{X}\right)$, is close enough to the true joint entropy of the source, $H\left(\underline{X}\right)$, then $H\left(\underline{X}\right)$ is the minimal average number of bits required to encode a source symbol.
Moreover, it can be shown \citep{cover2012elements} that the minimum average codeword length,  $\bar{l}_{min}$,  for a uniquely decodable code, satisfies
\begin{equation}
\label{optimal_code}
H\left(\underline{X}\right) \leq \bar{l}_{min} \leq H\left(\underline{X}\right)+1.
\end{equation}

\noindent Entropy coding is a lossless data compression scheme that strives to achieve the lower bound, $\bar{l}_{min}=H\left(\underline{X}\right)$. Two of the most common entropy coding techniques are Huffman coding \citep{huffman1952method} and arithmetic coding \citep{witten1987arithmetic}.The Huffman algorithm is an iterative construction of variable-length code table for encoding the source symbols. The algorithm derives this table from the probability of occurrence of each source symbol. Assuming these probabilities are dyadic (i.e., $-\log \, p(\underline{x})$ is an integer for every symbol $\underline{x} \in \underline{X}$), then the Huffman algorithm achieves $\bar{l}_{min}=H\left(\underline{X}\right)$. However, in the case where the probabilities are not dyadic then the Huffman code does not achieve the lower-bound of (\ref{optimal_code}) and may result in an average codeword length of up to  $H\left(\underline{X}\right)+1$ bits.
Moreover, although the Huffman code is theoretically easy to construct (linear in the number of symbols, assuming they are sorted according to their probabilities) it is practically a challenge to implement when the number of symbols increases \citep{moffat1997implementation}.  Huffman codes achieve the minimum average codeword length among all uniquely decodable codes that assign a separate codeword to each symbol. However, if the probability of one of the symbols is close to $1$, a Huffman code with an average codeword length close to the entropy can only be constructed if a large number of symbols is jointly coded. The popular method of arithmetic coding is designed to overcome this problem. \\

\noindent In arithmetic coding, instead of using a sequence of bits to represent a symbol, we represent it by a subinterval of the unit interval \citep{witten1987arithmetic}. This means that the code for a sequence of symbols is an interval whose length decreases as we add more symbols to the sequence. This property allows us to have a coding scheme that is incremental. In other words, the code for an extension to a sequence can be calculated simply from the code for the original sequence. Moreover, the codeword lengths are not restricted to be integral. The arithmetic coding  procedure achieves an average length for the block that is within $2$ bits of the entropy. Although this is not necessarily optimal for any fixed block length (as we show for Huffman code), the procedure is incremental and can be used for any block-length. Moreover, it does not require the source probabilities to be dyadic.  However, arithmetic codes are more complicated to implement and are a less likely to practically achieve the entropy of the source as the number of symbols increases. More specifically, due to the well-known underflow and overflow problems, finite precision implementations of the traditional adaptive arithmetic coding cannot work if the size of the source exceeds a certain limit \citep{yang2000universal}. For example, the widely used arithmetic coder by \cite{witten1987arithmetic} cannot work when the alphabet size is greater than $2^{15}$. The improved version of arithmetic coder by \cite{moffat1998arithmetic} extends the alphabet to size $2^{30}$ by using low precision arithmetic, at the expense of compression performance.  \\

\noindent Notice that a large number of symbols not only results in difficulties in implementing entropy codes: as the alphabet size increases, we require a growing number of samples for the empirical entropy to converge to the true entropy. Therefore, when dealing with sources over large alphabets we usually turn to a universal compression framework. Here, we assume that the empirical probability distribution is not necessarily equal to the true distribution and henceforth unknown to the decoder. This means that a compressed representation of the samples now involves with two parts -- the compressed samples and an overhead redundancy (where the redundancy is defined as difference between the number of bits used to transmit a message and the entropy of the sequence). \\

\noindent As mentioned above, encoding a sequence of $n$ samples, drawn from a memoryless source $\underline{X}$, requires at least $n$ times the empirical entropy, $\hat{H}(\underline{X})$. Assuming that an optimal codebook is assigned for sequence, after it is known, $n\hat{H}(\underline{X})$ is also the codelength of the sequence. The redundancy, on the other hand, may be quantified in several ways. One common way of measuring the coding redundancy is through the minimax criterion \citep{davisson1973universal}. Here, the \textit{worst-case redundancy} is the lowest number of extra bits (over the empirical entropy) required in the worst case (that is, among all sequences) by any possible encoder.
Many worst-case redundancy results are known when the source's alphabet is finite. A succession of papers initiated by \cite{shtarkov1977coding} show that for the collection $\mathcal{I}_m^n$ of i.i.d. distributions over length-$n$ sequences drawn from an alphabet of a fixed size $m$, the worst-case redundancy behaves asymptotically as $\frac{m-1}{2}\log{\frac{n}{m}}$, as $n$ grows. \cite{orlitsky2004speaking} extended this result to cases where $m$ varies with $n$. The \textit{standard compression} scheme they introduce differentiates between three situations in which $m=o(n)$, $n=o(m)$ and $m=\Theta(n)$. They provide leading term asymptotics and bounds to the worst-case minimax redundancy for these ranges of the alphabet size.  \cite{szpankowski2012minimax} completed this study, providing the precise asymptotics to these ranges. For the purpose of our work we adopt the leading terms of their results, showing that the worst-case minimax redundancy, when $m \rightarrow \infty$, as $n$ grows, behaves as follows:
\begin{itemize} 
\item For $m=o(n)$: \hfill \makebox[0pt][r]{%
            \begin{minipage}[b]{\textwidth}
              \begin{equation}
		\label{m=o(n)}
                 \quad \quad \quad \quad  \quad \quad \hat{R}(\mathcal{I}_m^n) \backsimeq \frac{m-1}{2}\log{\frac{n}{m}}+\frac{m}{2}\log{e}+\frac{m\log{e}}{3}\sqrt{\frac{m}{n}}  
              \end{equation}
          \end{minipage}}
\item For $n=o(m)$: \hfill \vspace{5pt}  \makebox[0pt][r]{%
            \begin{minipage}[b]{\textwidth}
              \begin{equation}
		\label{n=o(m)}
                  \; \quad  \quad \quad  \hat{R}(\mathcal{I}_m^n)  \backsimeq n\log{\frac{m}{n}}+\frac{3}{2}\frac{n^2}{m}\log{e}-\frac{3}{2}\frac{n}{m}\log{e}
              \end{equation}
          \end{minipage}}
\item For $m=\alpha n+l(n)$: \hfill \makebox[0pt][r]{%
            \begin{minipage}[b]{\textwidth}
              \begin{equation}
		  \label{m=theta(n)}
               \; \; \;  \quad \quad \quad
 \hat{R}(\mathcal{I}_m^n)  \backsimeq n\log{B_\alpha}+l(n)\log{C_\alpha}-\log{\sqrt{A_\alpha}} 
              \end{equation}

          \end{minipage}}
\end{itemize}
\hspace{26pt} where $\alpha$ is a positive constant, $l(n)=o(n)$ and
\begin{equation}
{\displaystyle C_\alpha \triangleq \frac{1}{2}+\frac{1}{2}\sqrt{1+\frac{4}{\alpha}}\quad , \quad A_\alpha \triangleq C_\alpha+\frac{2}{\alpha} \quad , \quad  B_\alpha \triangleq \alpha C_\alpha^{\alpha+2} e^{-\frac{1}{C_\alpha}}}. \nonumber 
\end{equation}
\\
\noindent A very common method for dealing with unknown, or very large alphabet, sources is through adaptation \citep{cleary1984data}. Adaptive entropy coding reduces this overhead redundancy by finding a better trade-off between the average codeword length and the cost of transmitting the empirical distribution. Specifically, the samples are sequentially processed so that each sample is encoded according to the empirical distribution of its preceding samples (with some bias towards symbols which are yet to appear). As the samples are transmitted, both the encoder and the decoder gradually adapt their models so that the empirical distribution is less effected by a single sample and the average code length approaches the samples' entropy. \\

\noindent In is paper from 2004, \cite{orlitsky2004universal} presented a novel framework for universal compression of memoryless sources over unknown and possibly infinite alphabets.  
According to their framework, the description of any string, over any alphabet, can be viewed as consisting of two parts: the symbols appearing in the
string and the pattern that they form. For example, the string ``abracadabra" can be described by conveying the \textit{pattern} ``12314151231"
and the \textit{dictionary}

\begin{table}[!ht]
\centering
{
\begin{tabular}{|c|c|c|c|c|c|}
\hline
index &1 &2 &3 &4&5 \\
\hline
letter &a &b &r &c&d \\
\hline
\end{tabular}}

\end{table} 
\noindent Together, the pattern and dictionary specify that the string ``abracadabra" consists of the first letter to appear (a), followed by the second letter to appear (b), then by the third to appear (r), the first that appeared (a again), the fourth (c), etc.
Therefore, a compressed string involves with a compression of the pattern and its corresponding dictionary. Orlitsky et al. derived the bounds for pattern compression, showing that the redundancy of patterns compression under  i.i.d. distributions over potentially infinite alphabets is bounded by $\left( \frac{3}{2}\log{e} \right)n^{1/3}$. Therefore, assuming the alphabet size is $m$ and the number of uniquely observed symbols is $n_0$, the dictionary can be described in $n_0\log{m}$ bits, leading to an overall lower bound of $n_0\log{m}+n^{1/3}$ bits on the compression redundancy. \\

\noindent An additional (and very common) universal compression scheme is the canonical Huffman coding \citep{witten1999managing}.  A canonical Huffman code is a particular type of Huffman code with unique properties which allow it to be described in a very compact manner. The advantage of a canonical Huffman tree is that one can encode a codebook in fewer bits than a fully described tree. Since a canonical Huffman codebook can be stored especially efficiently, most compressors start by generating a non-canonical Huffman codebook, and then convert it to a canonical form before using it. In canonical Huffman coding the bit lengths of each symbol are the same as in the traditional Huffman code. However, each code word is replaced with new code words (of the same length), such that a subsequent symbol is assigned the next binary number in sequence. 
For example, assume a Huffman code for four symbols, A to D:
\begin{table}[!ht]
\centering
{
\begin{tabular}{|c|c|c|c|c|}
\hline
symbol &A &B&C&D \\
\hline
codeword &11 &0 &101 &100 \\
\hline
\end{tabular}}
\end{table} 
\noindent Applying canonical Huffman coding to it we have
\begin{table}[!ht]
\centering
{
\begin{tabular}{|c|c|c|c|c|}
\hline
symbol &B &A&C&D \\
\hline
codeword &0 &10 &110 &111 \\
\hline
\end{tabular}}
\end{table} 
\noindent This way we do not need to store the entire Huffman mapping but only a list of all symbols in increasing order by their bit-lengths and record the number of symbols for each bit-length. This allows a more compact representation of the code, hence, lower redundancy.\\

\noindent An additional class of data encoding methods which is referred to in this chapter is lossy compression. In the lossy compression setup one applies inexact approximations for representing the content that has been encoded. In this chapter we focus on vector quantization, in which a high-dimensional vector $\underline{X} \in \mathbb{R}^d$ is to be represented by a finite number of points. Vector quantization works by clustering the observed samples of the vector $\underline{X}$ into groups, where each group is represented by its centroid point, such as in $k$-means and other clustering algorithms. Then, the centroid points that represent the observed samples are compressed in a lossless manner. In the lossy compression setup, one is usually interested in minimizing the amount of bits which represent the data for a given a distortion measure (or equivalently, minimizing the distortion for a given compressed data size).  The rate-distortion function defines the lower bound on this objective. It is defined as  
\begin{equation}
R\left(D\right)=\min_{P(\underline{Y}|\underline{X})}I(\underline{X};\underline{Y})\,\, s.t. \,\, \mathbb{E} \left\{D(\underline{X},\underline{Y}) \right\} \leq D
\end{equation}
\noindent where $\underline{X}$ is the source, $\underline{Y}$ is recovered version of $\underline{X}$ and $D(\underline{X},\underline{Y})$ is some distortion measure between $\underline{X}$ and $\underline{Y}$. Notice that since the quantization is a deterministic mapping between $\underline{X}$ and $\underline{Y}$, we have that $I(\underline{X};\underline{Y})=H(\underline{Y})$, i.e., the entropy of the ``codebook".\\
 
\noindent The Entropy Constrained Vector Quantization (ECVQ) is an iterative method for clustering the observed samples from $\underline{X}$ into centroid points which are later represented by a minimal average codeword length. The ECVQ algorithm minimizes the Lagrangian   
\begin{equation}
\label{ECVQ_intro}
L =\mathbb{E}\left\{D(\underline{X},\underline{Y})\right\}+\lambda\mathbb{E}\left\{l(\underline{X})\right\}
\end{equation}
where $\lambda$ is the Lagrange multiplier and $\mathbb{E}\left(l(\underline{X})\right)$ is the average codeword length for each symbol in $\underline{X}$. The ECVQ algorithm performs an iterative local minimization method similar to the generalized  \cite{lloyd1982least} algorithm. This means that for a given clustering of samples it constructs an entropy code to minimize the average codeword lengths of the centroids. Then, for a given coding of centroids it clusters the observed samples such that the average distortion is minimized, biased by the length of the codeword. This process continues until a local convergence occurs. The ECVQ algorithm performs local optimization (as a variant of the $k$-means algorithm) which is also not very scalable for an increasing number of samples. This means that in the presence of a large number of samples, or when the alphabet size of the samples is large enough, the clustering phase of the ECVQ becomes impractical. Therefore, in these cases, one usually uses a predefined lattice quantizer and only constructs a corresponding codebook for its centroids.\\

\noindent It is quite evident that large alphabet sources entails a variety of difficulties in all the compression setups mentioned above: it is more complicated to construct an entropy code for, it results in a great redundancy when universally encoded and it is much more challenging to design a vector quantizer for. In the following sections we introduce a framework which is intended to address these drawbacks.

\section{Large Alphabet Source Coding}
\label{classic source coding}
Assume a classic compression setup in which both the encoder and the decoder are familiar with the joint probability distribution of the source $\underline{X} \sim \underline{p}$, and the number of observations $n$ is sufficiently large in the sense that $\hat{H}(\underline{X})\approx H(\underline{X})$. As discussed above, both Huffman and arithmetic coding entail a quite involved implementation as the alphabet size increases. In addition, the Huffman code guarantees a redundancy of at most a single bit for every alphabet size, depending on the (non-)dyadic structure of $p$. On the other hand, arithmetic coding does not require a dyadic $p$, but only guarantees a redundancy of up to two bits, and is practically limited for smaller alphabet size  \citep{cover2012elements,yang2000universal}. In other words, both Huffman and arithmetic coding may result in an average codeword length which is a bit or two greater than $H(\underline{X})$. Notice that these extra bits per symbol may be a substantial redundancy, as this extra code-length is compared with the entropy of the source (which is at most $\log(m)$).\\

\noindent To overcome these drawbacks, we suggest a simple solution in which we first apply an invertible transformation to make the components of $\underline{X}$ ``as statistically independent as possible", following entropy encoding of each of its components separately. This scheme results in a redundancy which we previously defined as $C(\underline{p},g)=\sum_{j=1}^m H(Y_j) -H(\underline{X})$. However, it allows us to apply a Huffman or arithmetic encoding on each of the components separately; hence, over a binary alphabet. Moreover, notice we can group several components, $Y_j$, into blocks so that the joint entropy of the block is necessarily lower than the sum of marginal entropies. Notice that in this chapter we refer to blocks as a set of components (as opposed to a set of words, as in Section \ref{Block-wise Order Permutation}). Specifically, denote $b$ as the number of components in each block and $B$ as the number of blocks. Then, $b \times B =d$ and for each block $v =1, \dots, B$ we have that
\begin{equation}
\label{block inequality}
H(\underline{Y}^{(v)}) \leq \sum_{u=1}^b H(Y_u^{(v)})
\end{equation}
where $H(\underline{Y}^{(v)})$ is the entropy of the block $v$ and $H(Y_u^{(v)})$ is the marginal binary entropy of the $u^{th}$ component of the block $v$. Summing over all $B$ blocks we have 
\begin{equation}
\label{total inequality}
\sum_{v=1}^B H(\underline{Y}^{(v)}) \leq \sum_{v=1}^B\sum_{u=1}^b H_b(Y_u^{(v)})=\sum_{j=1}^d H(Y_j).
\end{equation}
This means we can always apply our suggested invertible transformation which minimizes 
$\sum_{j=1}^d H(Y_j)$,  and then the group components into $B$ blocks and encode each block separately. This results in  $\sum_{v=1}^B H(\underline{Y}^{(v)}) \leq \sum_{j=1}^d H(Y_j)$. By doing so, we increase the alphabet size of each block (to a point which is still not problematic to implement with Huffman or arithmetic coding) while at the same time we decrease the redundancy. We discuss different considerations in choosing the number of blocks $B$ in the following sections.\\

\noindent A more direct approach of minimizing the sum of block entropies $\sum_{v=1}^B H(\underline{Y}^{(v)})$ is to refer to each block as a symbol over a greater alphabet size, $2^b$. This allows us to seek an invertible transformation which minimizes the sum of marginal entropies, where each marginal entropy corresponds to a marginal probability distribution over an alphabet size $2^b$. This minimization problem is discussed in detail in Section \ref{GICA over FF}. However,  notice that both the Piece-wise Linear Relaxation algorithm (Section \ref{The Relaxed Generalized BICA}), and the solutions discussed in Section \ref{GICA over FF}, require an extensive computational effort in finding a minimizer for (\ref{eq:sum_ent_min_binary}) as the alphabet size increases. Therefore, we suggest applying the greedy order permutation as $m$ grows. This solution may result in quite a large redundancy for a several joint probability distributions $\underline{p}$ (as shown in Section \ref{worst case}). However, as we uniformly average over all possible $p$'s, the redundancy is bounded with a small constant as the alphabet size increases (Section  \ref{average case}). Moreover, the order permutation simply requires ordering the values of  $\underline{p}$, which is significantly faster than constructing a Huffman dictionary or arithmetic encoder.\\

\noindent To illustrate our suggested scheme, consider a source $\underline{X} \sim \underline{p}$ over an alphabet size $m$, which follows the Zipf's law distribution, 
\begin{equation}\nonumber
P(k;s,m)=\frac{k^{-s}}{\sum_{l=1}^m l^{-s}}
\end{equation}
where $m$ is the alphabet size and $s$ is the skewness parameter. The Zipf's law distribution is a commonly used heavy-tailed distribution, mostly in modeling of natural (real-world) quantities. It is widely used in physical and social sciences, linguistics, economics and many other fields. \\  

\noindent We would like to design an entropy code for $\underline{X}$ with $m=2^{16}$ and different values of $s$. We first apply a standard Huffman code as an example of a common entropy coding scheme. We further apply our suggested order permutation scheme (Chapter \ref{Order_Permutation}), in which we sort $\underline{p}$ in a descending order, followed by arithmetic encoding to each of the components separately.  We further group these components into two separate blocks (as discussed above) and apply an arithmetic encoder on each of the blocks.   
We repeat this experiment for a range of parameter values $s$. Figure \ref{fig:classic_compression} demonstrates the results we achieve.

\begin{figure}[ht]
\centering
\includegraphics[width = 0.8\textwidth,bb= 60 235 540 545,clip]{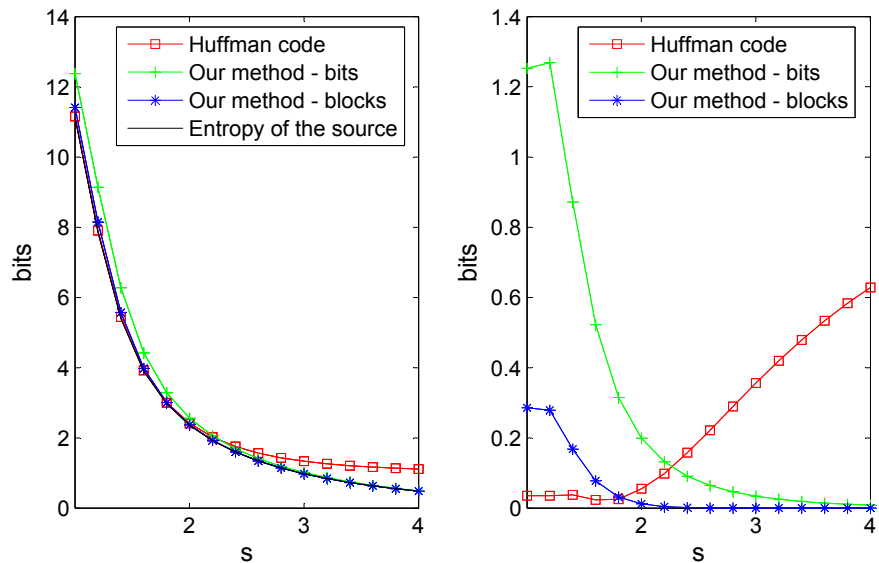}
\caption{Zipf's law simulation results. Left: the curve with the squares is the average codeword length using a Huffman code, the curve with the crosses corresponds to the average codeword length using our suggested methods when encoding each component separately, and the curve with the asterisks is our suggested method when encoding each of the two blocks separately. The black curve (which tightly lower-bounds all the curves) is the entropy of the source. Right: The difference between each encoding method and the entropy of the source}
\label{fig:classic_compression}
\end{figure}

\noindent Our results show that the Huffman code attains an average codeword length which is very close to the entropy of the source for lower values of $s$. However, as $s$ increases and the distribution of the source becomes more skewed, the Huffman code diverges from the entropy. On the other hand, our suggested method succeeds in attaining an average codeword length which is very close to the entropy of the source for every $s$, especially as $s$ increases, and when independently encoding each of the blocks.

\section{Universal Source Coding}
\label{universal source coding}
The classical source coding problem is typically concerned with a source whose alphabet size is much smaller than the length of the sequence. In this case one usually assumes that $\hat{H}(\underline{X}) \approx H(\underline{X})$.   
However, in many real world applications such an assumption is not valid. A paradigmatic example is the word-wise compression of natural language texts.
In this setup we draw a memoryless sequence of words, so that the alphabet size is often comparable to or even larger than the length of the source sequence. On a practical note, it is important to mention that a sequence of words can hardly be considered memoryless in most natural languages. However, the memoryless assumption becomes more valid as we increase the alphabet that we consider (for example, assuming that each symbol corresponds to an $n$-tuple of words or even a complete sentence). Therefore, for the simplicity of the presentation, we simply use words to illustrate our suggested technique. \\

\noindent As discussed above, the main challenge in large alphabet source coding is the redundancy of the code, which is formally defined as the excess number of bits used over the source's entropy. The redundancy may be quantified as the expected number of extra bits required to code a memoryless sequence drawn from $\underline{X} \sim \underline{p}$, when using a code that was constructed for $\underline{p}$, rather than using the ``true" code, optimized for the empirical distribution $\hat{\underline{p}}$.  
Another way to quantify these extra bits is to directly design a code for $\hat{\underline{p}}$, and transmit the encoded sequence together with this code. \\

\noindent Here again, we claim that in some cases, applying a transformation which decomposes the observed sequence into multiple ``as independent as possible" components results in a better compression rate. However, notice that now we also need to consider the number of bits required to describe the transformation. In other words, our redundancy involves not only with the cost described in (\ref{eq:min_criterion}), and the designated code for the observed sequence, but also with the cost of describing the invertible transformation to the receiver. This means that even the simple order permutation (Section \ref{Order_Permutation}) requires at most $n\log{m}$ bits to describe, where $m$ is the alphabet size and $n$ is the length of the sequence. This redundancy alone is not competitive with \cite{szpankowski2012minimax} worst-case redundancy results, described in (\ref{n=o(m)}).Therefore, we require a different approach which minimizes the sum of marginal entropies (\ref{eq:sum_ent_min_binary}) but at the same time is simpler to describe.\\

\noindent One possible solution is to seek for invertible, yet linear, transformations. This means that describing the transformation would now only require $\log^2{m}$ bits. However, the generalized linear BICA problem is also quite involved and preforms poorly as the dimension increases (see Section \ref{BICA_Vs_Linear}) .Therefore, we would like to modify our suggested combinatorial approach (Section \ref{Generalized_BICA}) so that the transformation we achieve requires fewer bits to describe.\\

\noindent As in the previous section, we argue that in some setups it is better to split the components of the data into blocks, with $b$ components in each block, and encode the blocks separately. Notice that we may set the value of $b$ so that the blocks are no longer considered as over a large alphabet size $(n \gg 2^b)$. This way, the redundancy of encoding each block separately is again negligible, at the cost of longer averaged codeword length. For simplicity of notation we define the number of blocks as $B$, and assume $B=\sfrac{d}{b}$ is a natural number. Therefore, encoding the $d$ components all together takes  $n\cdot \hat{H}(\underline{X})$ bits for the data itself, plus a redundancy term according to (\ref{m=o(n)}) and (\ref{n=o(m)}), while the block-wise compression takes about
\begin{equation}
\label{compression_blocks}
n\cdot \sum_{v=1}^{B}{ \hat{H}({\underline{X}}^{(v)})}+B\frac{2^b-1}{2}\log{\frac{n}{2^b}}
\end{equation}
bits, where the first term is $n$ times the sum of $B$ empirical block entropies and the second term is $B$ times the redundancy of each block when $n=o(2^b)$.
Two subsequent questions arise from this setup:
\begin{enumerate}
\item	What is the optimal value of $b$ that minimizes (\ref{compression_blocks})?
\item	Given a fixed value of $b$, how can we re-arrange $d$ components into $B$ blocks so that the averaged codeword length (which is bounded from below by the empirical entropy), together with the redundancy, is as small as possible?
\end{enumerate}
Let us start by fixing $b$ and focusing on the second question. \\

\noindent A naive shuffling approach is to exhaustively search for all possible combinations of clustering $d$ components into $B$ blocks. Assuming $d$ is quite large, an exhaustive search is practically infeasible. Moreover, the shuffling search space is quite limited and results in a very large value of (\ref{eq:min_criterion}), as shown below. Therefore, a different method is required.
We suggest applying our generalized BICA tool as an upper-bound search method for efficiently searching for a minimal possible averaged codeword length.
As in previous sections we define $\underline{Y}=g(\underline{X})$, where $g$ is some invertible transformation of $\underline{X}$.
Every block of the vector $\underline{Y}$ satisfies (\ref{block inequality}), where the entropy terms are now replaced with empirical entropies.  In the same manner as in Section \ref{classic source coding}, summing over all $B$ blocks results in  (\ref{total inequality}) where again, the entropy terms are replaced with empirical entropies.
This means that the sum of the empirical block entropies is  bounded from above by the empirical marginal entropies of the components of $\underline{Y}$ (with equality iff the components are independently distributed).
\begin{equation}
\label{universal total inequality}
\sum_{v=1}^B \hat{H}(\underline{Y}^{(v)}) \leq \sum_{j=1}^d \hat{H}(Y_j).
\end{equation}

\noindent Our suggested scheme \citep{painsky2015Universal} works as follows: 
We first randomly partition the $d$ components into $B$ blocks. We estimate the joint probability of each block and apply the combinatorial generalized BICA (Section \ref{The Relaxed Generalized BICA}) on it. The sum of empirical marginal entropies (of each block) is an upper bound on the empirical entropy of each block, as described in the previous paragraph. Now, let us randomly shuffle the $d$ components of the vector $\underline{Y}$. By ``shuffle" we refer to an exchange of positions of the components of $\underline{Y}$. Notice that by doing so, the sum of empirical marginal entropies of the entire vector $\sum_{i=1}^{d}{\hat{H}(Y_i)}$ is maintained. We now apply the generalized BICA on each of the (new) blocks. This way we minimize (or at least do not increase) the sum of empirical marginal entropies of the (new) blocks.  This obviously results in a lower sum of empirical marginal entropies of the entire vector $\underline{Y}$. It also means that we minimize the left hand side of (\ref{universal total inequality}), which upper bounds the sum of empirical block entropies, as the inequality in (\ref{universal total inequality}) suggests. In other words, we show that in each iteration we decrease (at least do not increase) an upper bound on our objective. We terminate once a maximal number of iterations is reached or we can no longer decrease the sum of empirical marginal entropies.Therefore, assuming we terminate at iteration $I_0$, encoding the data takes about 
\begin{align}
\label{compression_total}
   n\cdot \sum_{v=1}^{B}{\hat{H}^{[I_0]}({\underline{Y}}^{(v)})}+B\frac{2^b-1}{2}\log{\frac{n}{2^b}}+ I_0B\cdot &b2^b+I_0d\log{d}
\end{align}
bits, where the first term refers to the sum of empirical block entropies at the $I_0$ iteration, the third term refers to the representation of $I_0 \cdot B$ invertible transformation of each block during the process until $I_0$, and the fourth term refers to the bit permutations at the beginning of each iteration. Hence, to minimize (\ref{compression_total}) we need to find the optimal trade-off between a low value of $\sum_{v=1}^{B}{\hat{H}^{[I_0]}({\underline{Y}}^{(v})}$ and a low iteration number $I_0$. 
We may apply this technique with different values of $b$ to find the best compression scheme over all block sizes.
 
\subsection{Synthetic Experiments}

In order to demonstrate our suggested method we first generate a dataset according to the Zipf law distribution which was previously described.  We draw $n=10^6$ realizations from this distribution with an alphabet size $m=2^{20}$ and a parameter value $s=1.2$. We encounter $n_0=80,071$ unique words and attain an empirical entropy of $8.38$ bits (while the true entropy is $8.65$ bits). Therefore, compressing the drawn realizations in its given $2^{20}$ alphabet size takes a total of about $10^{6}\times 8.38+1.22\times10^6=9.6 \cdot 10^6$ bits, according to (\ref{m=theta(n)}). 
Using the patterns method \cite{orlitsky2004universal}, the redundancy we achieve is the redundancy of the pattern plus the size of the dictionary. Hence, the compressed size of the data set according to this method is lower bounded by $10^{6}\times8.38+80,071\times20+100=9.982\cdot10^6$ bits.
In addition to these asymptotic schemes we would also like to compare our method with a common practical approach. For this purpose we apply the canonical version of the Huffman code. Through the canonical Huffman code we are able to achieve a compression rate of $9.17$ bits per symbol, leading to a total compression size of about $1.21\cdot10^7$ bits.   \\

\noindent Let us now apply a block-wise compression. 
We first demonstrate the behavior of our suggested approach with four blocks $(B=4)$ as appears in Figure \ref{fig:zipf}.
To have a good starting point, we initiate our algorithm with a the naive shuffling search method (described above). This way we apply our optimization process on the best representation a random bit shuffling could attain (with a negligible $d\log{d}$ redundancy cost).  As we can see in Figure \ref{fig:zipf}.B, we  minimize (\ref{compression_total}) over $I_0=64$ and $\sum_{v=1}^{B}{\hat{H}({\underline{Y}}^{(v)})}=9.09$  to achieve a total of $9.144 \cdot 10^6$ bits for the entire dataset.\\

\noindent Table \ref{table:zipf_results} summarizes the results we achieve for different block sizes $B$. We see that the lowest compression size is achieved over $B=2$, i.e. two blocks. The reason is that for a fixed $n$, the redundancy is approximately exponential in the size of the block $b$. This means the redundancy drops exponentially with the number of blocks while the minimum of $\sum_{v=1}^{B}{\hat{H}({\underline{Y}}^{(v)})}$ keeps increasing. In other words, in this example we earn a great redundancy reduction when moving to a two-block representation while not losing too much in terms of the average code-word length we can achieve. We further notice that the optimal iterations number grows with the number of blocks. This results from the cost of describing the optimal transformation for each block, at each iteration, $I_0B\cdot b2^b$, which exponentially increase with the block size $b$. Comparing our results with the three methods described above we are able to reduce the total compression size in $8\cdot10^5$ bits, compared to the minimum among all our competitors. 
  
\begin{figure}[!ht]
\centering
\includegraphics[width = 0.8\textwidth,bb= 35 123 775 490,clip]{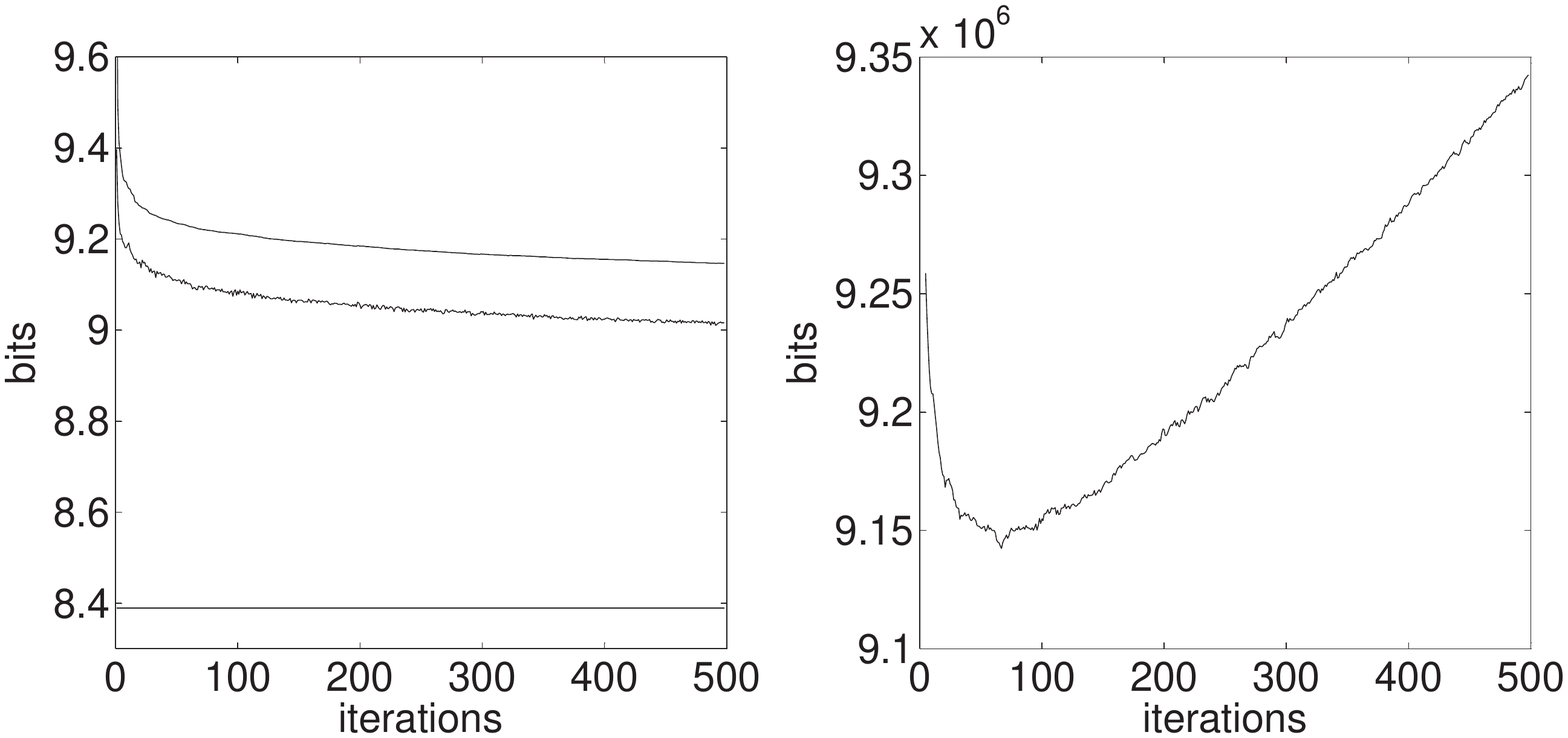}
\caption{Large Alphabet Source Coding via Generalized BICA with $B=4$ blocks. Left side (A): the horizontal line indicated the empirical entropy of $\underline{X}$. The upper curve is the sum of marginal empirical entropies and the lower curve is the sum of empirical block entropies (the outcome of our suggested framework). Right side (B): total compression size of our suggested method at each iteration.}
\label{fig:zipf}
\end{figure}

\begin{table}[!ht]
\caption{Block-Wise Compression via Generalized BICA Method for different block sizes}
\renewcommand{\baselinestretch}{1}\footnotesize
\label{table:zipf_results} 
\centering
\begin{tabular}{|M{1.5cm}|M{2.5cm}|M{1.5cm}|M{2cm}|M{1.8cm}|M{2.5cm}|N}  

\hline

\begin{tabular}{@{}c@{}}Number of \\ Blocks\end{tabular}
&\begin{tabular}{@{}c@{}} Minimum of \\ $\sum_{v=1}^{B}{\hat{H}({\underline{Y}}^{(v)})}$\end{tabular}
&Optimal $I_0$
&\begin{tabular}{@{}c@{}} Compressed \\  Data Size\end{tabular}
&Redundancy
&\begin{tabular}{@{}c@{}} Total Compression\\ Size\end{tabular}&\\[15pt]

\hline
$2$   &$8.69$ & $5$ & $8.69\cdot 10^6$ & $1.15\cdot 10^5$  & $\bold{8.805\cdot10^6}$ & \\[10pt]  
\hline
$3$   & $8.93$ & $19$ & $8.93\cdot 10^6$ & $5.55\cdot 10^4$  & $8.985\cdot10^6$ & \\[10pt]  
\hline
$4$   & $9.09$ & $64$ & $9.09\cdot 10^6$ & $5.41\cdot 10^4$  & $9.144\cdot10^6$ & \\[10pt] 
\hline
\end{tabular}
\end{table}

\subsection{Real-world Experiments}

We now turn to demonstrate our compression framework on real world data sets.
For this purpose we use collections of word frequencies of different natural languages. These word frequency lists are publicly available\footnote{\url{http://en.wiktionary.org/wiki/Wiktionary:Frequency_lists}} and describe the frequency each word appears in a language, based on hundreds of millions of words, collected from open source subtitles\footnote{\url{www.opensubtitles.org}} or based on different dictionaries and glossaries \citep{new2004lexique}.
Since each word frequency list holds several hundreds of thousands of different words, we choose a binary $d=20$ bit representation. We sample $10^7$ words from each language and examine our suggested framework, compared with the compression schemes mentioned above. The results we achieve are summarized in Table \ref{table:languages}. Notice the last column provides the percentage of the redundancy we save, which is essentially the most we can hope for (as we cannot go lower than $n\cdot \hat{H}(\underline{X})$ bits). As in the previous experiment, our suggested algorithm achieves the lowest compression size applied with two blocks after approximately $I_0=10$ iterations, from the same reasons mentioned above. 
Compared to the other methods, our suggested framework shows to achieve significantly lower compression sizes for all languages, saving an average of over one million bits per language.

\begin{table}[ht]
\caption{Natural Languages Experiment. For each compression method (D), (O) and (T) stand for the compressed data, the overhead and the total compression size (in bits) respectively. The We Save column is the amount of bits saved by our method, and its corresponding percentage of (O) and (T). $n_0$ is the number of unique words observed in each language, of the $10^7$ sampled words. Notice the Chinese corpus refers to characters.\newline}

\renewcommand{\baselinestretch}{1}\footnotesize
\label{table:languages} 
\centering
\begin{tabular}{|M{1.3cm}|M{2.13cm}|M{2.13cm}|M{2.13cm}|M{2.13cm}|M{1.8cm}|N}  

\hline

\begin{tabular}{@{}c@{}} Language \\ $(n_0)$ \end{tabular}
&\begin{tabular}{@{}c@{}} Standard \\ Compression \end{tabular}
&\begin{tabular}{@{}c@{}} Patterns \\ Compression \end{tabular}
&\begin{tabular}{@{}c@{}} Canonical \\  Huffman \end{tabular}
&\begin{tabular}{@{}c@{}} Our Suggested \\  Method \end{tabular}
&We Save&\\[20pt]

\hline

\begin{tabular}{@{}c@{}} English \\ $(129,834)$ \end{tabular}   
&\begin{tabular}{@{}c@{}} (D) $9.709\cdot10^{7}$ \\ (O) $2.624 \cdot10^{6}$\\ (T)  $9.971\cdot10^{7}$ \end{tabular} 
&\begin{tabular}{@{}c@{}} (D) $9.709\cdot10^{7}$ \\ (O) $2.597 \cdot10^{6}$ \\ (T)  $9.968\cdot10^{7}$ \end{tabular} 
&\begin{tabular}{@{}c@{}} (D) $9.737\cdot10^{7}$ \\ (O)  $5.294 \cdot10^{6}$ \\ (T)  $1.027 \cdot10^{8}$ \end{tabular} 
&\begin{tabular}{@{}c@{}} (D) $9.820 \cdot10^{7}$ \\ (O) $2.207 \cdot10^{5}$ \\ (T) $\bold{9.842 \cdot10^{7}}$ \end{tabular}  
&\begin{tabular}{@{}c@{}} $1.262 \cdot10^{6}$ \\  (O) $48.6\%$  \\ (T) $1.27\%$   \end{tabular}    & \\[32pt]

\hline

\begin{tabular}{@{}c@{}} Chinese \\ $(87,777)$ \end{tabular}     
&\begin{tabular}{@{}c@{}} (D)  $1.020\cdot10^{8}$\\ (O)  $2.624\cdot10^{6}$ \\ (T)   $1.046\cdot10^{8}$ \end{tabular} 
&\begin{tabular}{@{}c@{}} (D)  $1.020\cdot10^{8}$ \\ (O) $1.696\cdot10^{6}$ \\ (T)  $1.037\cdot10^{8}$ \end{tabular} 
&\begin{tabular}{@{}c@{}} (D) $1.023\cdot10^{8}$ \\ (O) $3.428\cdot10^{6}$ \\ (T) $1.057\cdot10^{8}$ \end{tabular} 
&\begin{tabular}{@{}c@{}} (D) $1.028\cdot10^{8}$ \\ (O) $2.001\cdot10^{5}$ \\ (T)  $\bold{1.030\cdot10^{8}}$ \end{tabular}   
&\begin{tabular}{@{}c@{}} $6.566 \cdot10^{5}$ \\  (O) $38.7\%$ \\ (T) $0.63\%$\end{tabular}    & \\[32pt]   
\hline

\begin{tabular}{@{}c@{}} Spanish \\ $(185,866)$ \end{tabular}     
&\begin{tabular}{@{}c@{}} (D)  $1.053\cdot10^{8}$\\ (O)  $2.624\cdot10^{6}$ \\ (T)   $1.079\cdot10^{8}$ \end{tabular} 
&\begin{tabular}{@{}c@{}} (D)  $1.053\cdot10^{8}$ \\ (O) $3.718\cdot10^{6}$ \\ (T)  $1.090\cdot10^{8}$ \end{tabular} 
&\begin{tabular}{@{}c@{}} (D)  $1.055\cdot10^{8}$ \\ (O) $7.700\cdot10^{6}$ \\ (T)  $1.132\cdot10^{8}$  \end{tabular} 
&\begin{tabular}{@{}c@{}} (D) $1.067\cdot10^{8}$ \\ (O) $2.207\cdot10^{5}$ \\ (T)  $\bold{1.069\cdot10^{8}}$ \end{tabular}   
&\begin{tabular}{@{}c@{}} $9.631 \cdot10^{5}$ \\  (O) $36.7\%$ \\ (T) $0.89\%$\end{tabular}    & \\[32pt]   
\hline

\begin{tabular}{@{}c@{}} French \\ $(139,674)$ \end{tabular}     
&\begin{tabular}{@{}c@{}} (D)  $1.009\cdot10^{8}$\\ (O)  $2.624\cdot10^{6}$ \\ (T)   $1.035\cdot10^{8}$ \end{tabular} 
&\begin{tabular}{@{}c@{}} (D)  $1.009\cdot10^{8}$ \\ (O) $2.794\cdot10^{6}$ \\ (T)  $1.036\cdot10^{8}$ \end{tabular} 
&\begin{tabular}{@{}c@{}} (D)  $1.011\cdot10^{8}$ \\ (O)  $5.745\cdot10^{6}$  \\ (T)  $1.069\cdot10^{8}$ \end{tabular} 
&\begin{tabular}{@{}c@{}} (D) $1.017\cdot10^{8}$ \\ (O) $2.207\cdot10^{5}$ \\ (T)  $\bold{1.019\cdot10^{8}}$ \end{tabular}   
&\begin{tabular}{@{}c@{}} $1.557 \cdot10^{6}$ \\  (O) $59.3\%$\\ (T) $1.50\%$ \end{tabular}    & \\[32pt]   
\hline

\begin{tabular}{@{}c@{}} Hebrew \\ $(250,917)$ \end{tabular}     
&\begin{tabular}{@{}c@{}} (D)  $1.173\cdot10^{8}$\\ (O)  $2.624\cdot10^{6}$ \\ (T)   $1.200\cdot10^{8}$ \end{tabular} 
&\begin{tabular}{@{}c@{}} (D)  $1.173\cdot10^{8}$ \\ (O) $5.019\cdot10^{6}$ \\ (T)  $1.224\cdot10^{8}$ \end{tabular} 
&\begin{tabular}{@{}c@{}} (D)  $1.176\cdot10^{8}$ \\ (O)  $1.054\cdot10^{7}$  \\ (T)  $1.281\cdot10^{8}$ \end{tabular} 
&\begin{tabular}{@{}c@{}} (D) $1.190\cdot10^{8}$ \\ (O) $1.796\cdot10^{5}$ \\ (T)  $\bold{1.192\cdot10^{8}}$ \end{tabular}   
&\begin{tabular}{@{}c@{}} $7.837 \cdot10^{5}$ \\  (O) $29.9\%$ \\ (T) $0.65\%$  \end{tabular}    & \\[32pt]   
\hline

\end{tabular}
\end{table}
\vspace{1em}

\section{Adaptive Entropy Coding}

As mentioned in previous sections, the major bottleneck in our suggested scheme is describing the permutation we applied to the receiver. In this section we suggest three additional strategies to tackle this problem, as presented in \citep{painsky2016largetrans}.\\

\noindent One possible solution is to transmit the permutation, which results in an additional redundancy of $n_0d$, where $n_0$ is the number of unique symbols that were sampled. As this redundancy may be too costly, we may consider a fixed sub-optimal order permutation according to the source's distribution. Assume both the encoder and the decoder know that the samples are drawn from a family of heavy-tailed distribution (for example, Zipf law with unknown parameters). Then, we may use a fixed order permutation, based on the expected order among the appearance of the symbols, followed by adaptive entropy coding for each of the components \citep{cleary1984data}. For example, assume we are to encode independent draws from an English dictionary. We know that the word ``The" is more frequent than the word ``Dictionary", even without knowing their exact probability of occurrence. This way we may apply a fixed order permutation, based only on the order of the frequency of appearance of the symbols. We denote this method as \textit {fixed permutation marginal encoding}. Notice that if we are lucky enough to have the empirical distribution ordered in the same manner as the source's distribution, this fixed permutation is optimal (identical to an order permutation of  the empirical distribution). However, notice that the order permutation does not necessarily minimize (\ref{eq:sum_ent_min_binary}). Therefore, the joint entropy of several components may be lower than the sum of these components' marginal entropies. This means we may apply the order permutation, followed by separating the resulting components into two groups (blocks, as previously described) and apply adaptive entropy coding on each of these blocks. This method is denoted as  \textit {fixed permutation block encoding}.\\

\noindent An additional approach for conveying the order permutation to the decoder is based on adaptation. Here, the samples are sequentially processed so that each sample is transformed according to an order permutation, based on the empirical distribution of the preceding samples. This way, the decoder receives each encoded sample, applies an inverse transformation and updates both the empirical distribution and the required inverse transformation for the next sample. We refer to an adaptive order permutation, followed by marginal adaptive entropy coding, as \textit {adaptive permutation marginal encoding}. In addition we have \textit {adaptive permutation block encoding}, in the same manner as above.\\

\subsection{Experiments}  
To illustrate our suggested methods, consider a source $\underline{X} \sim \underline{p}$ over an alphabet size $m$, which follows the Zipf's law distribution, as described throughout this chapter. Figure \ref{fig:adaptive_coding_1} presents the results we achieve, applying our suggested methods to independent samples of a Zipf distribution with $s=1$ and different alphabet sizes $m=2^d$. We compare our methods with traditional adaptive entropy coding techniques. 

\begin{figure}[!ht]
\centering
\includegraphics[width = 1.25\textwidth,bb= 0 180 790 625,clip]{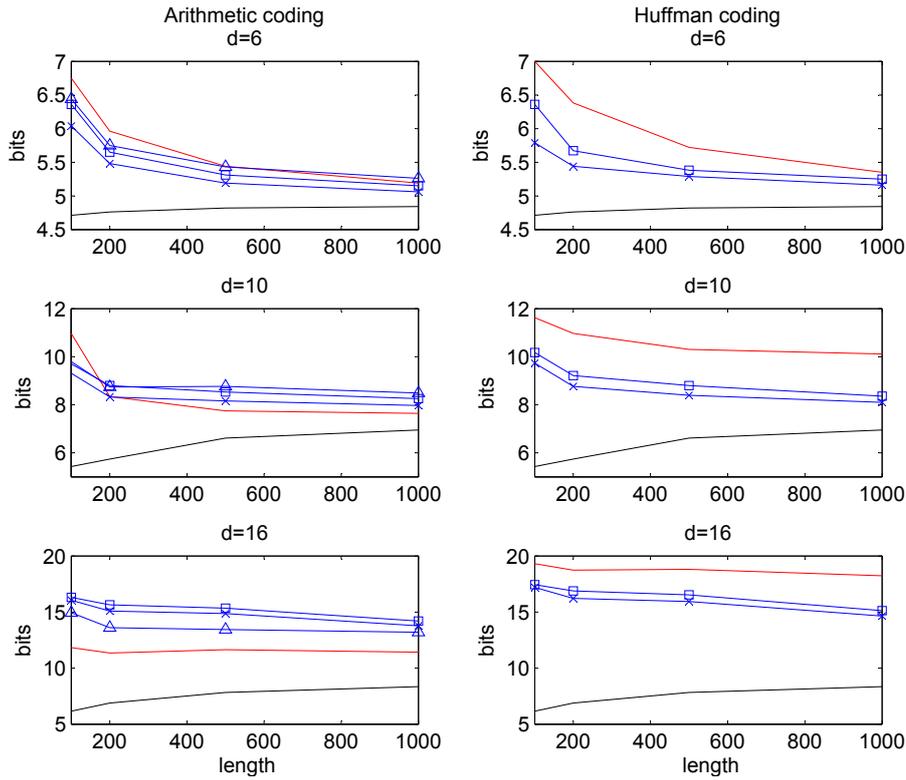}
\caption{Adaptive entropy coding of independent draws from a Zipf distribution with $s=1$. The charts on the left correspond to arithmetic coding with different alphabet sizes ($d=6,10$ and $16$) and different sequence lengths (horizontal axis of each chart). The charts on the right are Huffman coding. In each chart the black curve at the bottom is the empirical entropy, the red curve is the adaptive entropy coding, the curve with the triangles is \textit{adaptive permutation marginal encoding}, the curve with the squares is \textit{adaptive permutation block encoding} and the curve with the X's is  \textit{fixed permutation block encoding} }
\label{fig:adaptive_coding_1}
\end{figure}
\noindent As we can see, our suggested methods outperform the adaptive arithmetic coding for small sequence lengths of an alphabet size $d=6$ (upper chart of the left). As the length of the sequence increases, the alphabet is no longer considered ``large" (comparable or even smaller than the length of the sequence), and the gap between the schemes closes. As we increase the alphabet size (middle and bottom charts on the left) we notice our methods becomes less competitive. The reason is that the cost of transmitting the permutation (whether adaptively or as a fixed (and inaccurate) transformation) becomes too costly. As we examine our results with Huffman coding (charts on the right) we notice our suggested methods outperform the adaptive Huffman scheme for all alphabet sizes and sequence lengths. The reason is that the Huffman coding scheme performs quite poorly as the alphabet size increases. 
Comparing our three methods with each other we notice that the  \textit{fixed permutation block encoding} tends to perform better than the others. Notice that  the \textit{adaptive permutation marginal encoding} with a Huffman code results in a high code rate and it is therefore omitted from the charts. The reason is that the order permutation results in marginal probabilities which tend to have low entropies (degenerate component probabilities). This kind of components are specifically problematic for the Huffman coding scheme, as previously discussed above.    \\

\noindent Another aspect of our suggest approach is its low computational complexity, which results in a faster compression runtime. Figure \ref{fig:adaptive_coding_2} demonstrates  the total run-time of the adaptive arithmetic coding scheme, compared with the \textit{fixed permutation block encoding}, in the experiments above. Here we use a standard Matlab implementation of an adaptive arithmetic coder.  There exists a large body of work regarding more efficient and faster implementations (for example, \citep{fenwick1994new}). However, the emphasis here is to demonstrate that the order permutation is very simple and quick to apply, as it is simply a sorting algorithm followed by a small alphabet entropy coding.

\begin{figure}[!ht]
\centering
\includegraphics[width = 1.3\textwidth,bb= 50 290 730 505,clip]{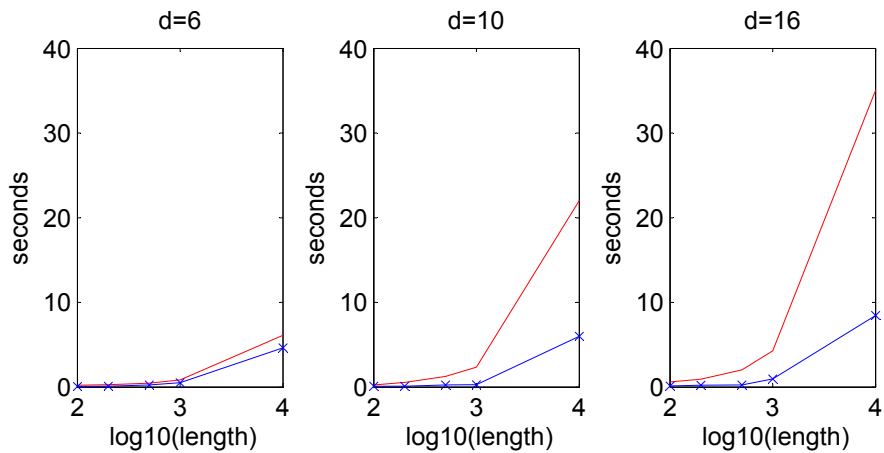}
\caption{Runtime of adaptive arithmetic coding and our suggested \textit{fixed permutation block encoding} scheme, in the experiments on the left charts in Figure \ref{fig:adaptive_coding_1}}
\label{fig:adaptive_coding_2}
\end{figure}
\noindent As mentioned above, the arithmetic coding scheme encodes an entire sequence of samples altogether. This results in a delay in decoding the samples, implementation synchronization issues  and a high sensitivity to errors in the coded message. To overcome these drawbacks we suggest a sliding window approach. This means that the encoder shall sequentially encode non-overlapping sub-sequences of length $l$ at each iteration, based of the empirical distribution of all preceding samples encoded in previous iterations. As before, we compare this scheme with adaptive order permutation based techniques. Here, at each iteration we apply an order permutation on a sub-sequence of length $l$, based on the empirical distribution of all preceding samples. Then we apply arithmetic encoding to each of the components/blocks of the transformed sub-sequences. For the purpose of this experiment we use independent draws from an English dictionary\footnote{\url{http://en.wiktionary.org/wiki/Wiktionary:Frequency_lists}}. Since the English dictionary holds almost a million different words, we choose a binary representation of $d=20$ bits. Figure \ref{fig:sequential_coding} summarizes the results we achieve for different sub-sequence lengths $l$. We first notice that the order-permutation based methods outperform the arithmetic coding as long as the alphabet size is considered large ($350 \times 10,000$ samples). This happens since the arithmetic coder exhibits a ``large alphabet" setup in each iteration ($d=20$, $l=100/1000/10,000$), even if it already learned the true distribution of the source. On the other hand, our \textit{adaptive permutation marginal encoding} method, for example, allows a ``small alphabet" compression ($d=2$) of each component at the small cost of $C(\underline{Y})$.

\begin{figure}[!ht]
\centering
\includegraphics[width = 0.95\textwidth,bb= 20 155 800 460,clip]{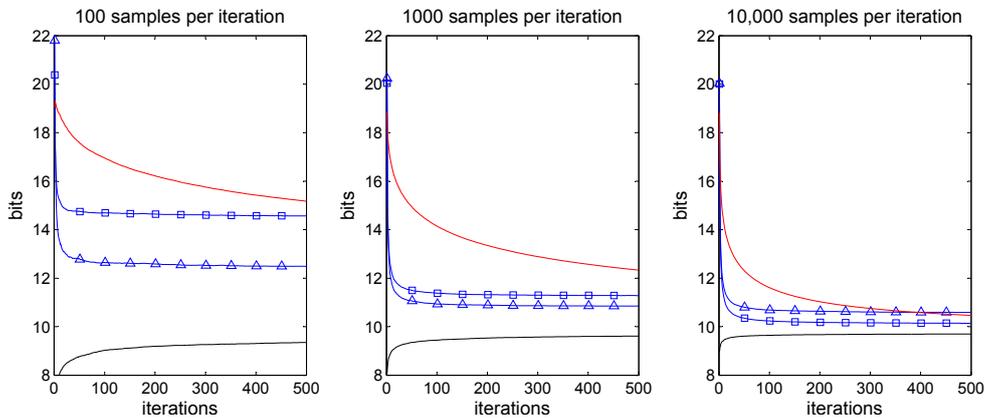}
\caption{Sliding window based adaptive arithmetic coding for window sizes $l=100/1000/10,000$.  In each chart the black curve on the bottom is the empirical entropy of the samples, the red curve is adaptive entropy coding, the curve with the triangles is \textit{adaptive permutation marginal encoding} and the curve with the squares is \textit{adaptive permutation block encoding}}
\label{fig:sequential_coding}
\end{figure}

\section{Vector Quantization}
Vector quantization refers to a lossy compression setup, in which a high-dimensional vector $\underline{X} \in \mathbb{R}^d$ is to be represented by a finite number of points. This means that the high dimensional observed samples are clustering into groups, where each group is represented by a representative point. For example, the famous $k$-means algorithm \citep{macqueen1967some} provides a method to determine the clusters and the representative points (centroids) for an Euclidean loss function. Then, these centroid points that represent the observed samples are compressed in a lossless manner. \\

\noindent As described above, in the lossy encoding setup one is usually interested in minimizing the amount of bits which represent the data for a given a distortion (or equivalently, minimizing the distortion for a given compressed data size).  The rate-distortion function defines the lower bound on this objective. In vector quantization, the representation is a deterministic mapping (defined as $P(\underline{Y}|\underline{X})$) from a source $\underline{X}$ to its quantized version $\underline{Y}$. Therefore we have that $H(\underline{Y}|\underline{X})=0$ and the rate distortion is simply   

\begin{equation}
\label{R(D)}
R\left(D\right)=\min_{P(\underline{Y}|\underline{X})}H(\underline{Y})\,\, s.t. \,\, \mathbb{E}\left\{D(\underline{X},\underline{Y})\right\} \leq D
\end{equation}
where $D(\underline{X},\underline{Y})$ is some distortion measure between $\underline{X}$ and $\underline{Y}$. 

\subsection{Entropy Constrained Vector Quantization}
The Entropy Constrained Vector Quantization (ECVQ) is an iterative method for clustering the observed samples into centroid points which are later represented by a minimal average codeword length. The ECVQ algorithm aims to find the minimizer of  
\begin{equation}
\label{ECVQ minimization}
J\left(D\right)=\min \mathbb{E}\left\{l(\underline{X})\right\} \,\, s.t. \,\, \mathbb{E}\left\{D(\underline{X},\underline{Y})\right\} \leq D
\end{equation}
where the minimization is over three terms: the vector quantizer (of $\underline{X}$), the entropy encoder (of the quantized version of $\underline{X}$) and the reconstruction module of $\underline{X}$ from its quantized version.\\
   
\noindent Let us use a similar notation to \cite{chou1989entropy}. Denote the vector quantizer $\alpha: \underline{x} \rightarrow \mathpzc{C}$ as a mapping from an observed sample to a cluster in $\mathpzc{C}$, where $\mathpzc{C}$ is a set of $m$ clusters.  Further, let $\gamma: \mathpzc{C} \rightarrow \mathpzc{c}$ be a mapping  from a cluster to a codeword. Therefore, the composition $\alpha \circ \gamma$ is the encoder. In the same manner, the decoder is a composition $\gamma^{-1} \circ \beta$, where $\gamma^{-1}$ is the inverse mapping from a codeword to a cluster and $\beta: \mathpzc{C} \rightarrow \underline{y}$ is the reconstruction of $\underline{x}$ from its quantized version. Therefore, the Lagrangian of the optimization problem (\ref{ECVQ minimization}) is 
\begin{equation}
\label{ECVQ}
L_{\lambda}(\alpha,\beta,\gamma) =\mathbb{E}\left\{D(\underline{X},\beta\left(\alpha\left(\underline{X}\right)\right)+\lambda\left|\gamma\left(\alpha\left(\underline{X}\right)\right)\right|\right\}
\end{equation}

\noindent The ECVQ objective is to find the coder $(\alpha,\beta,\gamma)$ which minimizes this functional. In their work, \cite{chou1989entropy} suggest an iterative descent algorithm similar to the
generalized \cite{lloyd1982least} algorithm. Their algorithm starts with an arbitrary initial coder. Then, for a fixed $\gamma$ and $\beta$ it finds a clustering $\alpha(\underline{X})$ as the minimizer of:

\begin{equation}
\label{ECVQ_a}
\alpha(\underline{X})=\argmin_{i \in \mathpzc{C}}\left\{D(\underline{X},\beta\left(i)\right)+\lambda\left|\gamma\left(i\right)\right| \right\}.
\end{equation}
Notice that for an Euclidean distortion, this problem is simply $k$-means clustering, with a ``bias" of $\lambda\left|\gamma\left(i\right)\right|$ on its objective function.\\

\noindent For a fixed  $\alpha$ and $\beta$, we notice that each cluster  $i \in \mathpzc{C}$ has an induced probability of occurrence $p_i$. Therefore,  the entropy encoder $\gamma$ is designed accordingly, so that $|\gamma(i)|$ is minimized. The Huffman algorithm could be incorporated into the design algorithm at this stage. However, for simplicity, we allow codewords to have non-integer lengths, and assign 

\begin{equation}
\label{ECVQ_b}
\left|\gamma\left(i\right)\right|=-\log(p_i).
\end{equation}
\noindent Finally, for a fixed $\alpha$ and $\gamma$, the reconstruction module $\beta$ is 
\begin{equation}
\label{ECVQ_c}
\beta(i)=\argmin_{\underline{y} \in \underline{Y}} \mathbb{E} \left\{ D\left(\underline{X},\underline{y}\right) | \alpha(\underline{X})=i \right\}.
\end{equation}
For example, for an euclidean distortion measure, $\beta(i)$'s are simply the centroids of the clusters $i \in \mathpzc{C}$.\\

\noindent Notice that the value objective (\ref{ECVQ}), when applying each of the three steps (\ref{ECVQ_a}-\ref{ECVQ_c}), is non-increasing. Therefore, as we apply these three steps repeatedly, the ECVQ algorithm is guarenteed to converge to a local minimum. 
Moreover, notice that for an Euclidean distortion measure, step (\ref{ECVQ_a}) of the ECVQ algorithm is a variant of the $k$-means algorithm. However, the $k$-means algorithms is known to be computationally difficult to execute as the number of observed samples increases. Hence, the ECVQ algorithm is also practically limited to a relatively small number of samples.\\

\noindent As in previous sections, we argue that when the alphabet size is large (corresponds to low distortion), it may be better to encode the source component-wise. This means, we would like to construct a vector quantizer such that the sum marginal entropies of $\underline{Y}$ is minimal, subject to the same distortion constraint as in (\ref{R(D)}). Specifically, 
\begin{equation}
\label{R(D)_ours}
\tilde{R}\left(D\right)=\min_{P(\underline{Y}|\underline{X})}\sum_{j=1}^d H(Y_j)\,\, s.t. \,\, \mathbb{E}\left\{D(\underline{X},\underline{Y})\right\} \leq D
\end{equation}
\noindent Notice that for a fixed distortion value, $R\left(D\right) \leq \tilde{R}\left(D\right)$ as sum of marginal entropies is bounded from below by the joint entropy. However, since encoding a source over a large alphabet may result in a large redundancy (as discussed in previous sections), the average codeword length of the ECVQ (\ref{ECVQ minimization}) is not necessarily lower than our suggested method (and usually even much larger). \\

\noindent Our suggested version of the ECVQ works as follows: we construct $\alpha$ and $\beta$ in the same manner as  ECVQ does, but replace the Huffman encoder (in $\gamma$) with our suggested linear relaxation to the BICA problem (Section \ref{Generalized_BICA}). This means that for a fixed $\alpha, \beta$, which induce a random vector over a finite alphabet size (with a finite probability distribution), we seek for a representation which makes its components ``as statistically independent as possible". The average codeword lengths are then achieved by arithmetic encoding on each of these components. \\

\noindent This scheme results not only with a different codebook, but also with a different quantizer than the ECVQ. This means that a quantizer which strives to construct a random vector (over a finite alphabet) with the lowest possible average codeword length (subject to a distortion constraint) is different than our quantizer, which seeks for a random vector with a minimal sum of marginal average codeword lengths (subject to the same distortion). \\

\noindent Our suggested scheme proves to converge to a local minimum in the same manner that ECVQ does. That is, for a fixed $\alpha, \beta$, our suggested relaxed BICA method finds a binary representation which minimizes the sum of marginal entropies. Therefore, we can always compare the representation it achieves in the current iteration with the representation it found in the previous iteration, and choose the one which minimizes the objective. This leads to a non-increasing objective each time it is applied. Moreover, notice that we do not have to use the complicated relaxed BICA scheme and apply the simpler order permutation (Section \ref{Order_Permutation}). This would only result in a possible worse encoder but local convergence is still guaranteed.\\

\noindent To illustrate the performance of our suggested method we conduct the following experiment:
We draw $1000$ independent samples from a six dimensional bivariate Gaussian mixture.  
We apply both the ECVQ algorithm, and our suggest BICA variation of the ECVQ, on these samples. Figure \ref{fig:ECVQ} demonstrates the average codeword length we achieve for different Euclidean (mean square error) distortion levels.

\begin{figure}[ht]
\centering
\includegraphics[width = 0.7\textwidth,bb= 0 190 600 590,clip]{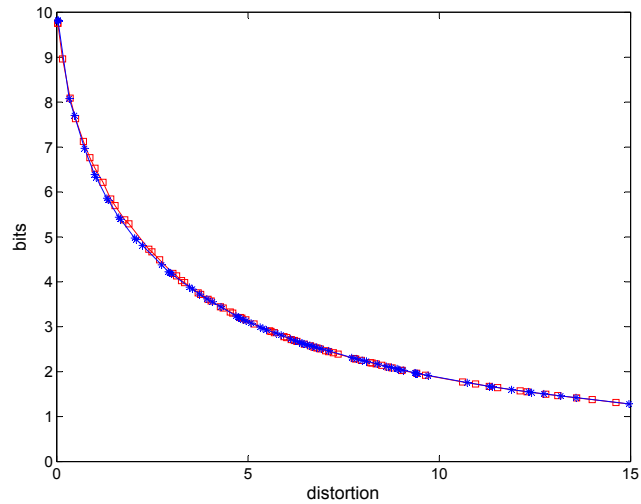}
\caption{ECVQ simulation. The curve with the squares corresponds to the average codeword length achieved by the classical ECVQ algorithm. The curve with the asterisks is the average codeword length achieved by our suggested BICA variant to the ECVQ algorithm} 
\label{fig:ECVQ}
\end{figure}
\noindent We first notice that both methods performs almost equally well. The reason is that $1000$ observations do not necessitate an alphabet size which is greater than $m=1000$ to a attain a zero distortion. In this ``small alphabet" regime, our suggested approach does not demonstrate its advantage over classical methods, as discussed in previous sections. However, we can still see it performs equally well. As we try to increase the number of observations (and henceforth the alphabet size) we encounter computational difficulties, which result from repeatedly performing a variant of the $k$-means algorithm (\ref{ECVQ_a}). This makes both ECVQ and our suggested method quite difficult to implement over a ``large alphabet size" (many observations and low distortion). \\

\noindent However, notice that if Gersho's conjecture is true \citep{gersho1979asymptotically}, and the best space-filling polytope is a lattice, then the optimum $d$-dimensional ECVQ at high resolution (low distortion) regime takes the form of a lattice \citep{zamir2014lattice}. This means that for this setup, $\gamma$ is simply a lattice quantizer. This idea is described in further detail in the next section.

\subsection{Vector Quantization with Fixed Lattices}
As demonstrated in the previous section, applying the ECVQ algorithm to a large number of observations $n$ with a low distortion constraint, is impractical. To overcome this problem we suggest using a predefined quantizer in the form of a lattice. This means that instead of seeking for a quantizer $\gamma$ that results in a random vector (over a finite alphabet) with a low average codeword length, we use a fixed quantizer, independent of the samples, and construct a codebook accordingly. Therefore, the performance of the codebook strongly depends on the empirical entropy of the quantized samples.\\

\noindent Since we are dealing with fixed lattices (vector quantizers), it is very likely that the empirical entropy of the quantized samples would be significantly different (lower) than the true entropy in low distortion regimes (large alphabet size). Therefore, the compressed data would consist of both the compressed samples themselves and a redundancy term, as explained in detail in Section \ref{universal source coding}. \\

\noindent 
Here again, we suggest that instead of encoding the quantized samples over a large alphabet size, we should first represent them in an ``as statistically independent  as possible" manner, and encode each component separately.\\

\noindent 
To demonstrate this scheme we turn to a classic quantizing problem, of a standard $d$-dimensional normal distribution.
Notice this quantizing problem is very well studied \citep{cover2012elements} and a lower bound for the average codeword length, for a given distortion value $D$, is given by
\begin{equation}
\label{normal R(D)}
R(D)=\max\left\{\frac{d}{2}\log\left(\frac{d}{D}\right),0\right\}.
\end{equation}
 \noindent In this experiment we draw $n$ samples from a standard $d$-dimensional multivariate normal distribution. Since the span of the normal distribution is infinite, we use a lattice which is only defined in a finite sphere. This means that each sample which falls outside this sphere is quantized to its nearest quantization point on the surface of the sphere. We define the radius of the sphere to be $5$ times the variance of the source (hence $r=5$). We first draw $n=10^5$ samples from $d=3,4$ and $8$ dimensional normal distributions. For $d=3$ we use a standard cubic lattice, while for $d=4$ we use an hexagonal lattice \citep{zamir2014lattice}. For $d=8$ we use an $8$-dimensional integer lattice \citep{zamir2014lattice}. The upper row of Figure \ref{fig:lattice1} demonstrates the results we achieve for the three cases respectively (left to right), where for each setup we compare the empirical joint entropy of the quantized samples (dashed line) with the sum of empirical marginal entropies, following our suggested approach (solid line). We further indicate the rate distortion lower bound (\ref{normal R(D)}) for each scenario, calculated according to the true distribution (line with x's). Notice the results are normalized according to the dimension $d$.
As we can see, the sum of empirical marginal entropies is very close to the empirical joint entropy for $d=3,4$. The rate distortion indeed bounds from below both of these curves. For $d=8$ the empirical joint entropy is significantly lower than the true entropy (especially in the low distortion regime). This is a result of an alphabet size which is larger than the number of samples $n$. However, in this case too, the sum of empirical marginal entropies is close to the joint empirical entropy. The behavior described above is maintained as we increase the number of samples to $n=10^6$, as indicated in the lower row of Figure \ref{fig:lattice1}. Notice again that the sum of marginal empirical entropies is very close to the joint empirical entropy, especially on the bounds (very high and very low distortion). The reason is that in both of these cases, where the joint probability is either almost uniform (low distortion) or almost degenerate (high distortion), there exists a representation which makes the components statistically independent. In other words, both the uniform and degenerate distributions can be shown to satisfy $\sum_{j=1}^d H(Y_j)=H(\underline{Y})$ under the order permutation. \\
      
\begin{figure}[ht]
\centering
\includegraphics[width = 0.9\textwidth,bb= 115 105 680 500,clip]{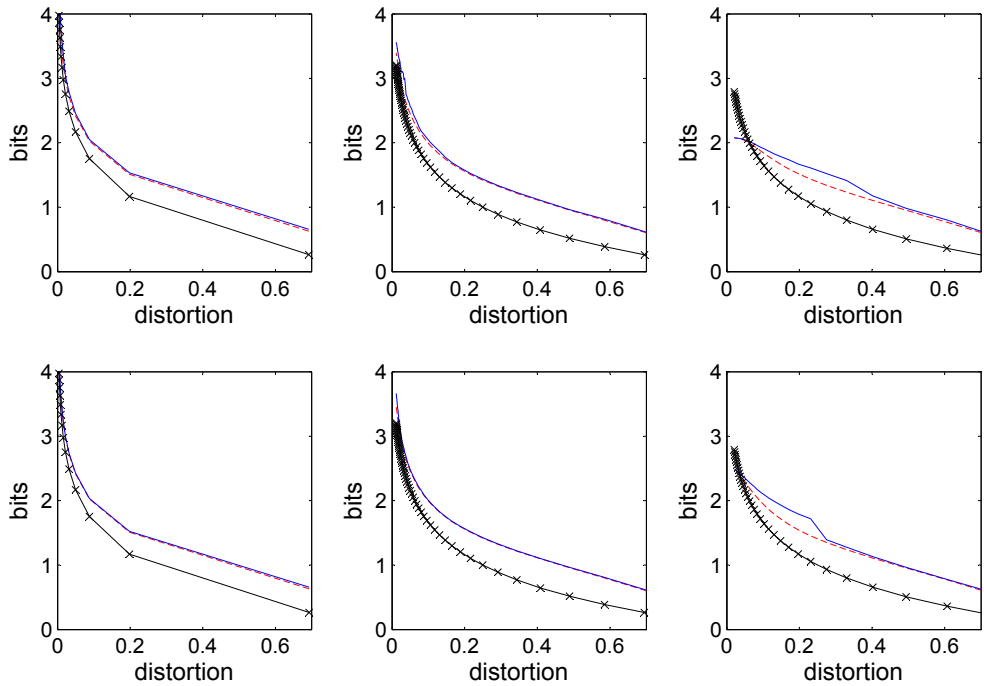}
\caption{Lattice quantization of $d$-dimensional standard normal distribution. The upper row corresponds to $n=10^5$ drawn samples while the lower row is $n=10^6$ samples. The columns correspond to the dimensions $d=3,4$ and $8$ respectively. In each setup, the dashed line is the joint empirical entropy while the solid line is the sum of marginal empirical entropies, following our suggested method. The line with the x's is the rate distortion (\ref{normal R(D)}), calculated according to the true distribution.} 
\label{fig:lattice1}
\end{figure}

\noindent We further present the total compression size of the quantized samples in this universal setting. Figure \ref{fig:lattice2} shows the amount of bits required for the quantized samples, in addition to the overhead redundancy, for both Huffman coding and our suggested scheme. As before, the rows correspond to $n=10^5$ and $n=10^6$ respectively, while the columns are $d=3,4$ and $8$, from left to right. We first notice that for $d=3,4$ both methods perform almost equally well. However, as $d$ increases, there exists a significant different between the classical coding scheme and our suggested method, for low distortion rate. The reason is that for larger dimensions, and low distortion rate, we need a very large number of quantization points, hence, a large alphabet size. This is exactly the regime where our suggested method demonstrates its enhanced capabilities, compared with standard methods.  

\begin{figure}[h]
\centering
\includegraphics[width = 0.95\textwidth,bb= 130 105 660 505,clip]{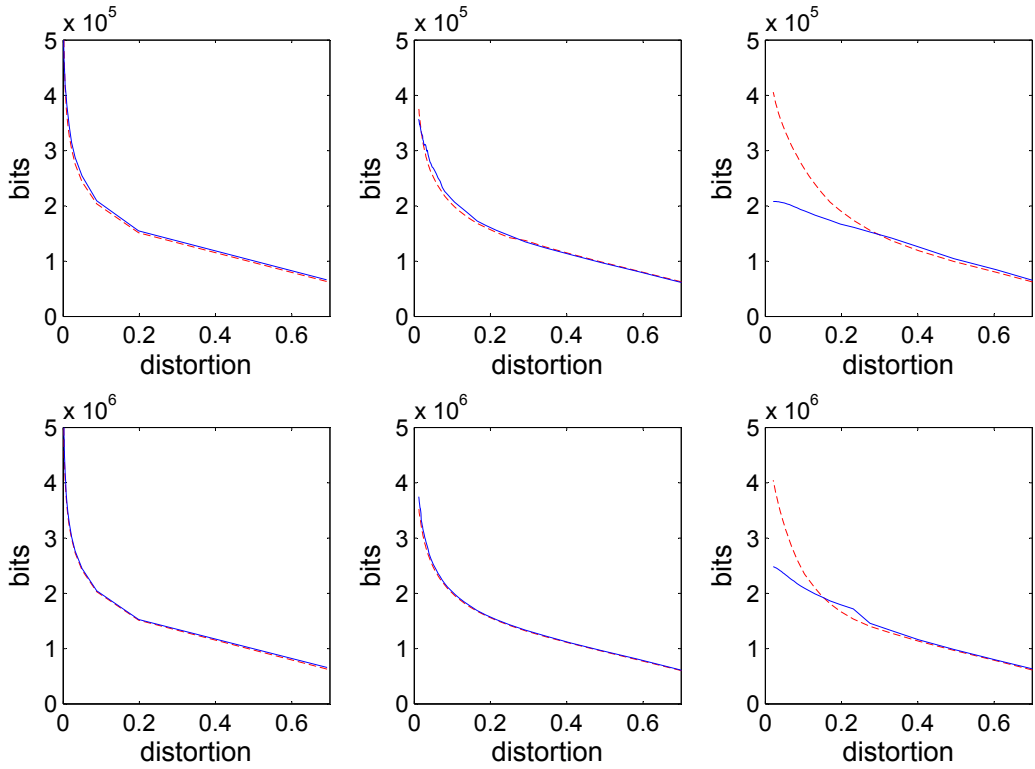}
\caption{Total compression size for lattice quantization of $d$-dimensional standard normal distribution. The upper row corresponds to $n=10^5$ drawn samples while the lower row is $n=10^6$ samples. The columns correspond to the dimensions $d=3,4$ and $8$, from left to right. In each setup, the dashed line is the total compression size through classical universal compression while the solid line is the total compression size using our suggested relaxed generalized BICA approach.} 
\label{fig:lattice2}
\end{figure}

\section{Discussion}
In this chapter we introduced a conceptual framework for large alphabet source coding. We suggest to decompose a large alphabet source into components which are ``as statistically independent as possible" and then encode each component separately. This way we overcome the well known difficulties of large alphabet source coding,  at the cost of: \begin{enumerate}  [(i)]
\item  Redundancy which results from encoding each component separately.   \label{a} 
\item  Computational difficulty of finding a transformation which decomposes the source. \label{b}
\end{enumerate} 
\vspace{-0.8 cm}
We propose two methods which focus on minimizing these costs. The first method is a piece-wise linear relaxation to the BICA (Chapter \ref{Generalized_BICA}). This method strives to decrease (\ref{a}) as much as possible, but its computationally complexity is quite involved. Our second method is the order permutation (Chapter \ref{Order_Permutation}) which is very simple to implement (hence, focuses on  (\ref{b})) but results in a larger redundancy as it is a greedy solution to (\ref{eq:sum_ent_min_binary}).\\

\noindent We demonstrated our suggested framework on three major large alphabet compression scenarios, which are the classic lossless source coding problem, universal source coding and vector quantization. We showed that in all of these cases, our suggested approach achieves a lower average codeword length than most commonly used methods.\\ 

\noindent All this together leads us to conclude that decomposing a large alphabet source into `` as statistically independent as possible" components, followed by entropy encoding of each components separately,  is both theoretically and practically beneficial.

%\chapter[Isotonic Modeling with Non-differentiable Loss Functions]{Isotonic Modeling with Non-differentiable Loss Functions}\label{Isotonic}
%\graphicspath{{Isotonic_Figures//}}
%\input{Chapters//Isotonic}

%\chapter[Cross-Validated Variable Selection in Tree-Based Methods]{Cross-Validated Variable Selection in Tree-Based Methods}\label{loo_trees}
%\chaptermark{CV Variable Selection in Tree-Based Methods}
%\graphicspath{{loo_trees_Figures//}}
%\input{Chapters//loo_trees}

%\chapter[Compressing Random Forests]{Compressing Random Forests}\label{trees_comp}
%\graphicspath{{trees_comp_Figures//}}
%\input{Chapters//trees_comp}

\iffalse
\newpage
\thispagestyle{empty}
\mbox{}
\thispagestyle{empty}
\fi

\begin{appendices}
  \chapter{}
  \label{unique_values_markov_model}
\begin{theorem}
Assume a binary random vector $\underline{X} \in \{0,1\}^d$ is generated from a first order stationary symmetric  Markov model. Then, the joint probability of $\underline{X}$,  $P_{\b{x}}=p_1, \dots, p_m$ only contains $d\cdot(d-1)+2$ unique (non-identical) elements of $p_1, \dots, p_m$.
\end{theorem}

\begin{proof}
We first notice that for a binary, symmetric and stationary Markov model, the probability of each word is solely determined by 

\begin{itemize}

\item	The value of the first (most significant) bit 
\item	The number of elements equal $1$ (or equivalently $0$) 
\item	The number of transitions from $1$ to $0$ (and vice versa).

\end{itemize}
 
\noindent For example, for $d=4$ the probability of $0100$ equals the probability of $0010$, while it is not equal to the probability of $0001$.\\

\noindent First, assume the number of transitions, denoted in $r$, is even. Further, assume that the first (most significant) bit equals zero. Then, the number of words 
with a unique probability is 
\begin{equation}
\label{U1}
U_1=\sum_{\substack{r=2,\\ r\; \text{is even}}}^{d-2} \sum_{k=\frac{r}{2}}^{d-\frac{r}{2}}1=\sum_{\substack{r=2,\\ r\; \text{is even}}}^{d-2}d-r
\end{equation}
where the summation over $r$ corresponds to the number of transitions, while the summation over $k$ is with respect to the number of $1$ elements given $r$. For example, for $d=4$, $r=2$ and $k=1$ we have the words $0100, 0010$ (which have the same probability as discussed above), while for $k=2$ we have $0110$.
In the same manner, assuming that the most significant bit is $0$ but now $r$ is odd, we have 
\begin{equation}
\label{U2}
U_2=\sum_{\substack{r=2,\\ r\; \text{is odd}}}^{d-1} \sum_{k=\frac{r+1}{2}}^{d-\frac{r+1}{2}}1=\sum_{\substack{r=2,\\ r\; \text{is odd}}}^{d-1}d-r.
\end{equation}
Putting together (\ref{U1}) and (\ref{U2}) we have that number of words with a unique probability, assuming the most significant bit is $0$, equals 
\begin{equation}
U1+U_2=\sum_{r=1}^{d-1} d-r = \frac{d\cdot(d-1)}{2}+1.
\end{equation}
\noindent The same derivation holds for the case where the most significant bit is $1$, leading to a total of $d\cdot(d-1)+2$ words with a unique probability 

\end{proof}

  \chapter{}
  \begin{proposition}
\label{H(X)}
Let $\underline{X} \sim \underline{p} $ be a random vector of an alphabet size $m$ and a joint probability distribution $\underline{p}$. The expected joint entropy of $\underline{X}$, where the expectation is  over a uniform simplex of joint probability distributions $\underline{p}$ is
\begin{equation}\nonumber
\mathbb{E}_{\underline{\smash{p}}}\left\{H(\underline{X}) \right\}=\frac{1}{\log_e{2}}\left(\psi(m+1)-\psi(2)\right) 
\end{equation}
where $\psi$ is the \textit{digamma function}.
\end{proposition}

\label{proof_for_H(X)}
\begin{proof}
We first notice that a uniform distribution over a simplex of a size $m$ is equivalent to a Direchlet distribution with parameters $\alpha_i=1, i=1, \dots,m$. The Direchlet distribution can be generated through normalized independent random variables from a Gamma distribution. This means that  for statistically independent $Z_i \sim \Gamma(k_i=1, \theta_i=1), i=1, \dots,m$ we have that 
\begin{equation}
\label{uniform simplex}
\frac{1}{\sum_{k=1}^{m}Z_k}\left(Z_1, \dots Z_m\right) \sim Dir\left(\alpha_1=1, \dots, \alpha_m=1\right).
\end{equation}
\noindent We are interested in the expected joint entropy of draws from (\ref{uniform simplex}), 
\begin{align}
\label{joint entropy}
\mathbb{E}_{\underline{\smash{p}}}\left\{H(\underline{X}) \right\}=&-\sum_{i=1}^{m}\mathbb{E}\left\{ \frac{Z_i}{\sum_{k=1}^{m}Z_k} \log{\frac{Z_i}{\sum_{k=1}^{m}Z_k}}   \right\} =\\\nonumber
& -m\mathbb{E}\left\{ \frac{Z_i}{\sum_{k=1}^{m}Z_k} \log{\frac{Z_i}{\sum_{k=1}^{m}Z_k}}\right\}
\end{align}

\noindent It can be shown that for two independent Gamma distributed random variables $X_1\sim\Gamma(\alpha_1,\theta)$ and $X_2\sim\Gamma(\alpha_2,\theta)$, the ratio $\frac{X_1}{X_1+X_2}$ follows a Beta distribution with parameters $(\alpha_1,\alpha_2)$. Let us denote $\tilde{Z}_i \triangleq \frac{Z_i}{\sum_{k=1}^{m}Z_k}=\frac{Z_i}{Z_i+\sum_{k \neq i} Z_k}$. Notice that  $Z_i\sim\Gamma(1,1)$ and $\sum_{k \neq i} Z_i\sim\Gamma(m-1,1)$ are mutually independent. Therefore,  
\begin{equation}
f_{\tilde{Z}_i}(z) =Beta(1,m-1)=\frac{(1-z)^{(m-2)}}{B(1,m-1)}.
\end{equation}
\noindent This means that
\begin{align}
\label{derivation1}
&\mathbb{E}\left\{ \frac{Z_i}{\sum_{k=1}^{m}Z_k} \log{\frac{Z_i}{\sum_{k=1}^{m}Z_k}}\right\}=\mathbb{E}\left\{ \tilde{Z}_i \log \tilde{Z}_i\right\}=\\\nonumber
&\frac{1}{B(1,m-1)}\int_0^1 z\log{(z)} (1-z)^{(m-2)}dz=\\\nonumber
&\frac{B(2,m-1)}{B(1,m-1)}\frac{1}{\log_e{(2)}}\frac{1}{B(2,m-1)}\int_0^1 \log_e{(z)}z (1-z)^{(m-2)}dz=\\\nonumber
&\frac{1}{m\log_e{(2)}}\mathbb{E}\left(\log_e{(U)}\right)
\end{align}
\noindent where $U$ follows a Beta distribution with parameters $(2, m-1)$.  The expected natural logarithm of a Beta distributed random variable, $V\sim Beta(\alpha_1, \alpha_2)$, follows $\mathbb{E}\left(\log_e{(V)}\right)=\psi(\alpha_1)-\psi(\alpha_1+\alpha_2)$ where $\psi$ is the \textit{digamma function}.
Putting this together with (\ref{joint entropy}) and (\ref{derivation1}) we attain
\begin{equation}
\mathbb{E}_{\underline{\smash{p}}}\left\{H(\underline{X}) \right\} = -m\mathbb{E}\left\{ \frac{Z_i}{\sum_{k=1}^{m}Z_k} \log{\frac{Z_i}{\sum_{k=1}^{m}Z_k}}\right\}=\frac{1}{\log_e{(2)}}\left(\psi(m+1)-\psi(2)\right) 
\end{equation}

\end{proof}

  \chapter{}
  \label{on_the_uniquness}
In the generalized Gram-Schmidt method we suggested that for any process $X$ with a cumulative distribution function $F(X^k)$ we would like to sequentially construct $Y^k$ such that:
\begin{enumerate}
\item $F(Y^k)=\prod_{j=1}^k F(Y_j)$.
\item  $X^k$ can be uniquely recovered from $Y^k$ for any $k$.
\end{enumerate}
We presented a sequential framework for constructing such a memoryless process, given that the desired probability measure is non-atomic.  For simplicity of notation we reformulate our problem as follows: Assume a random variable $Y$ is to be constructed from a random variable $X$ given $X$'s past, denoted as $X_p$.  Therefore we would like to construct a memoryless random variable $Y=g(X,X_p)$ with a given $F_Y (y)$ such that
\begin{enumerate} [(i)]
\item  $Y$ is statistically independent in $X_p$
\item  $X$ can be uniquely recovered from $Y$ given $X_p$
\item $Y \sim F_Y (y)$
\end{enumerate}
Our goal is therefore to find such $Y=g(X,X_p)$ and discuss its uniqueness.

\section{The Uniform Distribution Case}
In this section we consider a special case where $Y$ is uniformly distributed, $F_Y (y)=y$ $\forall y \in [0,1]$. For $Y$ to be statistically independent of $X_p$ it must satisfy
\begin{equation}
F_{Y|X_p} (y | X_p=x_p)=F_Y (y).
\label{bla}
\end{equation}
Deriving the left hand side of (\ref{bla}) we have that for all $x_p$, 
$$ F_{Y|X_p} (y | X_p=x_p)=P(Y \leq y | X_p=x_p)=P(g(X,X_p)\leq y | X_p=x_p).	$$

\subsection{Uniqueness of Monotonically Increasing Transformations}
The second constraint suggests $X$ can be uniquely recovered from $Y$ and $X_p$, which implies $X=g_{X_p}^{-1} (Y)$. Assume $g(X,X_p)$ is monotonically increasing with respect to $X$. Then, we have that
\begin{align}
\label{5.4}
F_{Y|X_p} (y | X_p=x_p)=&P(g(X,X_p) \leq y| X_p=x_p)=\\\nonumber
&P(X \leq g_{X_p}^{-1} (y) | X_p=x_p)=F_{X|X_p} (g_{X_p}^{-1} (y)|X_p=x_p) 
\end{align}
where the second equality follows from the monotonically increasing behavior of $g(X,X_p)$ with respect to $X$. Therefore, we are looking for a monotonically increasing transformation $x=g_{X_p}^{-1} (y)$ such that
$$ F_{X|X_p} (g_{X_p}^{-1}(y)| X_p=x_p)=F_Y (y)=y.$$
The following lemmas discuss the uniqueness of monotonically increasing mappings when $X$ is a non-atomic (Lemma \ref{uniquness_1}) or atomic (Lemma \ref{uniquenss_2}) measure.

\begin{lemma}
\label{uniquness_1}
Assume $X$ is a non-atomic random variable with a strictly monotonically increasing commutative distribution function  $F_X (x)$ (that is, $X$ takes values on a continuous set). Suppose there exists a transformation on its domain, $x=h(y)$ such that
$$ F_X (x)|_{x=h(y)}=F_Y (y).$$
Then,
\begin{enumerate}[(1)]
\item 	$x=h(y)$ is unique \label{1}
\item 	 $h(y)$ is monotonically non decreasing (increasing, if $F_Y (y)$ is strictly increasing) \label{2}.
\end{enumerate}
\end{lemma}
\begin{proof}
Let us begin with proving (\ref{1}). The transformation $x=h(y)$ satisfies
$$ F_X (x) |_{x=h(y)}=F_X (h(y))=P(X\leq h(y))=F_Y (y).$$
Suppose there is another transformation $x=g(y)$  that satisfies the conditions stated above. Then,
$$F_X (x) |_{x=g(y) }=F_X (g(y))=P(X\leq g(y))=F_Y (y).$$
Therefore, 
$$P(X\leq g(y))=P(X\leq h(y)) \quad \forall y.$$
Suppose $h(y)\neq g(y)$. This means that there exists at least a single $y=\tilde{y}$ where $g(\tilde{y})=h(\tilde{y})+\delta$ and $\delta \neq 0$. It follows that
$$ P(X \leq h(\tilde{y})+\delta)=P(X \leq  h(\tilde{y}))$$
or in other words
$$ F_X (h(\tilde{y}))=F_X (h(\tilde{y} )+\delta)$$
which contradicts the monotonically increasing behavior of $F_X (x)$ where the transformation is defined.\\
As for \noindent (\ref{2}), we have that $F_X (h(y))=F_Y (y) $ for all $y$. Therefore, 
$$F_X (h(y+\delta))=F_Y (y+\delta).$$
$F_Y (y)$ is a CDF which means that it satisfies $F_Y (y+\delta) \geq F_Y (y)$. Then, 
$$F_X (h(y+\delta))\geq F_X (h(y))$$
(strictly larger if $F_Y (y)$ is monotonically increasing).
Since $F_X (x)$ is monotonically increasing we have that $h(y+\delta)\geq h(y)$ (strictly larger if $F_Y (y)$ is monotonically increasing) 
\end{proof}

\begin{lemma}
\label{uniquenss_2}
Assume $X$ is a non-atomic random variable with a commutative distribution function  $F_X (x)$. Suppose there exists a transformation on its domain, $x=h(y)$ such that
$$ F_X (x)|_{x=h(y)}=F_Y (y).$$
Then,
\begin{enumerate}[(1)]
\item 	$x=h(y)$ is unique up to transformations in zero probability regions $X$'s domain \label{3}
\item 	 $h(y)$ is monotonically non decreasing (increasing, if $F_Y (y)$ is strictly increasing) \label{4}.
\end{enumerate}
\end{lemma}

\begin{proof}
(\ref{1}) As in Lemma \ref{uniquness_1}, 	let us assume that there exists another transformation $x= g(y)$  that satisfies the desired conditions. Therefore we have that
$$ P(X\leq g(y))=P(X\leq h(y))\quad \forall y.$$
Assuming $h(y)\neq g(y)$ we conclude that there exists at least a single value $y=\tilde{y}$ such that $g(\tilde{y})=h(\tilde{y})+\delta$ and $\delta \neq 0$. 
If both $h(\tilde{y})$ and $g(\tilde{y} )$ are valid values in $X$'s domain (positive probability) then we have $P(X\leq x_1 )=P(X\leq x_2)$.
This contradicts  $P(X=x_1 )>0$ and $P(X=x_2 )>0$ unless $x_1=x_2$.\\
\noindent  Moreover, if $g(\tilde{y} )\in [x_1,x_2 ]$ and $h(\tilde{y}) \notin [x_1,x_2 ]$ then again it contradicts $P(X=x_1 )>0$ and $P(X=x_2 )>0$ unless $x_1=x_2$. The only case in which we are not facing a contradiction is where  $g(\tilde{y}),h(\tilde{y}) \in[x_1,x_2 ]$. In other words, $x= g(y)$  is unique up to transformations in zero probability regions of $X$'s domain (regions which satisfy $P(X=g(\tilde{y} ))=0)$.\\

\noindent (\ref{4}) The monotonicity proof follows the same derivation as in Lemma \ref{uniquness_1}.    

\end{proof}

\noindent Therefore, assuming that there exists a transformation $x=g_{X_p}^{-1} (y)$ such that 
$$F_{X|X_p} (g_{X_p}^{-1} (y)|X=x_p)=F_Y (y)=y,$$
then it is unique and monotonically increasing. In this case we have that
\begin{align}
F_Y (y)=&F_{X|X_p} (g_{X_p}^{-1} (y)|X=x_p)=P(X\leq g_{X_p}^{-1} (y)|X=x_p)=\\\nonumber
&P(g(X,X_p)\leq y | X=x_p)  =F_{Y|X_p}(y|X_p=x_p)  
\end{align}
which means $Y$ is statistically independent of $X_p$.
Equivalently, if we find a monotonically increasing transformation $Y=g(X,X_p)$ that satisfies conditions (i), (ii) and (iii) then it is unique.

\subsection{Non Monotonically Increasing Transformations}
In the previous section we discussed the case in which we limit ourselves to functions $g(X,X_p)$ which are monotone in $X$. For this set of functions equation (\ref{5.4}) is a sufficient condition for satisfying (i) and (ii). 
However, we may find non monotonically increasing transformations $Y=h(X,X_p)$ which satisfy conditions (i), (ii) and (ii) but do not satisfy (\ref{5.4}). For example: $h(X,X_p)=1-g(X,X_p)$. Notice these transformations are necessarily measurable, as they map one distribution to another, and reversible with respect to $X$ given $X_p$ (condition ii). In this case, the following properties hold:
\begin{lemma}
\label{uniqueness_3}
Assume h(X,Y) satisfies the three conditions  mentioned above but does not satisfy equation (5.4). Then:
\begin{enumerate}[(1)]
\item $h(X,X_p)$ is not monotonically increasing in $X$ \label{5}
\item $h(X,X_p)$ is necessarily a ``reordering" of $g(X,X_p)$ \label{6}
\end{enumerate}
\end{lemma}

\begin{proof}
(\ref{5}) Assume there exists a transformation $Y=h(X,X_p)$ which satisfy the three conditions (i), (ii) and (iii). Moreover assume $h(X,X_p) \neq g(X,X_p)$.  We know that 
$$F_{Y|X_p} (y|X_p=x_p)=P(h(X,X_p) \leq y | X_p=x_p)=F_Y (y) $$   
but on the other hand, $h(X,X_p) \neq g(X,X_p)$  which implies
$$ F_{X|X_p} (h_{X_p}^{-1} (y)|X_p=x_p) \neq F_Y (y)$$					
since $g(X,X_p)$ is unique. Therefore,
$$ P(h(X,X_p) \leq y | X_p=x_p) \neq P (X \leq h_{X_p}^{-1} (y) | X_p=x_p)$$
which means $h(X,X_p)$ cannot be monotonically increasing.\\

\noindent (\ref{6}) Notice we can always generate a (reversible) transformation of $h(X,X_p)$ that will make it monotonically increasing with respect to $X$, since $X$ is uniquely recoverable from $h(X,X_p)$ and $X_p$. Consider this transformation as $S(h(X,X_p))$. Therefore, we found $Y=S(h(X,X_p))$ such that $y$ is monotonically increasing, independent of $X_p$ and $X$ is uniquely recoverable from $Y$ and $X_p$. This contradicts the uniqueness of $g(X,X_p)$ unless $S(h(X,X_p))=g(X,X_p)$, which means 
$h(X,X_p)=S^{-1} (g(X,X_p))$.
\end{proof}

\subsection{The Existence of a Monotonically Increasing Transformation}
Following the properties we presented in the previous sections, it is enough to find $Y=g(X,X_p)$ which is invertible and monotonically increasing with respect to $X$ given $X_p=x_p$, and satisfies
$$ F_{Y|X_p} (y│X_p=x_p)=F_{Y|X_p} (g_{X_p}^{-1} (y)|X_p=x_p)=F_Y (y)=y.$$ 	
If such $Y=g(X,X_p)$ exists then
\begin{enumerate}
\item  If $F_{X|X_p} (x│X_p=x_p)$ is monotonically increasing, then $Y=g(X,X_p)$ is unique according to Lemma \ref{uniquness_1}
\item  If $X|X_p$ takes on discrete values, then again $Y=g(X,X_p)$ is unique, up to different transformations in zero probability regions of the $X|X_p$
\item  Any other transformations $h(X,X_p)$ that may satisfy conditions (i),(ii) and (iii) is necessarily a function of $g(X,X_p)$ (and not monotonically increasing).
\end{enumerate}

\noindent Following lemma \ref{uniquness_1} we define $Y=F_{X|X_p} (x|x_p)-\Theta \cdot P_{X|X_p} (x│x_p)$, where $\Theta \sim \text{Unif}[0,1]$ is statistically independent of $X$ and $X_p$. Therefore we have that
\begin{align}
F_{Y|X_p} (y|x_p)=& P(F_{X|X_p} (x|x_p)-\Theta \cdot P_{X|X_p} (x│_p) \leq y | X_p=x_p)=\\\nonumber
&P(F_{X|X_p} (x|x_p)-\Theta \cdot P_{X|X_p} (x│x_p) \leq h^{-1} (y))=y=F_Y(y)
\end{align}
where the first equality follows from the fact that all the terms in $F_{X|X_p} (x|x_p)-\Theta \cdot P_{X|X_p} (x│x_p) \leq h^{-1} (y)$ are already conditioned on $X_p$, or statistically independent of $X_p$, and the second equality follows from $	F_{X|X_p} (x|x_p)-\Theta \cdot P_{X|X_p} (x│x_p) \sim \text{Unif}[0,1]$, according to lemma \ref{uniquness_1}. The third condition is remaining requirement. However, it is easy to see that $Y=F_{X|X_p} (x|x_p)-\Theta \cdot P_{X|X_p} (x│x_p)$ is reversible with respect to $X$ given $X_p=x_p$. Therefore, we found a monotonically increasing transformation $Y=g(X,X_p)$ that satisfies 
$$ F_{X|X_p} (x|x_p)=F_{X|X_p} (g_{X_p}^{-1} (y)|X_p=x_p)=F_Y (y)=y$$

\section{The Non-Uniform Case}

Going back to our original task, we are interested in finding such $Y=g(X,X_p)$ such that there exists a random variable $Y$ that satisfies conditions (i), (ii) and (iii). \\

\noindent Throughout the previous sections we discussed the uniqueness of the case in which $Y$ is uniformly distributed. Assume we are now interested in a non-uniformly distributed $Y$. Lemma \ref{uniquness_1} shows us that we can always reshape a uniform distribution to any probability measure by applying the inverse of the desired CDF on it. Moreover, if the desired probability measure is non-atomic, this transformation is reversible. Is this mapping unique? This question was already answered by Lemmas \ref{uniquenss_2} and \ref{uniqueness_3}; if we limit ourselves to monotonically increasing transformation, then the solution we found is unique.\\

\noindent However, assume we do not limit ourselves to monotonically increasing transformations and we have a transformation  $V=G(Y)$ that satisfies $V \sim F_V (v)$. Since $Y$  is uniformly distributed we can always shift between local transformations on sets of the same lengths while maintaining the transformation measurable. Then we can always find $S(G(Y))$ which makes it monotonically increasing with respect to $Y$. This contradicts the uniqueness of the monotonically increasing set unless $S(G(Y))$ equals the single unique transformation we found.  \\

\noindent Putting it all together we have a two stage process in which we first generate a uniform transformation and then shape it to a desired distribution $V$ through the inverse of the desired CDF. We show that in both stages, if we limit ourselves to monotonically increasing transformations the solution presented in \citep{shayevitz2011optimal} is unique. However, if we allow ourselves a broader family of functions we necessarily end up with either the same solution, or a ``reordering" of it which is not monotonically increasing.

  \chapter{}
  \label{Appendix_mutual_inf_max_points}
We analyze the three different regions of $\beta$, compared with the parameteres of the Markov process, $\alpha_1 \leq \alpha_2$.
\begin{proposition}
For $\beta<\alpha_1<\alpha_2<\frac{1}{2}$,  the maximal mutual information, $ I_{max}\left(X_k;Y_k|X^{k-1} \right)$, is monotonically increasing in $\beta$
\end{proposition}
\begin{proof}
Let us derive the maximal mutual information with respect to $\beta$:
\begin{align}
\frac{\partial}{\partial \beta} I_{max} \left(X_k;Y_k|X^{k-1} \right)& = \gamma \left(\log \frac{1-\beta}{\beta}-\alpha_1\left(\log\left(1-\frac{\beta}{\alpha_1}\right)-\log\frac{\beta}{\alpha_1}\right)\frac{1}{\alpha_1} \right)+\\\nonumber
&(1-\gamma) \left(\log \frac{1-\beta}{\beta}-\alpha_2\left(\log\left( 1-\frac{\beta}{\alpha_2} \right)-\log \frac{\beta}{\alpha_2}   \right) \frac{1}{\alpha_2}\right)=\\\nonumber
&\log \frac{1-\beta}{\beta}-\gamma \log \frac{\alpha_1-\beta}{\beta}-(1-\gamma) \log \frac{\alpha_2-\beta}{\beta}>\\\nonumber
&\log \frac {1-\beta}{\beta} - \gamma \log \frac{\alpha_1-\beta}{\beta}-(1-\gamma)\log \frac{\alpha_1-\beta}{\beta}=\\\nonumber
&\log \frac{1-\beta}{\alpha_1-\beta}>0 
\end{align}
where the first inequality follows from $\frac{\alpha_2-\beta}{\beta}>\frac{\alpha_1-\beta}{\beta}$ and the second inequality results from  $ \alpha_1<1  \Rightarrow \frac{1-\beta}{\alpha_1-\beta}>1$.
\end{proof}

\begin{proposition}
For $\alpha_1<\alpha_2<\beta<\frac{1}{2}$,  the maximal mutual information, $ I_{max}\left(X_k;Y_k|X^{k-1} \right)$, is monotonically decreasing in $\beta$
\end{proposition}
\begin{proof}
Let us again derive the maximal mutual information with respect to $\beta$:
\begin{align}
&\frac{\partial}{\partial \beta} I_{max} \left(X_k;Y_k|X^{k-1} \right) =\\\nonumber
& \gamma \left(\log \frac{1-\beta}{\beta}-(1-\alpha_1)\left(\log\left(1-\frac{\beta-\alpha_1}{1-\alpha_1}\right)-\log\frac{\beta-\alpha_1}{1-\alpha_1}\right)\frac{1}{1-\alpha_1} \right)+\\\nonumber
&(1-\gamma) \left(\log \frac{1-\beta}{\beta}-(1-\alpha_2)\left(\log\left( 1-\frac{\beta-\alpha_2}{1-\alpha_2} \right)-\log \frac{\beta-\alpha_2}{1-\alpha_2}   \right) \frac{1}{1-\alpha_2}\right)=\\\nonumber
&\log \frac{1-\beta}{\beta}-\gamma \log \frac{1-\beta}{\beta-\alpha_1}-(1-\gamma) \log \frac{1-\beta}{\beta-\alpha_2}<\\\nonumber
&\log \frac {1-\beta}{\beta} - \gamma \log \frac{1-\beta}{\beta-\alpha_1}-(1-\gamma)\log \frac{1-\beta}{\beta-\alpha_1}=\\\nonumber
&\log \frac{\beta-\alpha_1}{\beta}<0 
\end{align}
where the first inequality follows from $\frac{1-\beta}{\beta-\alpha_1}<\frac{1-\beta}{\beta-\alpha_2}$.
\end{proof}

\begin{proposition}
All optimum points in the range of $\alpha_1<\beta<\alpha_2$ are local minimums
\end{proposition}
\begin{proof}
In the same manner, we derive the maximal mutual information with respect to $\beta$:
\begin{align}
\frac{\partial}{\partial \beta} I_{max} \left(X_k;Y_k|X^{k-1} \right) =&\log \frac {1-\beta}{\beta} - \gamma \log \frac{1-\beta}{\beta-\alpha_1}-(1-\gamma)\log \frac{\alpha_2-\beta}{\beta}=\\\nonumber
&\log \frac{1-\beta}{\alpha_2-\beta}-\gamma\left(\log \frac {1-\beta}{1-\alpha_1}  -\log \frac {\alpha_2-\beta}{\beta}\right) 
\end{align}
\begin{align}
\frac{\partial^2}{\partial^2 \beta} I_{max} \left(X_k;Y_k|X^{k-1} \right) =& \frac{\alpha_2-\beta}{1-\beta}\cdot \frac{1-\alpha_2}{(\alpha_2-\beta)^2}-\\\nonumber
&\gamma\left(  \frac{1-\alpha_1}{1-\beta}\cdot \frac{-1}{1-\alpha_1}-\frac{\beta}{\alpha_2-\beta}\cdot\frac{\alpha_2}{\beta^2}\right)=\\\nonumber
&\frac{1-\alpha_2}{(\alpha_2-\beta)(1-\beta)}+\gamma\left(\frac{1}{1-\beta}+\frac{\alpha_2}{(\alpha_2-\beta)\beta}   \right)>0
\end{align}
\end{proof}

  \chapter{}
  \label{Appendix_stationary_case_a_1_less_a_2}
We would like to show that for $\alpha_1<\alpha_2<\frac{1}{2}$ the following applies:

\begin{equation}
\frac{\alpha_2}{1-\alpha_1-\alpha_2}<\frac{h_b(\alpha_2)-h_b(\alpha_1)+\alpha_2 h_b\left( \frac{\alpha_1}{\alpha_2}\right)}{\alpha_2 h_b\left( \frac{\alpha_1}{\alpha_2}\right)+(1-\alpha_1)h_b\left( \frac{\alpha_2-\alpha_1}{1-\alpha_1}\right)}
\end{equation}
\begin{proof}
\noindent Let us first cross multiply both sides of the inequality
\begin{align}
&\alpha^2_2 h_b\left( \frac{\alpha_1}{\alpha_2}\right)+\alpha_2(1-\alpha_1)h_b\left( \frac{\alpha_2-\alpha_1}{1-\alpha_1}\right)<\\\nonumber
&(1-\alpha_1-\alpha_2)(h_b(\alpha_2)-h_b(\alpha_1))+(1-\alpha_1)\alpha_2 h_b\left( \frac{\alpha_1}{\alpha_2}\right)+\alpha^2_2 h_b\left( \frac{\alpha_1}{\alpha_2}\right)
\end{align}
which leads to
\begin{equation}
\nonumber
(1-\alpha_1-\alpha_2)(h_b(\alpha_2)-h_b(\alpha_1))+(1-\alpha_1)\alpha_2 h_b\left( \frac{\alpha_1}{\alpha_2}\right)-\alpha_2(1-\alpha_1)h_b\left( \frac{\alpha_2-\alpha_1}{1-\alpha_1}\right)>0.
\end{equation}
\\
Since $h_b(\alpha_2)-h_b(\alpha_1)>0$ and $1-\alpha_1-\alpha_2>(1-\alpha_1)\alpha_2$ we have that
\begin{align}
\nonumber
&(1-\alpha_1-\alpha_2)(h_b(\alpha_2)-h_b(\alpha_1))+(1-\alpha_1)\alpha_2 h_b\left( \frac{\alpha_1}{\alpha_2}\right)-\alpha_2(1-\alpha_1)h_b\left( \frac{\alpha_2-\alpha_1}{1-\alpha_1}\right)>\\\nonumber
&(1-\alpha_1)\alpha_2\left[h_b(\alpha_2)-h_b(\alpha_1)+h_b\left( \frac{\alpha_1}{\alpha_2}\right)- h_b\left( \frac{\alpha_2-\alpha_1}{1-\alpha_1}\right)  \right].
\end{align}
Therefore, it is enough to show that $h_b(\alpha_2)-h_b(\alpha_1)+h_b\left( \frac{\alpha_1}{\alpha_2}\right)- h_b\left( \frac{\alpha_2-\alpha_1}{1-\alpha_1}\right) >0$. Since $h_b\left( \frac{\alpha_2-\alpha_1}{1-\alpha_1}\right)=h_b\left( \frac{1-\alpha_2}{1-\alpha_1}\right)$ we can rewrite the inequality as
\begin{align}
\nonumber
h_b(\alpha_2)-h_b(\alpha_1)>h_b\left( \frac{1-\alpha_2}{1-\alpha_1}\right)-h_b\left( \frac{\alpha_1}{\alpha_2}\right).
\end{align}
Notice that  $\alpha_1<\alpha_2<\frac{1}{2}$ follows that $\frac{1-\alpha_2}{1-\alpha_1}>\frac{1}{2}$.\\

\noindent Let us first consider the case where $\frac{\alpha_1}{\alpha_2}\geq\frac{1}{2}$. We have that
\begin{align}
 \frac{1-\alpha_2}{1-\alpha_1}-\frac{\alpha_1}{\alpha_2}=\frac{(\alpha_2-\alpha_1)(1-\alpha_1-\alpha_2)}{(1-\alpha_1)\alpha_2}>0.
\end{align}
Since $\frac{1-\alpha_2}{1-\alpha_1}-\frac{\alpha_1}{\alpha_2}>\frac{1}{2}$ and $h_b(p)$ is monotonically decreasing for $p\geq\frac{1}{2}$, we have that 
\begin{equation}
h_b\left( \frac{1-\alpha_2}{1-\alpha_1}\right)-h_b\left( \frac{\alpha_1}{\alpha_2}\right)<0<h_b(\alpha_2)-h_b(\alpha_1).
\end{equation}
Now consider the case where $\frac{\alpha_1}{\alpha_2}\geq\frac{1}{2}$. We notice that:
\begin{equation}
h_b\left( \frac{1-\alpha_2}{1-\alpha_1}\right)=h_b\left(1- \frac{1-\alpha_2}{1-\alpha_1}\right)=h_b\left( \frac{\alpha_2-\alpha_1}{1-\alpha_1}\right)
\end{equation}
where $\frac{\alpha_2-\alpha_1}{1-\alpha_1}<\frac{1}{2}.$
In addition, 
\begin{equation}
\frac{\alpha_2-\alpha_1}{1-\alpha_1}-\frac{\alpha_1}{\alpha_2}=\frac{(\alpha_2-\alpha_1)^2+\alpha_1(1-\alpha_2)}{(1-\alpha_1)\alpha_2}>0.
\end{equation}
Therefore, we would like to show that 
$$h_b(\alpha_2)-h_b(\alpha_1)>h_b\left( \frac{\alpha_2-\alpha_1}{1-\alpha_1}\right)-h_b\left( \frac{\alpha_1}{\alpha21}\right)$$
where all the binary entropy arguments are smaller than $\frac{1}{2}$ and both sides of the inequality are non-negative. In order to prove this inequality we remember that $h_b(p)$ is monotonically increasing with a decreasing slope, $\frac{\partial}{\partial p}h_b(p)=\log\frac{1-p}{p}$, for $p<\frac{1}{2}$. Then, it is enough to show that $\alpha_1<\frac{\alpha_1}{\alpha_2}$ (immediate result) and 
$$\alpha_2-\alpha_1>\frac{\alpha_2-\alpha_1}{1-\alpha_1}- \frac{\alpha_1}{\alpha_2}.$$
Looking at the difference between the two sides of the inequality we obtain:
\begin{align}
\frac{\alpha_2-\alpha_1}{1-\alpha_1}- \frac{\alpha_1}{\alpha_2}-(\alpha_2-\alpha_1)=&
(\alpha_2-\alpha_1)\frac{\alpha_1}{1-\alpha_1}-\frac{\alpha_1}{\alpha_2}<\\\nonumber
&\frac{1}{2}(1-\alpha_1)\frac{\alpha_1}{1-\alpha_1}-\frac{\alpha_1}{\alpha_2}=\alpha_1\left(\frac{\alpha_2-2}{2\alpha_2} \right)<0
\end{align}
where the inequality follows from $\frac{\alpha_2-\alpha_1}{1-\alpha_1}<\frac{1}{2} \Rightarrow \alpha_2-\alpha_1<\frac{1}{2}(1-\alpha_1)$.
\end{proof}

\end{appendices}

%\phantomsection
%\addcontentsline{toc}{chapter}{Appendix}
%\appendix
%\chapter[Appendix]{}\label{Appendix}
%\phantomsection
%\addcontentsline{toc}{chapter}{Appendix}
%\input{Chapters//Appendix}

%\newpage
%\thispagestyle{empty}
%\mbox{}

\backmatter

\bibliographystyle{authordate1}
\phantomsection
\addcontentsline{toc}{chapter}{Bibliography}
\bibliography{Chapters//References}

\end{document}